\documentclass[11pt]{article}

\usepackage[left=1in,top=1in,right=1in,bottom=1in,letterpaper]{geometry}

\usepackage{amsmath,amssymb,amsfonts,color,turnstile,mathrsfs,bm}
\usepackage{dsfont}
\usepackage{booktabs}
\usepackage{tabularx}
\usepackage{multirow}
\usepackage{subcaption}
\usepackage{wrapfig}
\usepackage{enumitem}
\usepackage{comment}
\usepackage{setspace}

\usepackage{graphicx}
\usepackage{pgffor}
\usepackage{epstopdf}
\usepackage[numbers,sort&compress]{natbib}
\usepackage{url}
\usepackage[hidelinks]{hyperref}

\usepackage[labelfont=bf,format=plain,justification=raggedright,singlelinecheck=false]{caption}
\usepackage[symbol]{footmisc}

\usepackage[ruled,vlined]{algorithm2e}

\usepackage{siunitx}
\usepackage{tikz}
\usetikzlibrary{
  arrows.meta,
  positioning,
  calc,
  shapes.misc,
  shapes.geometric,
  shapes.multipart,
  fit,
  matrix,
  shadows.blur,
  backgrounds
}

\usepackage{import}
\usetikzlibrary{quotes,arrows.meta}
\usetikzlibrary{positioning}

\usepackage{Ball}
\usepackage{Box}
\usepackage{RightBandedBox}

\tikzset{
  arrow/.style = {-{Stealth[length=1.6mm]}, line width=1.2pt,
                  rounded corners, line cap=round},
  data/.style  = {draw, rectangle, rounded corners=1pt,
                  line width=1pt,
                  fill=cyan!15, blur shadow,
                  minimum width=18mm, minimum height=9mm,
                  align=center, font=\sffamily\normalsize},
  proc/.style  = {draw, rectangle, rounded corners=2pt,
                  line width=1pt,
                  fill=green!20, blur shadow,
                  minimum width=18mm, minimum height=9mm,
                  align=center, font=\sffamily\normalsize},
  attn/.style  = {draw, ellipse,
                  line width=1pt,
                  fill=yellow!25, blur shadow,
                  minimum width=20mm, minimum height=10mm,
                  align=center, font=\sffamily\normalsize},
  merge/.style = {draw, diamond, aspect=2,
                  line width=1pt,
                  fill=orange!30, blur shadow,
                  minimum width=22mm, inner sep=1pt,
                  align=center, font=\sffamily\normalsize},
  backbone/.style = {fill=blue!7,   rounded corners=3pt, draw=none},
  promptbg/.style = {fill=orange!10,rounded corners=3pt, draw=none}
}

\usepackage{xcolor}
\definecolor{ConvColor}{rgb}{1.0,0.78,0.20}
\definecolor{ConvReluColor}{rgb}{1.0,0.60,0.20}
\definecolor{DeconvColor}{rgb}{0.3,0.6,1.0}
\definecolor{ResColor}{rgb}{0.55,0.35,0.9}
\definecolor{FlatColor}{rgb}{0.65,0.65,0.65}
\definecolor{FCColor}{rgb}{0.30,0.55,0.95}
\definecolor{PoolColor}{rgb}{0.8,0.1,0.1}

\tikzset{
  ae-connection/.style = {ultra thick, draw=black!70,
                          every node/.style={sloped,allow upside down},
                          opacity=0.8},
  ae-fc/.style     = {fill=FCColor,     bandfill=ConvReluColor},
  ae-conv/.style   = {fill=ConvColor,   bandfill=ConvReluColor},
  ae-deconv/.style = {fill=DeconvColor},
  ae-res/.style    = {fill=ResColor},
  ae-flat/.style   = {fill=FlatColor}
}

\makeatletter
\newif\if@abbrvbib
\@abbrvbibtrue  % Or set to false if you prefer plainnat
\makeatother

\allowdisplaybreaks

\usepackage{amsmath,amsfonts,bm}

\def\1{\bm{1}}

\DeclareMathAlphabet{\mathsfit}{\encodingdefault}{\sfdefault}{m}{sl}
\SetMathAlphabet{\mathsfit}{bold}{\encodingdefault}{\sfdefault}{bx}{n}

\def\gA{{\mathcal{A}}}

\def\gF{{\mathcal{F}}}
\def\gG{{\mathcal{G}}}
\def\gH{{\mathcal{H}}}

\def\gM{{\mathcal{M}}}
\def\gN{{\mathcal{N}}}

\def\gP{{\mathcal{P}}}

\def\gR{{\mathcal{R}}}

\def\gX{{\mathcal{X}}}
\def\gY{{\mathcal{Y}}}

\def\sR{{\mathbb{R}}}

\newcommand{\E}{\mathbb{E}}

\newcommand{\R}{\mathbb{R}}

\DeclareMathOperator*{\argmin}{arg\,min}

\newcommand{\imgsep}{1mm}
\newcommand{\op}{\textrm{op}}

\newcommand{\trace}{\textrm{tr}}
\newcommand{\G}{\mathcal{G}}
\newcommand{\f}{\mathcal{F}}

\newcommand{\nn}{\mathcal{NN}}

\usepackage{amsthm}
\newtheorem{thm}{Theorem}
\newtheorem{prop}{Proposition}
\newtheorem{defn}{Definition}
\newtheorem{rmk}{Remark}
\newtheorem{assum}{Assumption}
\newtheorem{lem}{Lemma}
\newtheorem{cor}{Corollary}

\title{In-Context Operator Learning on the Space of Probability Measures}

\author{
Frank Cole$^{\dagger}$\textsuperscript{*}\quad
Dixi Wang$^{\ddagger}$\textsuperscript{*}\quad
Yineng Chen$^{\ddagger}$ \quad
Yulong Lu$^{\dagger}$ \quad
Rongjie Lai$^{\ddagger}$ \\
\\
$^{\dagger}$ School of Mathematics, University of Minnesota, Minneapolis, MN 55455, USA \\
$^{\ddagger}$ Department of Mathematics, Purdue University, West Lafayette, IN 47907, USA \\
}

\date{} % leave empty for no date

\begin{document}
\maketitle
\footnotetext[1]{Equal contribution.}

\begin{abstract}
We introduce \emph{in-context operator learning on probability measure spaces} for optimal transport (OT). The goal is to learn a single solution operator that maps a pair of distributions to the OT map, using only few-shot samples from each distribution as a prompt and \emph{without} gradient updates at inference. We parameterize the solution operator and develop scaling-law theory in two regimes. In the \emph{nonparametric} setting, when tasks concentrate on a low-intrinsic-dimension manifold of source--target pairs, we establish generalization bounds that quantify how in-context accuracy scales with prompt size, intrinsic task dimension, and model capacity. In the \emph{parametric} setting (e.g., Gaussian families), we give an explicit architecture that recovers the exact OT map in context and provide finite-sample excess-risk bounds. Our numerical experiments on synthetic transports and generative-modeling benchmarks validate the framework.
\end{abstract}

\noindent\textbf{Keywords:} Optimal Transport; In-Context Operator Learning; Transformer.

\section{Introduction}
Measure transport on spaces of probability measures has become a unifying tool across applied mathematics and machine learning. By endowing distributions with geometric structure, transport provides principled ways to compare, interpolate, and map probability measures. In particular, optimal transport (OT) \citep{villani2008optimal} has emerged as a foundational tool in machine learning and applied mathematics, facilitating powerful geometric frameworks for comparing and transforming probability distributions. The computation of optimal transport maps—functions that minimize the transportation cost between probability distributions—has profound implications in diverse applications, including generative modeling \citep{pooladian2023multisample, rout2021generative}, image processing \citep{papadakis2015optimal}, data assimilation \citep{feyeux2018optimal}, and statistical inference \citep{chewi2024statistical}.

In parallel, in-context learning (ICL) \citep{garg2022can, dong2024survey}, first popularized through the emergence of large language models (LLMs) such as GPT-3 \citep{floridi2020gpt}, has revolutionized the way models adapt and generalize to new tasks without explicit fine-tuning. By conditioning on a small set of examples provided in the input prompt, these models demonstrate remarkable flexibility in performing diverse natural language tasks, ranging from text classification and sentiment analysis to translation and summarization. The ability to generalize from limited context alone, without gradient updates, has sparked significant interest in both theoretical and empirical studies aimed at understanding and extending the boundaries of this phenomenon.
Recently, the paradigm of ICL has extended beyond its roots in natural language processing (NLP), finding impactful applications in scientific computing. Notably, researchers have begun exploring its potential in the context of solving complex problems governed by partial differential equations (PDEs). By leveraging a small set of observed solutions or examples within a prompt, models can rapidly infer accurate predictions for PDE solutions, enabling efficient approximations of computationally intensive tasks across various scientific domains.

Motivated by the synergy between ICL and operator learning in scientific domains, this paper investigates the novel integration of in-context learning principles with optimal transport map estimation. Specifically, we explore how models trained on a variety of distribution pairs can rapidly adapt, through minimal contextual information, to accurately estimate optimal transport maps between previously unseen distribution pairs. 

To this end, we propose \textbf{in-context operator learning on probability measure spaces}. Given a query space $\gX$ and label space $\gY$, we formalize each task as a pair of distributions $(\rho_0,\rho_1)\in \mathcal{P}(\mathcal{X})\times\mathcal{P}(\mathcal{Y})$ drawn from a meta-distribution over tasks. Denote a \emph{prompt} as few-shot samples $X=\{x_i\}_{i=1}^k\sim \rho_0$ and $Y=\{y_i\}_{i=1}^k\sim \rho_1$, the goal is to produce a map $\widehat T_{X,Y}:\mathcal{X}\to\mathcal{Y}$ that approximates the OT map $T_{\rho_0\to\rho_1}$ in context. Conceptually, we learn a \emph{solution operator}
\begin{equation}
\mathscr{T}:\ \mathcal{P}(\mathcal{X})\times\mathcal{P}(\mathcal{Y})\longrightarrow \mathrm{Map}(\mathcal{X},\mathcal{Y}),
\qquad
(\rho_0,\rho_1)\longmapsto \mathscr{T}(\rho_0,\rho_1)\approx T_{\rho_0\to\rho_1},
\label{eq:operator}
\end{equation}
and train it \emph{amortized} over tasks so that, at test time, a few samples suffice to specify the task and invoke the appropriate transport.

One of the important ingredients in our framework is the choice of parameterization for the solution operator. Because we represent probability measures by \emph{samples}, the operator must (i) be \emph{permutation equivariant} with respect to the ordering of samples within each set, (ii) admit a \emph{variable-length} interface so predictions are well-defined for any number of samples, and (iii) be \emph{resampling-consistent}, stable under i.i.d.\ resampling from the same underlying measures and improving as sample size grows. \emph{Transformers} provide a natural choice: self- and cross-attention operate on sets without fixed ordering, yielding permutation equivariance; the attention mechanism natively scales to arbitrary cardinalities; and invariant/equivariant pooling (e.g., token averaging or attention-based aggregation) promotes resampling consistency by approximating integrals over empirical measures. Moreover, \emph{after training}, the attention coefficients implement \emph{adaptive, input-dependent weighted averages} of context tokens. Thus different inputs induce different averaging kernels that specialize the computation to the current task, precisely the behavior required for \emph{solution-operator} learning. This offers a flexible way to encode pairwise transport costs, enabling a single architecture to couple empirical measures and output set-to-set transport predictions.

To delineate when and why in-context transport works, we develop a scaling-law analysis in two regimes. In the \emph{nonparametric} case, we assume the collection of tasks concentrates on a low-intrinsic-dimensional \emph{task manifold} $\gM\subset \mathcal{P}(\mathcal{X})\times\mathcal{P}(\mathcal{Y})$. We show how generalization depends on (i) intrinsic task dimension, (ii) prompt size, and (iii) model capacity, yielding quantitative sample- and task-complexity bounds for in-context transport. In the \emph{parametric} case (e.g., Gaussian families), we construct an explicit transformer that \emph{recovers the exact OT map in context} and derive finite-sample in-context transport map estimation bounds.

Finally, we validate our framework with numerical experiments on synthetic setups and generative-modeling benchmarks, demonstrating accurate, prompt-driven transport consistent with the predicted scaling behavior.

\paragraph{Summary of main results} 
We highlight the main contributions of the paper as follows.

\begin{itemize}
    \item We cast OT solution operator estimation as an \emph{in-context operator learning} problem on probability measure spaces, learning a single operator $\mathscr{T}$ that maps $(\rho_0,\rho_1)$ to $T_{\rho_0\to\rho_1}$ from few-shot prompts.
    We then carry out approximation-theoretic and statistical analyses of this ICL problem under both nonparametric and parametric assumptions on the source–target pairs of measures.
    \item  We propose a transformer for set-to-set, variable-size prompts that performs contextual coupling between empirical measures and outputs sample-size-agnostic transport predictions.
    \item In the nonparametric setting, under the assumption that the task distribution is supported on a low-dimensional manifold (see Assumption \ref{ass:mlfd}), we establish generalization error bounds for the ICL of transport maps, providing quantitative estimates for both sample complexity and task complexity (see Theorem \ref{thm:general_icl}).
    \item In the parametric setting, where both the reference and target measures are centered Gaussians (see Assumption \ref{assum: taskdistr}), we construct an explicit transformer architecture that learns the exact OT map in context. Furthermore, we derive a quantitative generalization error bound for the resulting excess loss; see Theorem \ref{thm: lossgap}.
    \item We demonstrate the predictive performance of our ICL model through applications to generative modeling on both synthetic and real-world datasets.
\end{itemize}

\subsection{Related work}

\paragraph{Statistical aspects of optimal transport.} The statistical estimation of optimal transport maps has received significant attention in recent years. In this line of work, the most commonly analyzed estimators are plugin estimators based on replacing the target measures in the optimal transport problem with their empirical counterparts. The plugin framework encompasses several estimators such as one-nearest neighbor and wavelet-based estimators \citep{manole2024plugin}, barycentric projections \citep{deb2021rates}, solutions to the empirical semidual problem \citep{divol2025optimal}, and estimators based on entropic optimal transport \citep{mena2019statistical,pooladian2021entropic}. Many of the aforementioned estimators are shown to be minimax optimal under regularity assumptions on the optimal transport map. We refer to \citep{chewi2024statistical} for an overview of statistical optimal transport. These works study transport map estimators for a single pair of measures; consequently, the estimators must be recomputed when the measures are changed. Meta optimal transport \citep{amos2022meta} was proposed to solve multiple similar optimal transport problems for different choices of reference measures. Otherwise, their is little overlap with our work, since they approach problem using amortized optimization as opposed to in-context learning.

\paragraph{In-context learning.} The concept of in-context learning over a class of functions was formalized in \citep{garg2022can}, where it was demonstrated that transformers can learn simple function classes in-context. In-context learning of learning models has since attracted significant research, both from the empirical perspective \citep{von2023transformers,von2023uncovering,vladymyrov2024linear} and the theoretical perspective \citep{ahn2023transformers,mahankali2023one,zhang2023trained,kwon2025out, cole2024context, lu2025asymptotic}. Beyond linear functions, several works have studied in-context learning over more complex, nonlinear function spaces, both via explicit constructions \citep{kim2024transformers, bai2023transformers, guo2023transformers, shen2025understanding, li2025transformers} and direct optimization analysis \citep{kim2024transformers2, oko2024pretrained, yang2024context, li2024one}. While most works have studied in-context learning in the IID setting, several recent works have focused on non-IID data models \citep{sander2024transformers, zheng2024mesa, sander2024towards, wu2025transformers, cole2025context}.

\paragraph{In-context learning and operator learning.} In-context operator learning was introduced in \citep{yang2023context}, where it was demonstrated that pre-trained transformers could generalize to solve diverse families of differential equations. Concerning in-context learning on probability measure spaces, our methodology is based on the work of \citep{huang2024unsupervised}. Note that operator learning on the space of probability measures was studied in a different setting in \citep{bach2025learning}.

\subsection{Notation}
Let $\gX = \sR^d$ be the query space and  $\gY = \sR^d$ be the label space. Write $\mathrm{Map}(\R^d,\R^d)$ to denote the space of maps from $\R^d$ to $\R^d$. We denote by $\mathcal{P}_2(\R^d)$ the set of Borel probability measures on $\R^d$ with finite second moment.
Denote the Euclidean metric in $\R^d$ as $d(x,x') = \|x-x'\| := \sqrt{\sum |x_i-x'_i|^2}$.  We endow $\mathcal{P}_2(\R^d)$ with Wasserstein-2 distance defined as $W_2(\mu,\nu) := \inf_{\gamma\in\Gamma(\mu,\nu)} \Big(\int_{\R^d} d(x,x')^2\mathrm{d}\gamma(x,x')\Big)^{1/2},$ where $\Gamma(\mu,\nu)$ is the set of all couplings of $\mu,\nu\in \mathcal{P}_2(\sR^d)$.
A task is a pair of distribution $(\rho_0, \rho_1)$ that resides in a low-dimensional task manifold $\gM$, i.e. $(\rho_0, \rho_1)\in \gM \subset \mathcal{P}_2(\R^d)\times \mathcal{P}_2(\R^d)$. Denote by $\mathcal{P}(\mathcal{M})$ the set of Borel probability measures supported on $\mathcal{M}$, in which each measure is a distribution of distributions. We use the standard big-$O$ notation: if $f$ and $g$ are nonnegative functions we write $f(x) = O(g(x))$ to denote that there is a constant $C > 0$ such that $f(x) \leq Cg(x)$. We write $f = \Omega(g)$ if $g^{-1} = O(f^{-1}).$ We use big-$\tilde{O}$ notation to indicate that we are ignoring log factors. More notations can be found in Appendix \ref{app: notation}.

\subsection{Organization} The paper is organized as follows. In Section \ref{sec: probsetup}, we introduce the operator learning problem for optimal transport and formalize it mathematically as an in-context learning problem. We also introduce the relevant population and empirical risk functions, and we explicitly describe our assumptions on the nonparametric and parametric setups (see Sections \ref{sec:non-parametricsetup} and \ref{sec: parametric} respectively). In Section \ref{sec: results}, we carefully state our main theoretical results, including all relevant assumptions we make. In Section \ref{sec: nonparametricproofs}, we sketch the proofs of the results in nonparametric setting (see Appendix \ref{appendix:nonparametricproofs} for full details) and in Section \ref{sec: gaussiantogaussian}, we sketch the proofs of the results under parametric setting (see Appendix \ref{app: parametricproofs} for full details); auxiliary lemmas are deferred to Appendix \ref{sec: auxlemmas}. In Section \ref{sec: numerics}, we present the setting and results of our numerical experiments, both for synthetic data problems and generative modeling benchmarks; further experimental details are deferred to Appendix \ref{Appendix_num}. In Section \ref{sec: conclusion}, we conclude by discussing several avenues for future work.

\section{Problem setup}
\subsection{Operator learning as in-context learning for optimal transport}\label{sec: probsetup}

Optimal transport (OT) provides a principled framework for mapping one probability distribution to another in a way that minimizes a transport cost. Specifically, given two probability distributions $\rho_0, \rho_1 \in \mathcal{P}_2(\R^d)$, the quadratic-cost Monge OT problem seeks a measurable map $T: \R^d \rightarrow \R^d$ satisfying $T_{\#}\rho_0 = \rho_1$ that minimizes the transportation cost:
\[
T_{\rho_0, \rho_1} \in \argmin_{T: T_{\#} \rho_0 = \rho_1} \int_{\R^d} \|T(x) - x\|^2 \, \rho_0(dx).
\]
Under mild assumptions (e.g., absolute continuity of $\rho_1$ with respect to Lebesgue measure), the Monge problem admits a unique solution. We define the OT solution operator $\mathcal{G}_{OT}$ as the mapping that assigns to each pair of distributions $(\rho_0, \rho_1)$ to the optimal transport map between $\rho_0$ and $\rho_1.$
\begin{defn}
    The \textbf{OT solution operator} is the map
    \begin{equation*}
        \mathcal{G}_{OT}: \mathcal{P}_2(\R^d) \times \mathcal{P}_2(\R^d)  \rightarrow \mathrm{Map}(\R^d,\R^d), \quad \mathcal{G}_{OT}(\rho_0, \rho_1) = T_{\rho_0, \rho_1},
    \end{equation*}
    whenever the Monge problem is well-posed for $(\rho_0, \rho_1)$.
\end{defn}

The classical setting of optimal transport involves solving the Monge problem for a fixed pair $(\rho_0, \rho_1)$, often requiring expensive computations. A recent paradigm shift aims to learn a {\it single model} that can generalize across multiple such tasks, i.e., different pairs of input-output distributions. This motivates the study of {\it solution operator learning}, where the goal is to approximate the map $\mathcal{G}_{OT}$ over a collection of transport tasks drawn from a distribution over distributions.

Our objective is to study the learnability of the OT solution operator. We assume that the reference and target measures $(\rho_0,\rho_1)$ are drawn from a distribution on distributions $\mu \in \mathcal{P} \left(\mathcal{P}_2(\R^d) \times \mathcal{P}_2(\R^d) \right)$, which we refer to as the \textit{task distribution}, and we aim to learn the operator $\mathcal{G}_{OT}$ on the support of $\mu$. Our work is inspired by the methodology of \citep{huang2024unsupervised}, which approximates $\mathcal{G}_{OT}$ by solving the optimization problem
$$ \min_{\mathcal{G}} \E_{(\rho_0,\rho_1) \sim \mu} \left[ \E_{x \sim \rho_0}[\mathcal{G}(\rho_0,\rho_1;x) - x\|^2] + \lambda \textrm{MMD}^2 \left(\mathcal{G}(\rho_0,\rho_1)_{\#}\rho_0, \rho_1 \right) \right],
$$
which replaces the hard constraint of the Monge problem with a soft constraint using Maximum Mean Discrepancy(MMD), see Section \ref{sec:non-parametricsetup} for definition. As $\lambda \to \infty$, the constraint asymptotically recovers the Monge constraint.
In practice, however, any computationally tractable function $\mathcal{G}$ cannot act directly on the infinite-dimensional objects $(\rho_0, \rho_1)$. Instead, we study the minimization of the modified population loss given by
\begin{equation}\label{general_G}
   \mathcal{G}^{\dagger} = \arg\min_{\mathcal{G}} \mathcal{R}(\mathcal{G}) = \E_{(\rho_0,\rho_1)\sim\mu}\left[\E_{x \sim \rho_0}[\|\mathcal{G}(\hat{\rho}_0,\hat{\rho}_1; x)\|^2] + \lambda \textrm{MMD}^2 \left(\mathcal{G}(\hat{\rho}_0,\hat{\rho}_1)_{\#}\rho_0,\rho_1 \right) \right],
\end{equation}
where $\hat{\rho}_0 = \frac{1}{s} \sum_{j=1}^{s} \delta_{x_j}$ with $\{x_j\}^s_{j=1}\stackrel{iid}{\sim}\rho_0$  and $\hat{\rho}_1 = \frac{1}{s} \sum_{j=1}^{s} \delta_{y_j}$ with $\{y_j\}^s_{j=1}\stackrel{iid}{\sim}\rho_1$ are empirical counterparts to the measures $\rho_0$ and $\rho_1$.

This formulation of the problem is intimately related to \textit{in-context learning} -- the ability of pre-trained models to correctly solve new learning problems by conditioning on a short sequence of input-output pairs (referred to as \textit{context}) from the problem. In our setting, context comprises the empirical measures $(\hat{\rho}_0,\hat{\rho}_1)$ which are used by the predictor $\mathcal{G}$ to infer the correct learning task. The query is a new point $x \sim \rho_0$. This perspective allows us to understand solution operator learning in the space of probability measures via in-context learning, the latter of which has been the subject of much recent research. In particular, the problem of minimizing $\mathcal{R}$ can also be understood as in-context learning the family of optimal transport maps defined by the task distribution $\mu$. A fundamental challenge in the ICL of optimal transport maps is the absence of labeled data; this distinguishes our study from many existing works on ICL, which focus on classification or regression \citep{zhang2023trained, ahn2023transformers, mahankali2023one}.

We consider two settings for analyzing the statistical and computational properties of the empirical risk minimizer for the above problem: 
\begin{enumerate}
    \item \textbf{Nonparametric assumptions:} In this setting, we do not assume that the data measures $(\rho_0,\rho_1)$ admit a parametric form; instead, we assume that the predictor $\mathcal{G} \in \mathcal{F}$ has \textbf{Lipschitz regularity}. Without additional structure, it is generally beyond the capacity of neural networks to directly learn the infinite-dimensional task distribution $\mu \in \mathcal{P}(\mathcal{P}_2(\R^d) \times \mathcal{P}_2(\R^d))$ due to the curse of dimensionality. We therefore assume that $\mathcal{M} := \textrm{supp}(\mu)$ is a finite-dimensional Riemannian manifold, which we call the \textbf{task manifold}. Under these assumptions, we prove non-asymptotic sample complexity guarantees for excess loss, 
    which exponentially depend on the intrinsic dimension of the task manifold; see Theorem \ref{thm:general_icl}.  
    \item \textbf{Parametric assumptions:} In this setting, we assume that $(\rho_0,\rho_1)$ admit a parametric form. More specifically, we assume a one-sample problem, where the measure $\rho_0$ is fixed as the standard $d$-dimensional Gaussian measure, and the measure $\rho_1 \sim \mu$ is a non-isotropic Gaussian with a random covariance matrix $\Sigma.$ In this setting, the ground truth solution operator can be explicitly described, and we parameterize it using a simplified transformer architecture. Under these assumptions, we prove dimension-free sample complexity rates for the loss function. In addition, we prove a quantitative, dimension-free recovery bound on the solution operator; see Theorems \ref{thm: lossgap},\ref{thm: TPGE}. 
\end{enumerate}

\subsection{Non-parametric setup}
\label{sec:non-parametricsetup}
The space $(\gP_2(\R^d),W_2)$ is a complete separable metric space, where $W_2$ distance induces an Otto Riemann structure and $\gP_2(\R^d)$ can be formally interpreted as an infinite dimensional Riemannian manifold. See \citep{ambrosioGradientFlowsMetric2008} for rigorous arguments. In practice, tasks of interest often exhibit low-dimensional structures. To capture this, we assume that the distribution of task distributions $\mu \in \mathcal{P}(\mathcal{P}_2(\R^d) \times \mathcal{P}_2(\R^d))$ is supported on a low-dimensional manifold $\gM\subset\mathcal{P}_2(\R^d) \times \mathcal{P}_2(\R^d)$, which we refer to as the task manifold. 

\begin{assum}[Low-dimensional task manifold]\label{ass:mlfd}
     The task manifold $\mathcal{M}\subset \mathcal{P}_2(\R^d)\times\mathcal{P}_2(\R^d)$ is a $d_{\gM}$-dimensional compact Riemannian submanifold with reach $\tau_\gM>0$.
\end{assum} 
Parametric family, e.g. Gaussian family, determined by a parameter space $\Theta\subset \R^{d_{\theta}}$ satisfies Assumption \ref{ass:mlfd}. Moreover, empirical evidence suggests that many high-dimensional images, when viewed as discrete probability distributions, concentrate near a low-dimensional manifold in the probability space without explicit parametric structure \citep{pope2021}. 

Let $\mu\in\mathcal{P}(\mathcal{M})$ be the uniform distribution on $\gM$. Given a pair of task $(\rho_0,\rho_1)\sim \mu$, we generate a prompt consisting $s$ number of samples $\{x_i\}^s_{i=1}\stackrel{iid}{\sim}\rho_0,~\{y_i\}^s_{i=1}\stackrel{iid}{\sim}\rho_1$. Denote the empirical measures $\hat\rho_0 := \frac{1}{s}\sum^s_{i=1}\delta_{x_i},\hat\rho_1:=\frac{1}{s}\sum^s_{i=1}\delta_{y_i}$, we consider the minimizer within a function class $\gF$ that satisfies a Lipschitz condition, see Section \ref{sec:general}. 
The ICL population risk restricted on $\gF$ is 
\begin{equation}\label{def:iclpopulation}
    \gG^* = \argmin_{\gG\in\mathcal{F}} \mathcal{R}(\gG) := \mathbb{E}_{(\rho_0,\rho_1)\sim\mu}\mathbb{E}_{x\sim\rho_0}\big\|\gG(\hat{\rho}_0,\hat{\rho}_1;x)-x\big\|^2_2  + \lambda\mathrm{MMD}^2\big(\gG(\hat{\rho}_0,\hat{\rho}_1)_{\#} \rho_0,\rho_1\big),
\end{equation}
here $\lambda_1,\lambda_2\geq0$ are hyperparameters. The expectation in \eqref{def:iclpopulation} is not over different choices of prompts for each task, which is consistent with \eqref{general_G}. Moreover, the number of possible prompt choices does not appear in the generalization error bound. A similar observation can be found in \citep{mroueh2023towards}.  
The discrepancy $\mathrm{MMD}^2$ is defined as the following: let
$p,q\in\gP(\R^d)$, 
\begin{align}\label{def_mmd}
    \mathrm{MMD}^2(p,q) = \Big[\sup_{\|f\|_{\gH_k}\leq 1} (\mathbb{E}_x[f(x)] - \mathbb{E}_y [f(y)] ) \Big]^2,
\end{align}
where the supremum is taken oven the unit ball in a reproducing kernel Hilbert space $\gH_k$ induced by a kernel function $k(\cdot,\cdot):\R^d\times\R^d\to \R$.
Empirically, the model is presented $N$ number training tasks $\{\rho^k_0, \rho^k_1\}^N_{k=1}\stackrel{iid}{\sim} \mu$ with associated prompt sets $\{\hat\rho^k_0, \hat\rho^k_1\}^N_{k=1}$ and samples $\Tilde{\rho}^k_0 = \frac{1}{n}\sum^n_{j=1}\delta_{x^k_j}$, $\Tilde{\rho}^k_1 = \frac{1}{n}\sum^n_{j=1}\delta_{y^k_j}$. Denote by $\Lambda = \{\rho^k_0,\rho^k_1,\hat{\rho}^k_0,\hat{\rho}^k_1,\tilde{\rho}^k_0,\tilde{\rho}^k_1 \}_{k=1}^N $, we minimize the empirical risk: 
\begin{align}\label{def:emp_general}
\gG_{\Lambda}^* = \argmin_{\gG\in\gF}\mathcal{R}_{\Lambda}(\gG)= \frac{1}{N}\sum^N_{k=1}\Bigg[\frac{1}{n}\sum^n_{j=1}\|\gG(\hat{\rho}^{k}_0,\hat{\rho}^{k}_1;x^k_j)-x^k_j\|^2_2 + \lambda\mathrm{MMD}^2_u(\gG(\hat{\rho}^{k}_0,\hat{\rho}^{k}_1)_\# \Tilde{\rho}^k_0, \Tilde{\rho}^k_1)\Bigg].
\end{align}
Here, 
$\mathrm{MMD}^2_u$ is an unbiased U-statistics estimator of $\mathrm{MMD}^2$. Given $X=\{x_i\}^m_{i=1}\stackrel{iid}{\sim}p$, $Y=\{y_i\}^m_{i=1}\stackrel{iid}{\sim}q$,
\begin{align}\label{def_mmdu}
    \mathrm{MMD}^2_u(X,Y) = \frac{1}{m(m-1)}\sum^m_{i\neq j}k(x_i,x_i)+k(y_i,y_j)-k(x_i,y_j)-k(x_j,y_i).
\end{align}

Generalization under the non-parametric setup is measured by the generalization gap
\begin{align*}
    \mathbb{E}_{\Lambda}|\mathcal{R}(\mathcal{G}_{\Lambda}^{\ast}) - \mathcal{R}(\mathcal{G}^{\ast})|,
\end{align*}
where the expectation is taken for training tasks, prompts, samples and  test samples.

\subsection{Parametric setup}\label{sec: parametric}

In this section, we study the problem defined in Section \ref{sec:non-parametricsetup} under parametric assumptions on the measures.
Specifically, we consider the setting where the base measure $\rho_0 = \mathcal{N}(0,I)$ is fixed across tasks to be the standard Gaussian, and the target measure $\rho_1 = \mathcal{N}(0,\Sigma)$ is a Gaussian with a task-varying covariance matrix $\Sigma \in \R^{d \times d}$. The task manifold $\mathcal{M} \subset \R^{d \times d}$ determines the set of admissible covariance matrices. Since a centered normal distribution is completely determined by its covariance, we abuse notation and view $\mu$ as a probability distribution directly on $\Sigma.$ We put some standard assumptions on the distribution over $\Sigma$.

\begin{assum}\label{assum: taskdistr}
    Let $\mu$ denote the probability distribution on covariance matrices $\Sigma$. Then there exists a probability distribution $\mu_{diag}$ on $\R_+^d$ and an orthogonal matrix $U \in \R^{d \times d}$ such that $\Sigma \sim \mu$ if and only if $U^T \Sigma U = \textrm{diag}(\sigma_1^2, \dots, \sigma_d^2)$ and $(\sigma_1^2, \dots, \sigma_d^2) \sim \mu_{diag}.$ In addition, there exist constants $\sigma_{\max}^2$ and $\sigma_{\min}^2$ such that $\max_i \sigma_i^2 \leq \sigma_{\max}^2$ and $\min_i \sigma_i^2 \geq \sigma_{\min}^2$.
\end{assum}
Assumption \ref{assum: taskdistr} states that all matrices in the support of $\mu$ are simultaneously diagonalizable by a common orthogonal matrix, and that their eigenvalues are uniformly bounded from above and below. Additionally, since the base measure $\rho_0 = \mathcal{N}(0,I)$ is fixed across tasks and easy to sample, we focus on the one-sample setting where only $\rho_1$ is discretized and all expectations with respect to $\rho_0$ appearing in the loss function are computed exactly.

\paragraph{Loss function and generalization}\label{subsec: lossfcn}

Due to the parametric nature of this problem, we consider a subset of operators $\{\mathcal{G}_{\theta}: \theta \in \Theta\}$ parameterized by a finite-dimensional parameter space $\Theta$; see Paragraph \ref{subsec: TF} for the details of the architecture. To simplify the notation, we will view the loss as a function of the parameter $\theta$ rather than the operator $\mathcal{G}_{\theta}$. Since the true OT map between Gaussians is linear, we are motivated to consider a parameterization such that the dependence of the operator $\mathcal{G}_{\theta}$ on the query is linear. Paragraph \ref{subsec: TF} explains how to construct a transformer-based architecture which preserves this inductive bias. To emphasize the linearity, we will often express the parameterized in-context mapping as
\begin{equation}\label{eq: iclmapping}
(y_1, \dots, y_n;x) \mapsto A_{n,\theta}x,\end{equation}
where $A_{n,\theta} =  \mathcal{G}_\theta(\{y_i\}_{i=1}^n)\in \R^{d \times d} $ is the matrix representation, which depends both on the parameter $\theta$ and the context $y_1, \dots, y_n$. Meanwhile, since Gaussian measures are completely determined by their first two moments, it is natural to define the MMD using a kernel that emphasizes the first two moments. To this end, we use the quadratic kernel given by $k(x,y) = \langle x, y \rangle^2.$ It is easy to check that the MMD with respect to $k$ between two measures $p_1$ and $p_2$ having finite second moments is given by
$ \textrm{MMD}^2(p_1,p_2) = \left\|\E_{p_1}[xx^T]-\E_{p_2}[yy^T] \right\|_F^2.$
This allows us to write the risk in the more tractable form
\begin{equation}
    \mathcal{R}(\theta) = \E_{\Sigma \sim \mu, (y_1, \dots, y_{2n}), \sim \mathcal{N}(0,\Sigma)} \left[ \E_{x \sim \mathcal{N}(0,I)}\left\|A_{n,\theta} x - x \right\|^2 + \lambda \left\|A_{n,\theta}^2 - \Sigma_n \right\|_F^2 \right],
\end{equation}
where the expectation is defined by first sampling $\Sigma$ from $\mu$ and then sampling $y_1, \dots, y_{2n}$ from $\mathcal{N}(0,\Sigma);$ the first $n$ samples are used as the input of the in-context mapping determined by $\theta$, while the latter $n$ samples are used to approximate $\mathcal{N}(0,\Sigma)$ in the computation of the MMD provided by $\Sigma_n = \frac{1}{n} \sum_{i=n+1}^{2n} y_i y_i^T$. Discretizing the expectation with respect to $\mu$ gives rise to an empirical risk functional; given $N$ iid samples $\Sigma^{(1)}, \dots, \Sigma^{(N)} \sim \mu$, defining $\rho_1^{(i)} = \mathcal{N}(0,\Sigma^{(i)}),$ and, for each $1 \leq i \leq N$, sampling $y_1^{(i)}, \dots, y_{2n}^{(i)} \sim \rho_1^{(i)}$. As in Section \ref{sec:non-parametricsetup}, we denote by $\Lambda$ the training set and $\mathcal{R}_{\Lambda}$ the empirical risk. We are interested in the generalization properties of the empirical risk estimator.
\begin{equation}\label{eq: empriskminimizer}
  \widehat{\theta} = \argmin_{\theta\in\Theta}  \mathcal{R}_{\Lambda}(\theta) = \argmin_{\theta \in \Theta} \frac{1}{N} \sum_{i=1}^{N} \left(\E_{x \sim \mathcal{N}(0,I)}\left\|A_{n,\theta}^{(i)}x-x \right\|^2 + \lambda \left\|\left(A_{n,\theta}^{(i)}\right)^2 - \Sigma_n^{(i)} \right\|^2  \right),
\end{equation}
where
$ A_{n,\theta}^{(i)} = Q \cdot \frac{1}{n} \sum_{k=1}^{n} (\phi(y_k^{(i)}))  (\phi(y_k^{(i)}))^T \cdot Q^T \in \R^{d \times d}
$
and
$ \Sigma_n^{(i)} = \frac{1}{n} \sum_{k=n+1}^{2n} (y_k^{(i)}) (y_k^{(i)})^T \in \R^{d \times d}.
$
There are two relevant measures of the generalization of $\widehat{\theta}$. First, we consider the \textit{excess loss}, defined as
\begin{align}\label{eq: excessloss}
    \mathcal{R}(\widehat{\theta}) - \mathcal{R}(\mathcal{G}^{\dagger}),
\end{align}
where we recall that $\mathcal{G}^{\dagger}$ is the minimizer of $\mathcal{R}$ over all measurable functions of the context $(y_1, \dots, y_n)$ and the query $x$. In particular, unlike the generalization gap, the excess loss takes into account the approximation error of the architecture. Our second measure of generalization error is the \textit{transport map generalization error}
\begin{equation}
    \E_{\Sigma \sim \mu, (y_1, \dots, y_n) \sim \mathcal{N}(0,\Sigma)} \left[ \left\|\widehat{A}_n - \Sigma^{1/2} \right\|^2 \right].
\end{equation}

\paragraph{Transformer hypothesis class}\label{subsec: TF}

Transformers are particularly well-suited to approximate operators on the space of probability measures because of their ability to take sequences of arbitrary length as input, as well as their permutation invariance. The distinguishing architectural feature of transformers is the \textit{attention mechanism}. Given a matrix in $Z \in \R^{k \times T}$, viewed as an $\R^k$-valued sequence of length $T$, a softmax self-attention layer maps the sequence $Z$ to another sequence given by
\begin{equation}\label{eq: transformer} Z \mapsto Z + W_P \cdot  \textrm{softmax} \left( Z \cdot \frac{Z^T W_Q Z}{\eta(T)} \right),
\end{equation}
where $W_P, W_Q \in \R^{k \times k}$ are the learnable parameters, $\textrm{SM}$ is the softmax function (applied column-wise), and $\eta(T)$ is a normalization factor depending on the length of the sequence. The transformer architecture consists of repeatedly composing self-attention layers with pointwise feedforward layers. In this work, we study a simplified transformer architecture, consisting of a pointwise feedforward layer followed by a \textit{linear self-attention layer}, where the softmax function in Equation \eqref{eq: transformer} is replaced with the identity function. 
We use a transformer to define an in-context mapping as follows. Given $n$ samples $x_1, \dots, x_n$ from the source measure $\rho_0$, $n$ samples $y_1, \dots, y_n$ from the the target measure $\rho_1$, and a query $x \sim \rho_0$, we construct the embedding matrix 
\begin{align*}
    Z = \begin{bmatrix}
        0 & \dots & 0 & x \\
        \phi(y_1) & \dots & \phi(y_n) & 0
    \end{bmatrix},
\end{align*}
where $\phi$ is a feedforward layer. Then, we apply a linear self-attention layer to the matrix $Z$, with $k = 2d$, $T = n+1$, and the normalization $\eta(T) = T-1,$ and extract the bottom $d$ rows of the $(n+1)$-st column of the output matrix. To simplify the analysis, we adopt a specific parameterization of the attention weights: we write $W_P = \begin{bmatrix}
    0 & 0 \\ 0 & Q
\end{bmatrix}$ and $W_Q = \begin{bmatrix}
    0 & 0 \\
    Q^T & 0 
\end{bmatrix},$ where $Q \in \R^{d \times d}$ is the learnable weight matrix. The resulting in-context mapping can be explicitly written as
\begin{equation}\label{eqn: Atheta}
    (y_1, \dots, y_n,x) \mapsto \left( Q \cdot \frac{1}{n} \sum_{i=1}^{n} \phi(y_i) \phi(y_i)^T \cdot Q^T\right) \cdot x =: A_{n,\theta} \cdot x.
\end{equation}
This parameterization has the benefit of being a linear mapping with respect to the query $x.$ We assume that the feedforward layer $\phi \in \Phi(M)$ comes from the class of shallow neural networks weights bounded by $M$, where $M > 0$ is a capacity bound to be chosen precisely later (see Appendix \ref{app: neuralnet} for precise definitions). We then define our hypothesis class $\Theta$ as the class of such transformers:
\begin{equation}
   \Theta = \left\{(y_1, \dots, y_n,x) \mapsto Q \cdot \frac{1}{n} \sum_{i=1}^{n} \phi(y_i) \phi(y_i)^T \cdot Q^T \cdot x: \|Q\|_F \leq C_{\Theta}, \; \phi \in \Phi(M), \right\}.
\end{equation}

We will write $\theta = (Q,\phi)$ as a shorthand for the parameters of the architecture, and, in mild abuse of notation, we will use $\Theta$ to denote both the space of parameters $\theta = (Q,\phi)$ and the space of in-context mappings induced by the parameters. Note that the matrix $A_{n,\theta}$ 
is symmetric; this inductive bias helps preserve the structure of our problem, since Brenier's Theorem guarantees that the optimal transport map is the gradient of a convex function (hence a symmetric matrix in the Gaussian setting).

\section{Main results}\label{sec: results}
\subsection{Generalization error for nonparametric problem}\label{sec:general}

We establish a generalization error bound for the in-context OT map within a function class $\gF$ under Lipschitz regularity assumption.
Specifically, the function class $\gF$ satisfies
\begin{assum}\label{ass:bdd&lips}
\begin{enumerate}
    \item (Boundedness) There exists a constant $M_\mathcal{F}>0$ such that for any $\rho_0,\rho_1\in \gP_2(\R^d)$, 
        \begin{align*}
            \sup_{\G\in\mathcal{F}} \|\G(\rho_0,\rho_1;\cdot)\|_\infty \leq M_\mathcal{F}.
        \end{align*}
    \item (Lipschitzness) For any $\G\in\mathcal{F}$, $x, x'\in\R^d$ and  $\rho_0,\rho'_0, \rho_1,\rho'_1\in \gP_2(\R^d)$, there exists a constant $L_\mathcal{F}>0$ such that
        \begin{equation*}
            \|\G(\rho_0,\rho_1;x) - \G(\rho'_0, \rho'_1;x')\|\leq L_\mathcal{F} (d(x,x') + W_2(\rho_0, \rho'_0) + W_2(\rho_1,\rho'_1)).
        \end{equation*}
\end{enumerate}
\end{assum}

This type of Lipschitz regularity condition arises in the study of distribution-dependent SDEs and gradient flows in the space of probability measures \citep{ambrosioGradientFlowsMetric2008}. Note that our assumption is different from the setup in \citep{mroueh2023towards} as the Wasserstein distance there is on the joint distribution. Here, the Wasserstein distance considered in Assumption \ref{ass:bdd&lips} is on the marginals, which is more empirically accessible, and sufficiently serves the purpose as we compare the marginals using $\mathrm{MMD}$ for the terminal loss. Moreover, we require that both base and target measures have finite $p$-th moment, $p>2$ to control the truncation error under the setup $\gX=\gY=\R^d$. Here we choose $p=3$. 
\begin{assum}\label{ass:moment}
    The base and target measures have finite third moment, i.e. there exists a positive constant $\widetilde{M}$ such that 
    \begin{equation*}
         \int_{\mathbb{R}^d} |x|^3\rho(\mathrm{d}x) \leq \widetilde{M}.
    \end{equation*}
\end{assum}

\begin{assum}\label{ass:kernel}
    The kernel function $k(\cdot,\cdot)$ is measurable, bounded, and the functions in its associated RKHS vanish at infinity. In addition, the kernel satisfies: for all $\rho\in\gP(\R^d),$
    \[
    \rho\mapsto \int_{\R^d} k(\cdot,x)\rho(\mathrm{d}x) \qquad \text{is injective.}
    \]
\end{assum}
Kernel functions satisfying Assumption \ref{ass:kernel} are \textit{characteristic} kernels, which is a sufficient condition for a kernel to metrize the weak convergence of probability measures, hence MMD metrizes the weak convergence. Gaussian, Laplace and inverse-multiquadratic kernels supported on $\R^d$ are known to be characteristic \citep{Sriperumbudur,JMLR:v19:16-291}. 
MMD with quadratic kernel cannot distinguish two different distributions with same mean and variance, however is sufficient for characterizing the generalization gap for the Gaussian-Gaussian setup. Additionally, characteristic kernels are bounded kernels thus circumvents the need for truncation and leads to a tighter bound, see Lemma \ref{lem:terminal}. 

We first characterize the generalization gap between the empirical minimizer $\gG^*_\Lambda$ in \eqref{def:emp_general} and the minimizer $\G^*\in\mathcal{F}$ of the population risk  \eqref{def:iclpopulation} without specifying the parametrization of $\gG_\Lambda$.  
\begin{thm}[Generalization of Few-Shot In-Context Learning]\label{thm:general_icl}
    Suppose Assumptions \ref{ass:mlfd},\ref{ass:bdd&lips}, \ref{ass:moment}, \ref{ass:kernel} hold.
    Let the training tasks $(\rho^k_0, \rho^k_1)^N_{k=1}$ be drawn by sampling $\mathcal{M}$ i.i.d. according to the uniform probability measure $\mu$ on $\mathcal{M}$. 
    The few-shot empirical risk minimizer defined in \eqref{def:emp_general} satisfies 
    \begin{align}\label{ineq:fs}
        &\mathbb{E}_\Lambda|\mathcal{R}(\gG^*_\Lambda)-\mathcal{R}(\G^*)|\nonumber\\
        =& O\Big( N^{-\frac{1}{4\max\{d,d_{\gM}\}}}+N^{-\frac{1}{4}}n^{\frac{d+1}{3}}+N^{\frac{1}{8\max\{d,d_{\gM}\}}}s^{-\frac{1}{2}} + n^{\max\{-\frac{1}{6},-\frac{1}{d}\}} + s^{\max\{-\frac{1}{6},-\frac{1}{d}\}} \Big),
    \end{align}
    where $O(\cdot)$ hides the dependency on $\lambda,d_{\gM},d,M_\gF,L_\gF,\widetilde{M},\mathrm{diam}(\gM)$ and $\tau_\gM$. 
\end{thm}
Theorem \ref{thm:general_icl} characterizes the overall generalization error as the combination of two components: one arising from the optimal transport and the other from the MMD terminal cost. By assuming the task manifold has intrinsic dimension $d_{\gM}$ sitting in the infinite-dimensional density space, Theorem \ref{thm:general_icl} leverages that the statistical rate of convergence depends on intrinsic dimension of task manifold. Since the training tasks are uniformly sampled on the compact manifold $\gM$, sufficient number of tasks $N(d_{\gM})$ guarantees a good covering with high probability so that the unseen tasks during pre-training are close to some seen tasks, which leads to better generalization.

\subsection{Generalization error for parametric problem}\label{subsec: parametric}

We present our main results on the generalization error of the empirical risk minimizer $\widehat{\theta}$ and the in-context transport map estimator $\widehat{A}_n$. Our first result gives a bound on the excess loss. Before presenting it, recall that the model class $\Theta$ has two capacity parameters: the norm bound $C_{\Theta}$ on the weight matrices of the attention layer, and the bound $M$ on the path norm of the feedforward layer. We will fix the first parameter as $C_{\Theta} = \sqrt{d}$.

\begin{thm}[Excess loss estimate]\label{thm: lossgap}
    Let $\widehat{\theta}$ be the empirical risk minimizer over $\Theta$ as defined in Equation \eqref{eq: empriskminimizer}. Then, with $N$ the number of tasks, $n$ the number of samples per task, $\lambda$ the regularization parameter, and $M$ the capacity bound of the feedforward layer, we have the excess loss bound
     \begin{align*}
        &\mathcal{R}(\widehat{\theta}) - \mathcal{R}(\mathcal{G}^{\dagger}) \\
        &= \tilde{O} \left( \underbrace{\frac{(1+\lambda)M^8}{\sqrt{N}}}_{\textrm{pre-training generalization error}} + \underbrace{\left(\sqrt{\frac{1}{n}} + \frac{\lambda}{n} \right)}_{\textrm{in-context generalization error}} + \underbrace{\frac{1}{M^2} + \frac{\lambda}{M^4}}_{\textrm{approximation error}} + \underbrace{\left( \frac{1}{\lambda}\right)}_{\textrm{regularization error}} \right)
    \end{align*}
    with probability at least $1 - O \left(\frac{1}{\textrm{poly}(N)} \right)$. Above, $\tilde{O} \left( \cdot \right)$ omits factors which depend logarithmically on sample sizes or polynomially problem parameters. The same bound also holds in expectation over the training set. %\RJ{What is $\epsilon$? The 1st term is inversely proportion to $\epsilon$? }
\end{thm}
Theorem \ref{thm: lossgap} decomposes the excess loss into a sum of four terms. The first term, the \textit{pre-training generalization error}, captures the error due to a finite number of pre-training tasks in the empirical risk. The second term, the \textit{in-context generalization error}, represents the error due to having a finite number of samples in each prompt. The third term, reflecting the approximation error of our model class, decreases when the parameter $M$ tends to $\infty$. The final term, reflecting the error due to finite lambda, reflects a bias in the minimizer of the loss function. Our error decomposition is analogous to that of Theorem 4.5 in \citep{kim2024transformers}. It can be seen as a variant of the usual approximation-estimation error decomposition, where the estimation error now decouples into a sum of two terms, and with an additional error due to finite regularization. See Theorem \ref{thm: lossgapprecise} in Appendix \ref{sec: excesslosspf} for a more precise version of the statement which tracks the constant factors.

\begin{rmk}
    A counterintuitive feature of the estimate in Theorem \ref{thm: lossgap} is that the pre-training generalization error increases as the capacity $M$ of the feedforward layer increases. This is because the pre-training generalization error inherits a multiplicative factor that depends on certain covering numbers of the model class, and these covering numbers are increasing functions of $M$. The main technical role of the feedforward ReLU network is to construct an approximation to the one-dimensional square root function, and this approximation error is controlled by the capacity $M$. If one instead defines the parameter space $\Theta$ using feedforward networks with $\textrm{sqrt}$ activation, we can obtain an improved bound on the loss gap of $\tilde{O} \left((1+\lambda)N^{-1/2} + n^{-1/2} + \lambda n^{-1} + \lambda^{-1} \right),$ which is independent of $M$. We focus on the ReLU activation for our theory due to its practical prevalence.
\end{rmk}

Next, we bound the transport map generalization error. To this end, we need an additional assumption on the task distribution $\mu$.

\begin{assum}\label{assum: continuityofeigenprojections}
    There exist constants $\epsilon_{\mu},$ $K_{\mu}$ such that the following holds. For every $\Sigma \in \textrm{supp}(\mu)$, every $\epsilon \leq \epsilon_{\mu}$ and every symmetric matrix $A$ with $\|A-\Sigma\|_F \leq \epsilon$, there exist orthogonal matrices $P$ and $Q$ (depending on $A$) such that $P^T A P$ and $Q^T \Sigma Q$ are diagonal, and 
    $$ \|P-Q\|_F \leq K_{\mu} \epsilon.
    $$
    We also assume that the projection matrices $P$ and $Q$ induce the same ordering of the eigenvalues of $A$ and $\Sigma.$
\end{assum} 
Assumption \ref{assum: continuityofeigenprojections} states that for each $\Sigma \in \textrm{supp}(\mu)$, the projections onto the eigenspace of $\Sigma$ are continuous with respect to perturbations of $\Sigma$, and that the modulus of continuity is uniform in $\Sigma$. For our purposes, Assumption \ref{assum: continuityofeigenprojections} ensures that the landscape of the $\mathcal{R}$ has favorable properties, which in turn allow us to bound the transport map generalization error.

\begin{rmk}
    Assumption \ref{assum: continuityofeigenprojections} is standard perturbation theory. For instance, \citep{chen2003non} proves that any symmetric matrix $A$ satisfies the assumption with constants $O(\delta^{-2})$, where $\delta$ is the minimum gap between consecutive eigenvalues of $A$. Thus, a sufficient condition for Assumption \ref{assum: continuityofeigenprojections} is for matrices in $\textrm{supp}(\mu)$ to have uniformly separated eigenvalues. However, this is condition is not strictly necessary; for instance, if $\textrm{supp}(\mu)$ consists of multiples of the identity matrix, then Assumption \ref{assum: continuityofeigenprojections} is satisfied for any choice of $\epsilon_{\mu}, K_{\mu}$. In general, determining the optimal constants in the assumption is a delicate task, and since this is orthogonal to the goals of this paper, we do not pursue it further.
\end{rmk}

\begin{thm}[Transport map generalization error]\label{thm: TPGE}
    Adopt Assumption \ref{assum: continuityofeigenprojections}, and suppose that the sample sizes $N,n$ are sufficiently large, the regularization $\lambda$ is sufficiently large, and the feedforward capacity bound $M$ is sufficiently large. Then, it holds with probability $1 - \frac{1}{\textrm{poly}(N)}$ that
    \begin{equation}
    \begin{aligned}
        \E \left[ \left\|\widehat{A}_n-\Sigma^{1/2} \right\|_F^2 \right] &= \tilde{O} \Big( \frac{1}{1+2\lambda \sigma_{\min}^2} \left(\frac{(1+\lambda)M^8}{\sqrt{N}} + \sqrt{\frac{1}{n}} + \frac{\lambda}{n} + \frac{1}{M^2} + \frac{\lambda}{M^4} \right)\\
        & \qquad + \frac{1}{\lambda} + \textrm{poly}(M) e^{-\left(\Omega \left(n M^{-2} \right) \right)^{1/4}} \Big), %\exp \left(-\frac{1}{2} \left(\Omega \left(\frac{n}{\epsilon} \right) \right)^{1/4} \right) \right). 
     \end{aligned}
     \end{equation}
    where the expectation is taken over $\Sigma \sim \mu$ and $y_1, \dots, y_n \sim N(0,\Sigma)$. In particular, if $N = \Omega(n^{-6})$, $\lambda = n$, and $M = \textrm{polylog}(n) \sqrt{n}$, we have, up to log factors,
    \begin{align}\label{eq: transportmaperror}
         \E \left[ \left\|\widehat{A}_n-\Sigma^{1/2} \right\|_F^2 \right] &= \tilde{O} \left(\frac{1}{n} \right)
    \end{align}
    with high probability. The same bound also holds in expectation over the training set $\Lambda$.
\end{thm}
For brevity, we do not include the constants factors depending on $d$, $\sigma_{\max}$, and $\sigma_{\min}$ in the estimates of Theorem \ref{thm: TPGE}, though all such factors are clearly expressed in the proofs in Appendix \ref{sec: excesslosspf}. Theorem \ref{thm: TPGE} demonstrates that the pre-trained transformer recovers the true optimal transport map associated to a new prompt with a new Gaussian target and provides quantitative estimates. As in Theorem \ref{thm: lossgap}, the full estimate in Theorem \ref{thm: TPGE} can be represented as a sum of pre-training generalization error, in-context generalization error, neural network approximation error, and regularization error. The final term, of order $\textrm{poly}(M) e^{-\left(\Omega \left(n M^{-2} \right) \right)^{1/4}}$ is a failure probability bound arising from  certain concentration inequalities for the prompt vectors $y_1, \dots, y_n$; while this term may be an artifact of our proof strategy, it becomes a higher-order term as long as $M$ is larger than $n^{1/2}$ by a polylog factor. When setting $\lambda$, $M$, and $N$ as appropriate functions of the prompt length, the final bound recovers the parametric rate $O(n^{-1})$. This is owed to the strong convexity of the loss about its global minimizer. To the best of our knowledge, this is the first result which provides explicit finite-sample estimates for in-context learning of optimal transport maps, albeit in the idealized case where the target measure is a Gaussian.

\section{Proof for non-parametric problem}\label{sec: nonparametricproofs}

\paragraph{Proof for Theorem \ref{thm:general_icl}:} 
We start from an error decomposition:
\begin{align}\label{ineq:decomp}
    \mathcal{R}(\gG_{\Lambda}^*) - \mathcal{R}(\gG^*) 
    &= \mathcal{R}(\gG_{\Lambda}^* ) - \mathcal{R}_{\Lambda}(\gG_{\Lambda}^*) + \mathcal{R}_{\Lambda}(\gG_{\Lambda}^*) -\mathcal{R}_{\Lambda}(\gG^*) +
    \mathcal{R}_{\Lambda}(\gG^*) 
    - \mathcal{R}(\gG^*)\nonumber\\
    &\leq \mathcal{R}(\gG_{\Lambda}^* ) - \mathcal{R}_{\Lambda}(\gG_{\Lambda}^*) + \mathcal{R}_{\Lambda}(\gG^*) 
    - \mathcal{R}(\gG^*)\nonumber\\
    &\leq 2\sup_{\G\in\mathcal{F}}\big|\mathcal{R}(\G) - \gR_\Lambda(\G)\big|.
\end{align}
The term $\mathcal{R}_{\Lambda}(\gG_{\Lambda}^*) -\mathcal{R}_{\Lambda}(\gG^*)\leq 0$ as $\G^*_\Lambda$ is the minimizer of $\gR_\Lambda(\cdot)$ in $\gF$.  
We write the population risk as the sum of OT loss and the $\mathrm{MMD}$ terminal loss, i.e.  $\mathcal{R}(\G) = \mathcal{R}^1(\G) + \mathcal{R}^2(\G)$
where 
\begin{align*}
    \mathcal{R}^1(\G)&:=\mathbb{E}_{(\rho_0,\rho_1)\sim\mu}\mathbb{E}_{x\sim\rho_0}\big\|\gG(\hat{\rho}_0,\hat{\rho}_1;x)-x\big\|^2_2,\\
    \mathcal{R}^2(\G)&:= \lambda\mathbb{E}_{(\rho_0,\rho_1)\sim\mu}\mathbb{E}_{x\sim\rho_0}\mathrm{MMD}^2(\gG(\hat{\rho}_0,\hat{\rho}_1)_{\#}\rho_0,\rho_1).
\end{align*}
Similarly, we write the empirical counterparts as
$\mathcal{R}_{\Lambda}(\G) =\mathcal{R}^1_\Lambda(\G)+\mathcal{R}^2_\Lambda(\G)%,~~~~\text{with}~~ \Lambda = \{\rho^{k_1}_0,\rho^{k_1}_1,\hat\rho^{k_1}_0,\hat\rho^{k_1}_1,\tilde{\rho}^{k_1}_0,\tilde{\rho}^{k_1}_1 \}_{k_1=1}^N. 
$ for each training set $\Lambda$.
Taking expectation over $\Lambda$, \eqref{ineq:decomp} becomes
\begin{align}\label{ineq:1+2}
    &\mathbb{E}_\Lambda|\mathcal{R}(\gG_{\Lambda}^*) - \mathcal{R}(\gG^*)|\nonumber\\
    \leq &2\mathbb{E}_\Lambda\sup_{\G\in\mathcal{F}} |\mathcal{R}(\gG) - \mathcal{R}_{\Lambda}(\gG)|\leq 2\mathbb{E}_\Lambda\sup_{\G\in\mathcal{F}}|\mathcal{R}^1(\G) - \mathcal{R}^1_\Lambda(\G)|+ 2\mathbb{E}_\Lambda\sup_{\G\in\mathcal{F}}|\mathcal{R}^2(\G) -\mathcal{R}^2_\Lambda(\G)|.
\end{align}
 We bound each term in \eqref{ineq:1+2} separately in Lemma \ref{lem:transport} and Lemma \ref{lem:terminal}, then Theorem \ref{thm:general_icl} directly follows. 
\begin{lem}\label{lem:transport} 
    Suppose Assumption \ref{ass:mlfd},\ref{ass:bdd&lips},\ref{ass:moment} hold. The transport map generalization error satisfies
    \begin{equation}\label{eq:otloss}
    \begin{aligned}
\mathbb{E}_\Lambda\sup_{\G\in\mathcal{F}}|\mathcal{R}^1(\G) - \mathcal{R}^1_\Lambda(\G)|
        = O\Big(N^{-\frac{1}{d_{\gM}+1}}+ s^{\max\{-\frac{1}{6},-\frac{1}{d}\}}+n^{\max\{-\frac{1}{6},-\frac{1}{d}\}}\Big).
    \end{aligned}
    \end{equation}
    The hidden constants in $O(\cdot)$ depend on $\widetilde{M},M_\f, L_\f, \mathrm{diam}(\gM)$ and $\tau_\gM$.
\end{lem}
 
Under Assumption \ref{ass:bdd&lips}, the generalization error in Lemma \ref{lem:transport} relies on controlling the distribution shift between the training tasks and the unseen test tasks, which is measured by minibatch Wasserstein between empirical measures (see \citep{Sommerfield2019},\citep{fatras20a},\citep{mroueh2023towards}).  With high probability, the uniformly distributed training tasks form an $\epsilon$-covering on $\gM$, given the task number $N(\epsilon)$ sufficiently large. During the inference stage, the new task is within $\epsilon$-neighborhood of some seen tasks which guarantees good generalization.
To handle the unbounded domain, we need the base measure $\rho_0\in\gP(\R^d)$ have finite $p$-th moment, $p\geq 2$. The power $\frac{1}{6}$ in \eqref{eq:otloss} is due to Assumption \ref{ass:moment} where we choose $p=3$. See Appendix \ref{appendix:nonparametricproofs} for proof details.
\begin{lem}\label{lem:terminal}
    Suppose Assumption \ref{ass:mlfd},\ref{ass:bdd&lips},\ref{ass:moment} holds. The $\mathrm{MMD}$ terminal generalization error satisfies 
    \begin{equation}\label{eq:mmdloss}
        \mathbb{E}_\Lambda \sup_{\G\in\f}|\mathcal{R}^2(\G) -\mathcal{R}^2_\Lambda(\G)| = O\Big(N^{-\frac{1}{4\max\{d,d_{\gM}\}}}n^{\frac{2}{3}}+N^{-\frac{1}{4}}n^{\frac{d+3}{3}}+N^{\frac{1}{8\max\{d,d_{\gM}\}}}s^{-\frac{1}{2}}n^{\frac{2}{3}} + n^{-\frac{1}{3}}\Big).
    \end{equation}
    for quadratic kernel.
    Further suppose Assumption \ref{ass:kernel} holds, we obtain a tighter bound
    \begin{align}\label{eq:mmdloss_char}
        \mathbb{E}_\Lambda \sup_{\G\in\f}|\mathcal{R}^2(\G) -\mathcal{R}^2_\Lambda(\G)| = O\Big(N^{-\frac{1}{4\max\{d,d_{\gM}\}}}+N^{-\frac{1}{4}}n^{\frac{d+1}{3}}+N^{\frac{1}{8\max\{d,d_{\gM}\}}}s^{-\frac{1}{2}} + n^{-1}\Big).
    \end{align}
    The hidden constants in $O(\cdot)$ depend on $\lambda,d_\gM, d,\widetilde{M},M_\f$ and $L_\f$.
\end{lem}

To control model complexity, we derive covering number of the function class $\f$ to control model complexity, which introduces a term that grows with a positive power of the number of tasks $N$. Therefore, a larger number of prompts $s$ is required to balance this term. With the choice of quadratic kernel on $\R^d\times\R^d$, the proof of Lemma \ref{lem:terminal} requires truncating the domain of target measure $\rho_1$ to ensure boundedness, which introduces the positive power of $n$ term in \eqref{eq:mmdloss}.

\section{Proof for parametric problem}\label{sec: gaussiantogaussian}

\subsection{Proof sketches}\label{sec: pfsketch}
\paragraph{Proof sketch for Theorem \ref{thm: lossgap}:}
    The key to our proof of Theorem \ref{thm: lossgap} is a careful decomposition of the excess loss. To this end, we need to introduce some notation. Define the individual loss function $\ell$ by
\begin{align*}
    \ell(\theta, y_1, \dots, y_{2n}) = \E_{x \sim \mathcal{N}(0,I)}[\|A_{n,\theta}x-x\|^2] + \lambda \|A_{n,\theta}^2 - \Sigma_n\|_F^2,
\end{align*}
so that the population risk can be expressed as
\begin{align*}
    \mathcal{R}(\theta) = \E_{\Sigma \sim \mu, y_1, \dots, y_{2n} \sim \mathcal{N}(0,\Sigma)} [\ell(\theta, y_1, \dots, y_{2n})].
\end{align*}
For $t > 0$ and random variables $y_1, \dots, y_{2n} \in \R^d$, consider the event
\begin{align*}
    \mathcal{A}_t &= \bigcap_{i=1}^{n} \left \{\|y_i\| \leq \sqrt{d} \sigma_{\max} + t, \; \|\Sigma_n\|_{\op} \leq \sigma_{\max}^2 \left(1 + t + \sqrt{\frac{d}{n}} \right) \right\},
\end{align*}
where we recall $\Sigma_n = \frac{1}{n} \sum_{k=n+1}^{2n} y_k y_k^T.$ Define the $t$-truncated individual loss function 
\begin{align*}
    \ell_t(\theta, y_1, \dots, y_{2n}) = \ell(\theta, y_1, \dots, y_{2n}) \cdot \mathbf{1}\left(y_1, \dots, y_{2n} \in \mathcal{A}_t \right).
\end{align*}
Let $\mathcal{R}_{t}$ and $\mathcal{R}_{t,\Lambda}$ denote the population and empirical risk functionals with $\ell$ replaced by $\ell_t$. Then we can upper bound the excess loss in the following way.

\begin{lem}\label{lem: errordecomp}[Decomposition of Excess loss]
    Let $\widehat{\theta} \in \textrm{arg} \min_{\theta \in \Theta} \mathcal{R}_{\Lambda}(\theta).$ Then, for any $\theta^{\ast} \in \Theta$, we have
    \begin{align*}
        \mathcal{R}(\widehat{\theta}) - \mathcal{R}(\mathcal{G}^{\dagger}) &\leq \underbrace{\sup_{\theta \in \Theta} \left(\mathcal{R}(\theta) - \mathcal{R}_{t}(\theta) \right)}_{\textrm{Truncation error}} + \underbrace{2 \sup_{\theta \in \Theta} \left|\mathcal{R}_{t}(\theta) - \mathcal{R}_{t,\Lambda}(\theta) \right|}_{\textrm{Pre-training generalization error}} + \underbrace{\left(\mathcal{R}(\theta^{\ast}) - \mathcal{R}(\mathcal{G}^{\dagger}) \right)}_{\textrm{Approximation error}}
    \end{align*}
    with probability $\geq 1 - \delta_{N,n}$, where $$\delta_{N,n} = 2N\left(n\exp \left(-\frac{t^2}{C \sigma_{\max}^2} \right) + \exp \left(- \frac{nt^2}{2} \right) \right),$$ and $C > 0$ is a universal, dimension-independent constant.
\end{lem}
\begin{proof}
    For any $t > 0$ and $\theta^{\ast} \in \Theta$, we decompose the excess loss as
    \begin{align*}
        \mathcal{R}(\widehat{\theta}) &- \mathcal{R}(\mathcal{G}^{\dagger}) = \left(\mathcal{R}(\widehat{\theta}) - \mathcal{R}_{t}(\widehat{\theta}) \right) + \left( \mathcal{R}_{t}(\widehat{\theta}) - \mathcal{R}_{t,\Lambda}(\widehat{\theta}) \right) + \left( \mathcal{R}_{t,\Lambda}(\widehat{\theta}) - \mathcal{R}_{t,\Lambda}(\theta^{\ast})\right)  \\
        &+ \left(\mathcal{R}_{t,\Lambda}(\theta^{\ast}) - \mathcal{R}_{t}(\theta^{\ast}) \right) + \left(\mathcal{R}_{t}(\theta^{\ast}) - \mathcal{R}(\theta^{\ast}) \right) + \left(\mathcal{R}(\theta^{\ast}) -\mathcal{R}(\mathcal{G}^{\dagger}) \right).
    \end{align*}
    The first term is bounded above by $\sup_{\theta \in \Theta} \left(\mathcal{R}(\theta) - \mathcal{R}_{t}(\theta) \right),$ while the second and fourth terms are bounded by $\sup_{\theta \in \Theta} \left|\mathcal{R}_{t}(\theta) - \mathcal{R}_{t,\Lambda}(\theta) \right|.$ The fifth term is non-positive because $\mathcal{R}{\mu,t}(\cdot) \leq \mathcal{R}(\cdot).$ To control the second term, let us assume that the event $\mathcal{A}_t$ holds for each training prompt $\{y_1^{(i)}, \dots, y_{2n}^{(i)}\},$ we have $\mathcal{R}_{\Lambda}(\cdot) = \mathcal{R}_{t,\Lambda}(\cdot)$. This implies that $\widehat{\theta} \in \textrm{arg} \min_{\theta \in \Theta} \mathcal{R}_{t,\Lambda}(\theta)$, and therefore that 
    $$  \left( \mathcal{R}_{t,\Lambda}(\widehat{\theta}) - \mathcal{R}_{t,\Lambda}(\theta^{\ast})\right) \leq 0
    $$
    for any $\theta^{\ast} \in \Theta$. By standard concentration inequalities for Gaussian random vectors and the covariance matrices (e.g., \ref{lem: gaussianconc} and Example 6.3 in \citep{wainwright2019high}), the event $\mathcal{A}_t$ holds for a single prompt $\{y_1, \dots, y_{2n}\}$ with probability at least $1 - \delta_{N,n}$.
\end{proof}
    Lemma \ref{lem: errordecomp} upper bounds the excess loss by a sum of truncation error, pre-training generalization error, and approximation error. We proceed by bounding each term individually.

    \paragraph{I: The truncation error}
    The truncation error arises from replacing the data distributions with their truncations to a compact set, and thus the error incurred at this step depends on the concentration of the data distributions. Using standard concentration properties of the Gaussian, we derive a bound on the truncation error.

    \begin{lem}[Truncation error bound]
        The truncation error is bounded by
        \begin{align*}
            \sup_{\theta \in \Theta} \left(\mathcal{R}(\theta) - \mathcal{R}_{t}(\theta) \right) = O_{\textrm{poly}(\epsilon^{-1})} \left( n^{1/2} \exp \left(-\frac{t^2}{2C \sigma_{\max}^2} \right) + \exp \left( -\frac{nt^2}{4} \right) \right).
        \end{align*}
    \end{lem}
    See Appendix \ref{sec: excesslosspf} for proof details. This shows that the truncation error is quite mild, since it decays exponentially in the truncation parameter $t$. 

    \paragraph{II: The statistical error}
    The statistical error is bounded using techniques from empirical process theory. Specifically, we use Dudley's theorem to reduce the statistical error to an integral of the metric entropy of a suitable function class, which we then bound by deriving covering number estimates for the transformer architecture. Importantly, these techniques rely on the compactness of the data distribution, which is the reason for the truncation step. The bound is summarized below.
    \begin{lem}[Statistical error bound]
        The statistical error satisfies the bound
        \begin{align*}
            \sup_{\theta \in \Theta} \left|\mathcal{R}_{t}(\theta) - \mathcal{R}_{t,\Lambda}(\theta) \right| = O_{\textrm{poly}(t), \textrm{poly}(\epsilon^{-1}), \log(N)} \left(\frac{1}{\sqrt{N}} \right).
        \end{align*}
    \end{lem}
    The dependence on the approximation error parameter $\epsilon$ arises from the fact that the covering number bound depends on the complexity of the network class, which in turn depends on $\epsilon.$ We refer to Appendix \ref{sec: excesslosspf} for the proof details.

    \paragraph{III: The approximation error}
    The approximation error $\mathcal{R}(\theta^{\ast}) - \min\mathcal{R}$ is bounded by explicitly constructing a transformer to approximate the Gaussian-to-Gaussian optimal transport map in-context. Since this result may be of independent interest, we state it below.
    \begin{prop}[Transformers implement Gaussian optimal transport in-context]\label{prop: OTtransformer}
        There exists a transformer parameterization which approximates the Gaussian-to-Gaussian optimal transport mapping in-context. Specifically, for any $\epsilon > 0$ and any task distribution $\mu$ satisfying Assumption \ref{assum: taskdistr}, there exists a transformer with parameters $\theta = (Q, \phi) \in \Theta$ which takes a sequence $(y_1, \dots, y_n, x)$, with $y_i \sim \mathcal{N}(0,\Sigma)$ as an input, and outputs an estimate $\widehat{y}_{n,\theta}$ satisfying
        \begin{align*}
            \left\|\widehat{y}_{n,\theta} - \Sigma^{1/2}x \right\| \lesssim \epsilon + \sqrt{\frac{t}{n}}, \; \; \textrm{w.p. $\geq 1-e^{-t}$},
        \end{align*}
        where the probability is taken over $\Sigma \sim \mu$, $y_1, \dots, y_n \sim \mathcal{N}(0,\Sigma)$.
    \end{prop}
    Before bounding the approximation error, we describe the construction used to prove Proposition \ref{prop: OTtransformer} (the full proof is deferred to Appendix \ref{sec: excesslosspf}). Recall that the in-context mapping determined by the transformer hypothesis class is given by
    \begin{align*}
        (y_1, \dots, y_n, x) \mapsto \left( Q \cdot \frac{1}{n} \sum_{i=1}^{n} \phi(y_i) \phi(y_i)^T \cdot Q^T \right)x,
    \end{align*}
    where $y_1, \dots, y_n \sim \mathcal{N}(0,\Sigma),$ $x \sim \mathcal{N}(0,I)$, and $(Q,\phi)$ are the learnable parameters of the architecture. The aim is to choose the attention matrix $Q = Q_{\ast}$ and feedforward network $\phi = \phi_{\ast}$ such that
    $$ \left( Q_{\ast} \cdot \frac{1}{n} \sum_{i=1}^{n} \phi_{\ast}(y_i) \phi_{\ast}(y_i)^T \cdot Q_{\ast}^T \right)x \approx \Sigma^{1/2}x.
    $$
    Assume the covariance can be diagonalized according to $\Sigma = U \Lambda U^T.$ In Appendix \ref{sec: auxlemmas}, we show how to choose a general function $g$ in such a way that the random vector $g(y_i)$ has mean zero and covariance $\Lambda^{1/2}$. It follows that $\frac{1}{n} \sum_{i=1}^{n} g(y_i) g(y_i)^T \rightarrow \Lambda^{1/2}$ as $n \rightarrow \infty.$ We then show that such a choice of $g$ can be approximated up to an additive error $\epsilon$ by a feedforward network $\phi_{\ast}.$ Then, choosing the attention matrix $Q_{\ast}$ to equal the diagonalizing matrix $U$, our construction satisfies 
    $$ \left( Q_{\ast} \cdot \frac{1}{n} \sum_{i=1}^{n} \phi_{\ast}(y_i) \phi_{\ast}(y_i)^T \cdot Q_{\ast}^T \right)x \rightarrow \Sigma^{1/2}x
    $$
    as the sample size $n \rightarrow \infty$ and the approximation tolerance $\epsilon \rightarrow 0.$ To the best of our knowledge, our construction is the first mathematical demonstration that transformers can represent the solution to the optimal transport between continuous measures (the paper \citep{daneshmand2024provable} studies the ability of transformers to solve discrete OT problems).
    
   We now leverage the construction of Proposition \ref{prop: OTtransformer} to prove a bound on the approximation error term, proven rigorously in Appendix \ref{sec: excesslosspf}.
    \begin{lem}[Approximation error bound]
        The approximation error satisfies
        \begin{align*}
            \mathcal{R}(\theta^{\ast}) - \mathcal{R}(\mathcal{G}^{\dagger}) \lesssim (1+t) \left( \underbrace{\left(\sqrt{\frac{1}{n}} + \frac{\lambda}{n}\right)}_{\textrm{in-context gen. error}} + \underbrace{\left(\epsilon + \lambda \epsilon^2\right)}_{\textrm{neural network approx. error}} + \underbrace{\frac{1}{\lambda}}_{\textrm{regularization error}} \right).
        \end{align*}
    \end{lem}
    The approximation error decomposes into a sum of three terms: the first term accounts for the fact that we have only a finite number $n$ of samples for the downstream task; the second term accounts for the approximation error of the neural network class; the third term accounts for the fact that we need $\lambda$ to be large in order for the risk $\mathcal{R}$ to be minimized by the optimal transport map. 

    Theorem \ref{thm: lossgap} is then proven by combining the bounds for the truncation, statistical, and approximation errors, and choosing the truncation parameter $t$ as a suitable function of $n$, $N$, and $\epsilon$. Once again, refer to Appendix \ref{sec: excesslosspf} for the proof details.

    \paragraph{Proof sketch for Theorem \ref{thm: TPGE}} Our proof of Theorem \ref{thm: TPGE} leverages the strong convexity exhibited by a variant of the risk $\mathcal{R}$ about its minimizer. Specifically, for a fixed covariance $\Sigma \in \textrm{supp}(\mu)$, and $\lambda > 0$, define $f_{\Sigma,\lambda}: \R^{d \times d}_{sym} \rightarrow \R_+$ by
    $$ f_{\Sigma,\lambda}(A) = \|A-I\|_F^2 + \lambda \|A^2 - \Sigma^2\|_F^2.
    $$
    In Proposition \ref{prop: largelambdalimit} of Appendix \ref{sec: transportmapfs}, we show that as $\lambda \rightarrow \infty,$
    \begin{enumerate}
        \item The minimizers of $f_{\Sigma,\lambda}$ converge to $\Sigma^{1/2}$ with rate $O \left(\frac{1}{\lambda^2} \right)$;
        \item The minimum value of $f_{\Sigma,\lambda}$ converges to $W_2^2 \left(\mathcal{N}(0,I),\mathcal{N}(0,\Sigma) \right)$ with rate $O \left( \frac{1}{\lambda} \right)$;
        \item For all $\lambda$ sufficiently large, and all $A$ such that $f_{\Sigma,\lambda}(A) - \min f_{\Sigma,\lambda} \leq C$ for some absolute constant $C > 0$, we have the stability estimate
        \begin{align}\label{eqn: stabilitysketch1}
            \|A-\Sigma^{1/2}\|_F^2 \lesssim \frac{1}{1+2\lambda \sigma_{\min}^2} \left( f_{\Sigma,\lambda}(A) - \min f_{\Sigma,\lambda} \right) + O \left( \frac{1}{\lambda} \right).
        \end{align}
    \end{enumerate}
    See Proposition \ref{prop: largelambdalimit} for the precise statement. In particular, the estimates above can be used to show that $f_{\Sigma,\lambda}$ is related to the excess loss according to the inequality
    \begin{align}\label{eqn: stabilitysketch2}
        \E_{\Sigma \sim \mu, (y_1, \dots, y_n) \sim N(0,\Sigma)} \left[f_{\Sigma,\lambda}(A_{n,\theta}) - \min f_{\Sigma,\lambda} \right] \leq \left|\mathcal{R}(\theta) - \mathcal{R}(\mathcal{G}^{\dagger}) \right| + O \left(\frac{\lambda}{n} + \frac{1}{\lambda} \right),
    \end{align}
    valid for any $\theta \in \Theta$. Equations \eqref{eqn: stabilitysketch1} and \eqref{eqn: stabilitysketch2} can be combined to bound the transport map generalization error by the excess loss, up to $O \left(\frac{\lambda}{n} + \frac{1}{\lambda} \right),$ and then Theorem \ref{thm: lossgap} can be used to control the excess loss when the parameters are selected by empirical risk minimization. It is worth noting that Equation \eqref{eqn: stabilitysketch1} holds only when $f_{\Sigma,\lambda}(A) - \min f_{\Sigma,\lambda} \leq C$, and thus applying this bound to the \textit{random} matrix $A_{n,\theta}$ requires additional care. Specifically, we use a concentration argument to reduce to the event where this inequality holds; these details are deferred to Appendix \ref{sec: transportmapfs}.
    
\section{Numerical experiments}\label{sec: numerics}

We evaluate the proposed solution-operator framework for optimal transport on both synthetic and real-world datasets. All computations are performed on NVIDIA A100 GPUs.  
The solution operator is trained to map a fixed source distribution
$
\rho_0 = \mathcal N(0,I)
$
in a chosen dimension to a collection of target distributions $\{\rho_1^k\}_{k=1}^N$, each corresponding to a separate task.
For task $k$, we draw prompts $\{(x_i, y_i^k)\}_{i=1}^s$ independently from the product distribution $\rho_0 \times \rho_1^k$. 

To assess the quality of the inferred transport map produced by the learned solution operator, we compute the unbiased maximum mean discrepancy $\mathrm{MMD}_u$ between samples generated by the learned map and ground-truth samples drawn from $\rho_1^k$.

We adopt a multi-scale RBF kernel for $\mathrm{MMD}_u$, defined as
\[
k(u,v)
= \sum_{l=1}^L w_l \exp\!\left(-\frac{\|u-v\|_2^2}{\sigma_l}\right),
\qquad 
w_l \ge 0,\quad \sum_{l} w_l = 1.
\]
The bandwidths $\{\sigma_l\}_{l=1}^L$ follow a geometric progression,
\[
\sigma_l = \sigma_0\, 2^{\,l - \lceil L/2 \rceil},
\qquad l = 1,\dots,L.
\]
In all experiments, we fix $L = 5$. The base bandwidth $\sigma_0$ is chosen adaptively for each batch as the mean pairwise squared distance across all pooled samples. Given predicted samples $\hat{Y}=\{\hat{y}_j\}_{j=1}^{n}$ and ground-truth samples $Y=\{y_j\}_{j=1}^{n}$ from $\rho_1^k$, we set
\[
\sigma_0
= \frac{1}{2n(2n-1)}
\sum_{i \neq j} \|z_i - z_j\|_2^2,
\qquad
z_i \in \hat{Y} \cup Y.
\]

\paragraph{ICL Model Architecture}
Across all experiments, we employ the  cross-attention MLP architecture. Let $d$ be the latent dimension and $h$ the hidden width. The source samples and the target samples in each prompt are first embedded into $h$ channels using separate point-wise MLPs (implemented as 1D convolutions with kernel size~1). A self-attention layer is then applied over the concatenated prompt tokens to form a prompt context. A single multi-head cross-attention layer allows each query token to attend to this context. Finally, a point-wise MLP maps the resulting representations back to dimension~$d$, producing the prediction output $\hat{y}$. All numerical experiments use 4 attention heads and a hidden width of $h = 2048$. The architecture is illustrated in Figure~\ref{fig:arch}.

\begin{figure}[ht] % h = here, t = top, b = bottom, p = page of floats
\centering
\includegraphics[width=\linewidth]{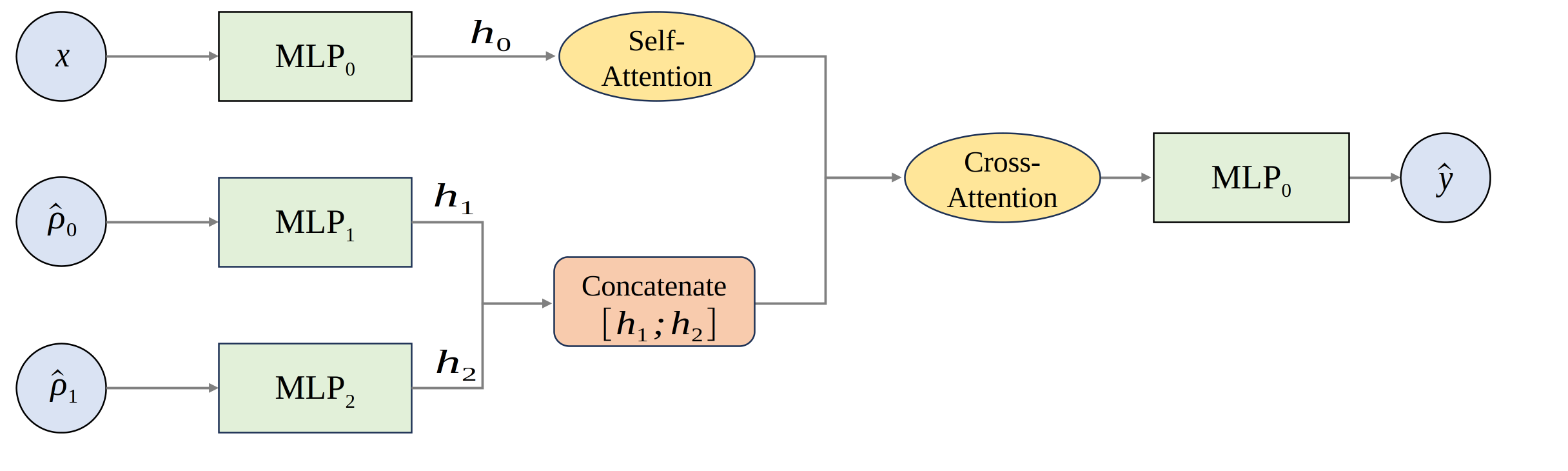} % image file name
\centering \caption{Illustration of the ICL model architecture.}
    \label{fig:arch}
\end{figure}

Additional experimental details are provided in Appendix~\ref{Appendix_num}.

\subsection{Gaussian to Gaussian}\label{sec:G-to-G}

We first consider a synthetic setting in \(\mathbb{R}^2\) where the source $\rho_0$ is the standard normal \(\mathcal{N}(0,I_2)\) and target $\rho_1$ is another Gaussian \(\mathcal{N}(\mu,\Sigma)\). Because the OT map between two Gaussians is an affine map and admits a closed-form expression, 
the Gaussian-Gaussian case provides a natural validation of our ICL framework. Qualitatively, we visualize the learned distributions and mapping trajectories at selected landmarks. Quantitatively, 
we verify Theorem~\ref{thm: lossgap} between Gaussians with different covariances in Section \ref{sec:ScalingLaw}. 
In this scenario, a task is fully parametrized by the pair $(\mu,\Sigma)$, and we consider two types of low-dimensional task manifolds:
\begin{itemize}
     \item Each task is a pair \((\mathcal{N}(0,I_2),\mathcal{N}(\mu,I_2))\), where the target Gaussian $\mathcal{N}(\mu,\Sigma)$ satisfies
\begin{equation}\label{mean_shift}
    \Sigma = I_2, \quad\text{and}~~\mu \sim \mathrm{Unif}([4,6]^2).
\end{equation}

\item Each task is a pair \((\mathcal{N}(0,I_2),\mathcal{N}(0,\Sigma))\), where the target Gaussian $\mathcal{N}(\mu,\Sigma)$ satisfies
\begin{equation}\label{cov_shift}
    \mu=0,\quad \text{and}~~\Sigma = \mathrm{diag}(\sigma_1,\sigma_2) \quad\text{with}~~ \sigma_1,\sigma_2 \sim \mathrm{Unif}[1,3].
\end{equation}

\end{itemize}
For each case above, 5000 training tasks and 2000 testing tasks are sampled from independent pools, ensuring no leakage while maintaining the same distributional support. In Fig.~\ref{fig:triplegaussian} and Fig.~\ref{fig:triplegaussian2}, both the learned and target distributions are visualized by using superimposed kernel density contour lines, computed via a standard 2D Gaussian KDE with a fixed bandwidth schedule. Additionally, we plot the inferred solution trajectories on random selected samples (landmarks) which are straight and non-intersecting. These observations suggest that the learned solution operator is capable of recovering the OT map between two Gaussian distributions provided conditional samplings from varies target Gaussians as prompts. 

%%%%%%%%%%%%%%%%%%%%%%%%%%%%%%%%%%%%%%%%%%%%%%%%%%%%%%%%%%%%%%%%%%%%%%

\begin{figure}[!ht]
  \centering
%   %--- subfigure 1 ---
  \begin{subfigure}[t]{0.32\textwidth}
    \centering
    \includegraphics[width=\linewidth]{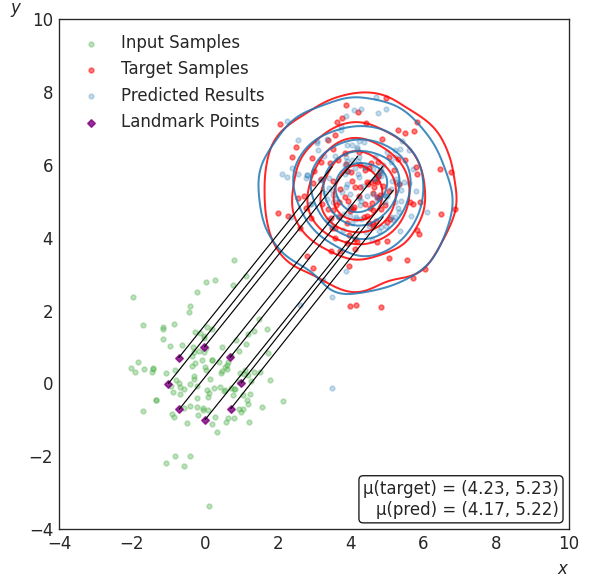}
    %\caption{}
    \label{fig:gauss1mean}
  \end{subfigure}\hfill
%   %--- subfigure 1 ---
  \begin{subfigure}[t]{0.32\textwidth}
    \centering
    \includegraphics[width=\linewidth]{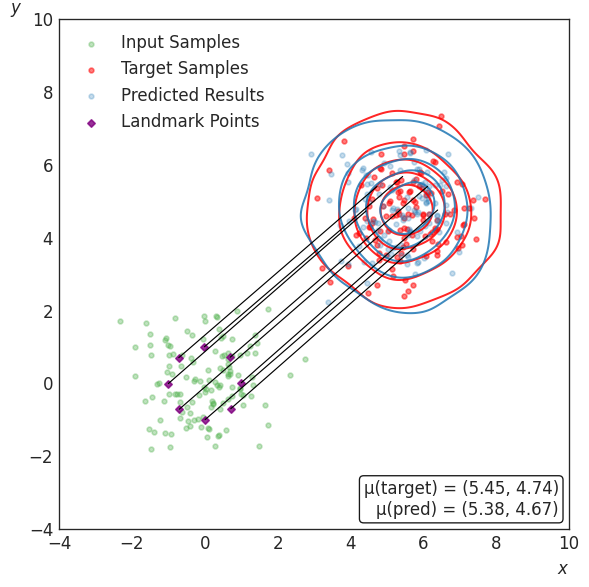}
    %\caption{}
    \label{fig:gauss2mean}
  \end{subfigure}\hfill
%   %--- subfigure 1 ---
  \begin{subfigure}[t]{0.32\textwidth}
    \centering
    \includegraphics[width=\linewidth]{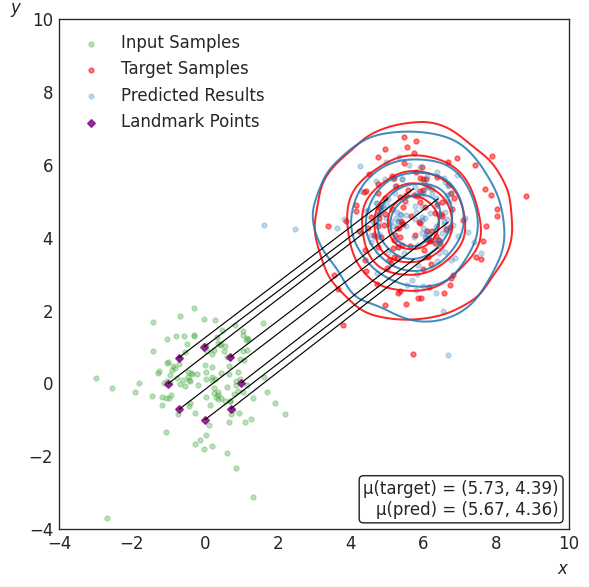}
    %\caption{}
    \label{fig:gauss3mean}
  \end{subfigure}

  \caption{Model prediction with prompts sampled from Gaussian with varying target means and fixed covariance as stated in \eqref{mean_shift}. Blue and red contours show the predicted and target distribution respectively. The segments connecting landmarks to their mapped predictions are nearly parallel, which is consistent with the optimal transport map \(T(x)=x+\mu\). }
  \label{fig:triplegaussian}
\end{figure}

\begin{figure}[ht]
  \centering
  %--- Subfig 1 ---
  \begin{subfigure}[t]{0.32\textwidth}
    \centering
    \includegraphics[width=\linewidth]{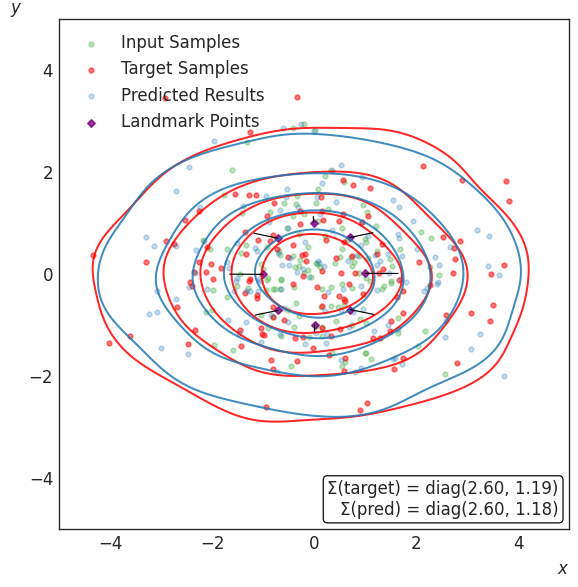}
    %\caption{}
    \label{fig:gauss1var}
  \end{subfigure}\hfill
  %--- Subfig 2 ---
  \begin{subfigure}[t]{0.32\textwidth}
    \centering
    \includegraphics[width=\linewidth]{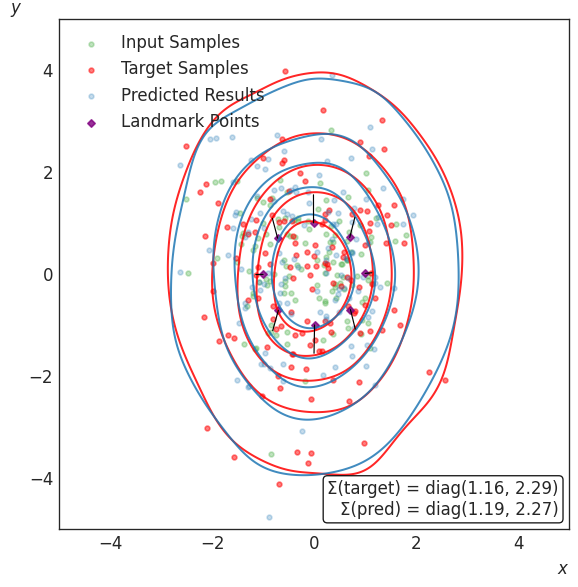}
    %\caption{}
    \label{fig:gauss2var}
  \end{subfigure}\hfill
  %--- Subfig 3 ---
  \begin{subfigure}[t]{0.32\textwidth}
    \centering
    \includegraphics[width=\linewidth]{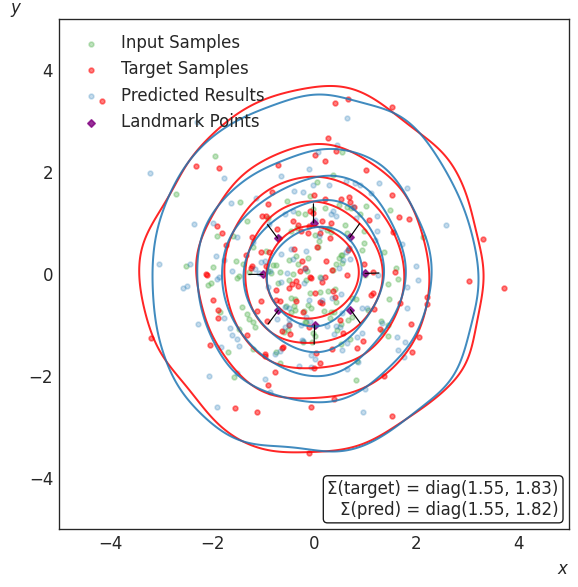}
    %\caption{}
    \label{fig:gauss3var}
  \end{subfigure}

  \caption{Model prediction with prompts sampled from different diagonal covariances and fixed mean as stated in \eqref{cov_shift}. Blue and red contours show the predicted and target distribution respectively. The black segments are learned trajectories for selected landmarks.
  }
  \label{fig:triplegaussian2}
\end{figure}

\subsection{Gaussian to high-dimensional image datasets}\label{sec: gausstohighdim}
To further assess the capability of the proposed ICL solution-operator framework, we conduct numerical experiments in high-dimensional settings. The purpose of these experiments is twofold. First, we aim to evaluate the effectiveness of the learned solution operator when the data dimension becomes large. Second, we demonstrate a very useful feature of our framework: its natural ability to perform prompt-driven conditional sampling, which emerges directly from the ICL formulation of the solution operator learning. 

In conventional generative modeling pipelines, all classes of images are typically treated as samples from a single mixed distribution. During inference, a latent point drawn from a base Gaussian distribution cannot be controlled to generate a sample from a specific class unless one introduces additonal conditioning mechanisms. In contrast, the proposed ICL solution-operator formulation views each class as a separate target distribution $\rho^k$, and the solution operator $G$ learns a family of transport maps parameterized by prompts. More precisely, for prompt sets $(\hat{\rho}_0, \hat{\rho}^k)$ associated with task $k$, the learned operator satisfies
$$
G(\hat{\rho}_0, \hat{\rho}^k, \cdot)_{\#}\rho_0 \approx \rho^k
$$
so that providing prompts sampled from $\rho_1^k$ automatically induces a transport map that generates new samples consistent with class $k$. This mechanism yields a simple but powerful conditional inference procedure: selecting prompts from a target class naturally produces samples from that class, without changing the model or its parameters. For the sake of space, we illustrate only a few representative examples here. Additional results and examples are provided in Appendix~\ref{Appendix_image}.

\paragraph{Gaussian to MNIST}
MNIST consists of 70k handwritten digit images across 10 classes (digits 0--9), which we treat as samples from the target task distributions $\rho^k$. Each image has resolution $28\times 28$ pixels. We use the standard 60k/10k train/test split. A pretrained encoder maps each image in $\mathbb{R}^{28\times 28}$ to a latent vector in $\mathbb{R}^d$ with dimension $d=4$.

In our setup, each task corresponds to transporting the base distribution $\rho_0 \in \mathcal{P}_2(\mathbb{R}^d)$, taken as a standard Gaussian, to one of the digit-class distributions $\rho^k \in \mathcal{P}_2(\mathbb{R}^d)$. During testing, we draw samples from a given digit class to form the prompt and use the learned operator to generate predictions in the latent space. The resulting latent codes are then decoded by the pretrained decoder to produce $28\times 28$ grayscale images.
Figure~\ref{fig:mnist} illustrates representative class-conditioned generations, demonstrating that the proposed framework can successfully generate images from the intended class solely through controlling the prompt with samples from that same class.

\begin{figure}[ht]
  \centering

  \foreach \row in {12,13,15}{
    \foreach \col in {0,...,9}{% 
      \includegraphics[width=0.09\linewidth]{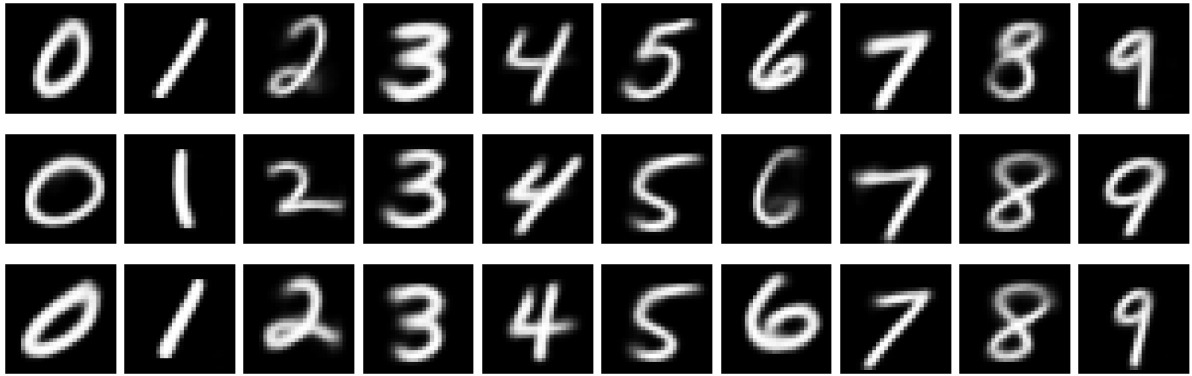}%
      \ifnum\col<9\hspace{\imgsep}\fi
    }%
    \par\medskip
  }%
\caption{Visualization results of the Gaussian to MNIST. Different task corresponds to mapping from Standard Gaussian to one digit class. Conditional sampling images from one digit class as prompts, the column-wise images are the output of the learned in-context model for various Gaussian inputs.}
\label{fig:mnist}
\end{figure}

\paragraph{Gaussian to Fashion-MNIST}
We further evaluate the proposed framework on the Fashion-MNIST dataset. In this experiment, the ICL model learns mappings from the standard Gaussian base distribution $\rho_0 \in \mathcal{P}_2(\mathbb{R}^d)$ to the latent-code distributions of Fashion-MNIST classes $\rho_1 \in \mathcal{P}_2(\mathbb{R}^d)$. Fashion-MNIST consists of 70k grayscale images of size $28 \times 28$ across 10 categories, and we use the standard 60k/10k train/test split with a fixed class ordering. A pretrained encoder maps each image in $\mathbb{R}^{28 \times 28}$ to a latent vector in $\mathbb{R}^d$ with dimension $d=15$.

Similar as in the Gaussian--MNIST setting, each task corresponds to transporting the Gaussian base distribution to one specific fashion-class distribution in $\mathcal{P}_2(\mathbb{R}^d)$. The ten categories, including T-shirt/top, trousers, pullover, dress, coat, sandal, shirt, sneaker, bag, and ankle boot, are treated as separate target distributions, each containing only images from that class. During testing, predicted latent codes are decoded using the pretrained Fashion-MNIST decoder to recover $28 \times 28$ grayscale images.
Figure~\ref{fashionvis} shows class-conditioned generations. Once again, the model successfully produces images from the intended class solely by controlling the prompt with samples drawn from that same class, confirming the framework's natural ability for prompt-driven conditional generation.

\begin{figure}[ht]
  \centering
  % 
  % 

  %---  ---
  \foreach \row in {12,...,14}{% 
    \foreach \col in {0,...,9}{% 
      \includegraphics[width=0.09\linewidth]{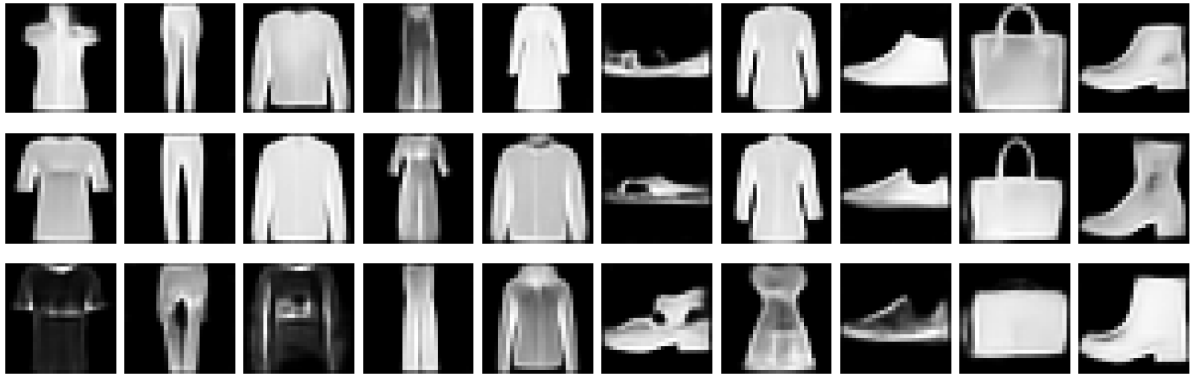}%
      \ifnum\col<9\hspace{\imgsep}\fi
    }%
    \par\medskip % 
  }%
  
\caption{Visualization results of the Gaussian to Fashion-MNIST. 
Each column corresponds to the inference results of each task. Given the prompts (images) sampled from one fashion class as conditioning, the learned model generates images from the same fashion class with consistency.}
\label{fashionvis}
\end{figure}

\paragraph{Gaussian to ModelNet10}
We further extend our evaluation to 3D point clouds to assess the robustness and scalability of the proposed ICL framework. In this experiment, we use the ModelNet10 dataset, which contains 4,899 pre-aligned 3D CAD models spanning 10 object categories (bathtub, bed, chair, desk, dresser, monitor, nightstand, sofa, table, and toilet). We adopt the standard per-class train/test splits and maintain a fixed class ordering across all experiments.

A pretrained autoencoder maps each point cloud in $\mathbb{R}^{512}$ to a latent vector in $\mathbb{R}^d$ with dimension $d = 10$. Each task is defined as transporting the Gaussian base distribution $\rho_0 \in \mathcal{P}_2(\mathbb{R}^d)$ to the latent-code distribution associated with a single object class in $\mathcal{P}_2(\mathbb{R}^d)$. During inference, the predicted latent codes are passed through the fixed decoder to obtain $(3n_{\text{pts}})$-dimensional point clouds, which we visualize as scatter plots using a perceptually uniform colormap along the $z$-coordinate for clarity.
Figure~\ref{3dpoint} presents representative inference results, showing that the learned model consistently generates point clouds belonging to the same object class when conditioned on prompts from that class, further demonstrating the robustness of the proposed framework in high-dimensional geometric settings.

\begin{figure}[ht]
  \centering
  \foreach \y in {0,...,2}{%
    \foreach \x in {0,...,9}{%
      \includegraphics[width=0.09\linewidth]{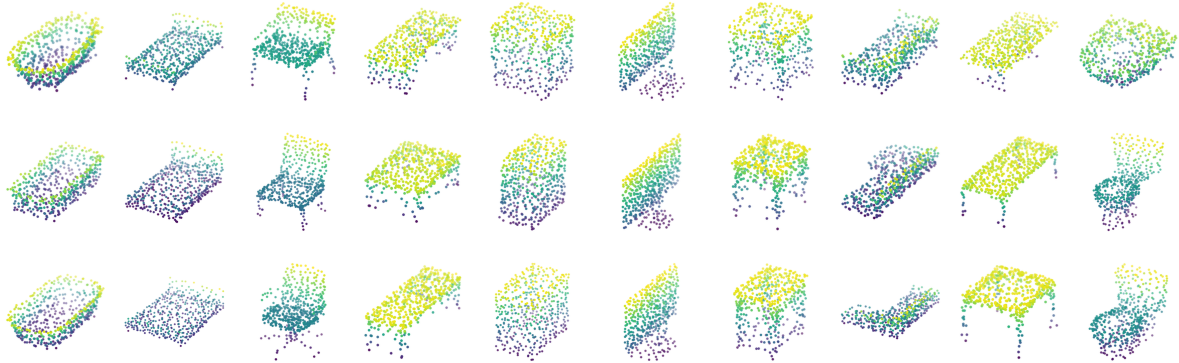}%
      \ifnum\x<9\hspace{\imgsep}\fi
    }%
    \par\medskip
  }%
  \caption{Visualization results of Gaussian to ModelNet10. Each column corresponds to the inference results of one task, for which the prompts are images conditional sampled from the corresponding single object class in ModelNet10. 
  }
  \label{3dpoint}
\end{figure}

To assess the quality of the inferred transport map produced by the learned solution operator, we compute the unbiased maximum mean discrepancy $\mathrm{MMD}_u$ between samples generated by the learned map and ground-truth samples drawn from $\rho_1^k$.

We adopt a multi-scale RBF kernel for $\mathrm{MMD}_u$, defined as
\[
k(u,v)
= \sum_{l=1}^L w_l \exp\!\left(-\frac{\|u-v\|_2^2}{\sigma_l}\right),
\qquad 
w_l \ge 0,\quad \sum_{l} w_l = 1.
\]
The bandwidths $\{\sigma_l\}_{l=1}^L$ follow a geometric progression,
\[
\sigma_l = \sigma_0\, 2^{\,l - \lceil L/2 \rceil},
\qquad l = 1,\dots,L.
\]
In all experiments, we fix $L = 5$. The base bandwidth $\sigma_0$ is chosen adaptively for each batch as the mean pairwise squared distance across all pooled samples. Given predicted samples $\hat{Y}=\{\hat{y}_j\}_{j=1}^{n}$ and ground-truth samples $Y=\{y_j\}_{j=1}^{n}$ from $\rho_1^k$, we set
\[
\sigma_0
= \frac{1}{2n(2n-1)}
\sum_{i \neq j} \|z_i - z_j\|_2^2,
\qquad
z_i \in \hat{Y} \cup Y.
\]

Table~\ref{tab:mmd_all_part} reports the per-class and overall MMD values for the MNIST, Fashion-MNIST, and ModelNet10 datasets. All values are uniformly scaled by $10^2$. The consistently small MMD values indicate that the proposed ICL model successfully captures the underlying task distribution across different datasets and modalities.  

\begin{table}[htbp]
\centering
\caption{Per-class and overall MMD results (after scaling up by $10^2$) for each dataset.}
\label{tab:mmd_all_part}
\setlength{\tabcolsep}{4.2pt}
\begin{tabular}{l|c|cccccccccc}
\hline
Dataset & Overall & 0 & 1 & 2 & 3 & 4 & 5 & 6 & 7 & 8 & 9 \\
\hline
MNIST
& 0.419 & 0.222 & 0.053 & 0.610 & 0.237 & 0.127 & 0.411 & 0.569 & 0.698 & 0.827 & 0.437 \\
Fashion-MNIST
& 0.350 & 0.323 & 0.371 & 0.265 & 0.532 & 0.357 & 0.438 & 0.287 & 0.417 & 0.227 & 0.278 \\
ModelNet10
& 4.27  & 3.96  & 4.20  & 1.59  & 2.86  & 9.05  & 2.21  & 6.49  & 1.07  & 8.40  & 2.84 \\
\hline
\end{tabular}
\end{table}

\subsection{Scaling law validation}\label{sec:ScalingLaw}
We validate the theoretical results of Section \ref{subsec: parametric}. Assume that the reference measure $\rho_0$ is the standard Gaussian, we sample $\rho_1$ as $\rho_1 = \mathcal{N}(0,\Sigma)$ with $\Sigma = \textrm{diag}(\sigma^2,\sigma^2)$ and $\sigma^2 \sim \textrm{Unif}[1,3].$ We empirical test the scaling predictions in Theorems 2 and 3 with respect to the prompt length $n$. During training, we randomly select $n$ from $\{600,800,1000,1200,1400,1600\},$ and during testing we choose $n = 5000.$ The model architecture is the linear transformer described in Section 2.3, and we choose $\lambda = 1000$ for the regularization parameter. Theorem \ref{thm: lossgap} predicts that the \textbf{excess loss} scales as
$$ \mathcal{R}(\hat{\theta}) - \min \mathcal{R} = O\left( \sqrt{\frac{1}{n}} + \frac{\lambda}{n} + \textrm{const} \right),
$$
where the constant accounts for other terms in the estimate of Theorem \ref{thm: lossgap} which are independent of $n$. Similarly, Theorem \ref{thm: TPGE} predicts the \textbf{transport map generalization error} to scale as 
$$ \E \left[ \left\| \hat{A}_n- \Sigma^{1/2} \right\|_F^2 \right] = O\left( \sqrt{\frac{1}{n}} + \frac{\lambda}{n} + \textrm{const} \right).
$$
On left picture of Figure~\ref{fig:quan1}, we plot the excess loss as a function of the prompt length $n$ and fit the resulting curve using least squares. Similarly, the right picture of Figure~\ref{fig:quan1} shows the transport-map estimation error versus the prompt length $n$, together with a least-squares fit. In both cases, the coefficients of determination are close to one, confirming that the empirical scaling behavior aligns with our theoretical prediction on how the error depends on the context length.

\begin{figure}[ht] % h = here, t = top, b = bottom, p = page of floats
    \centering
    \includegraphics[width=0.49\textwidth]{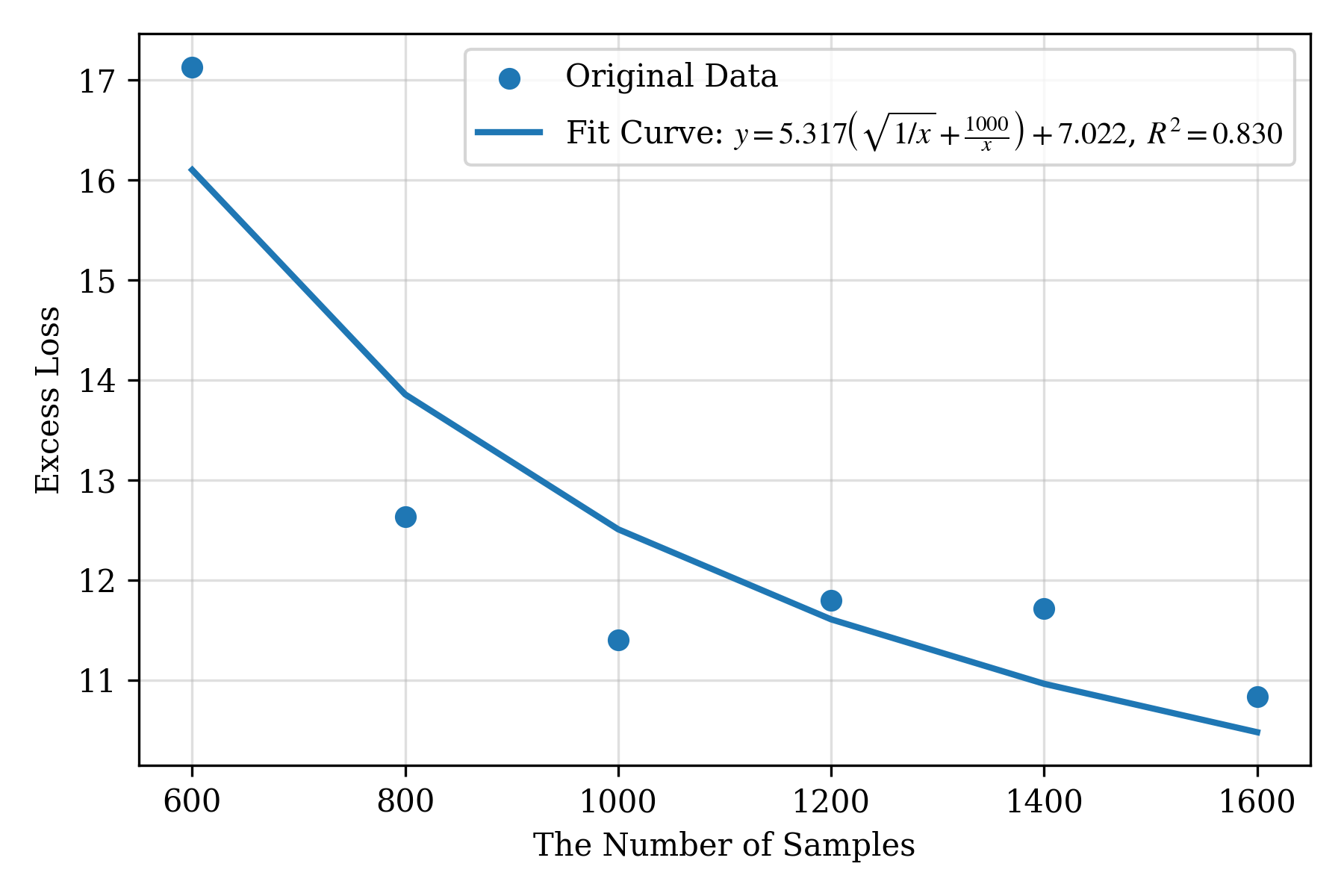} % image file name
    \includegraphics[width=0.49\textwidth]{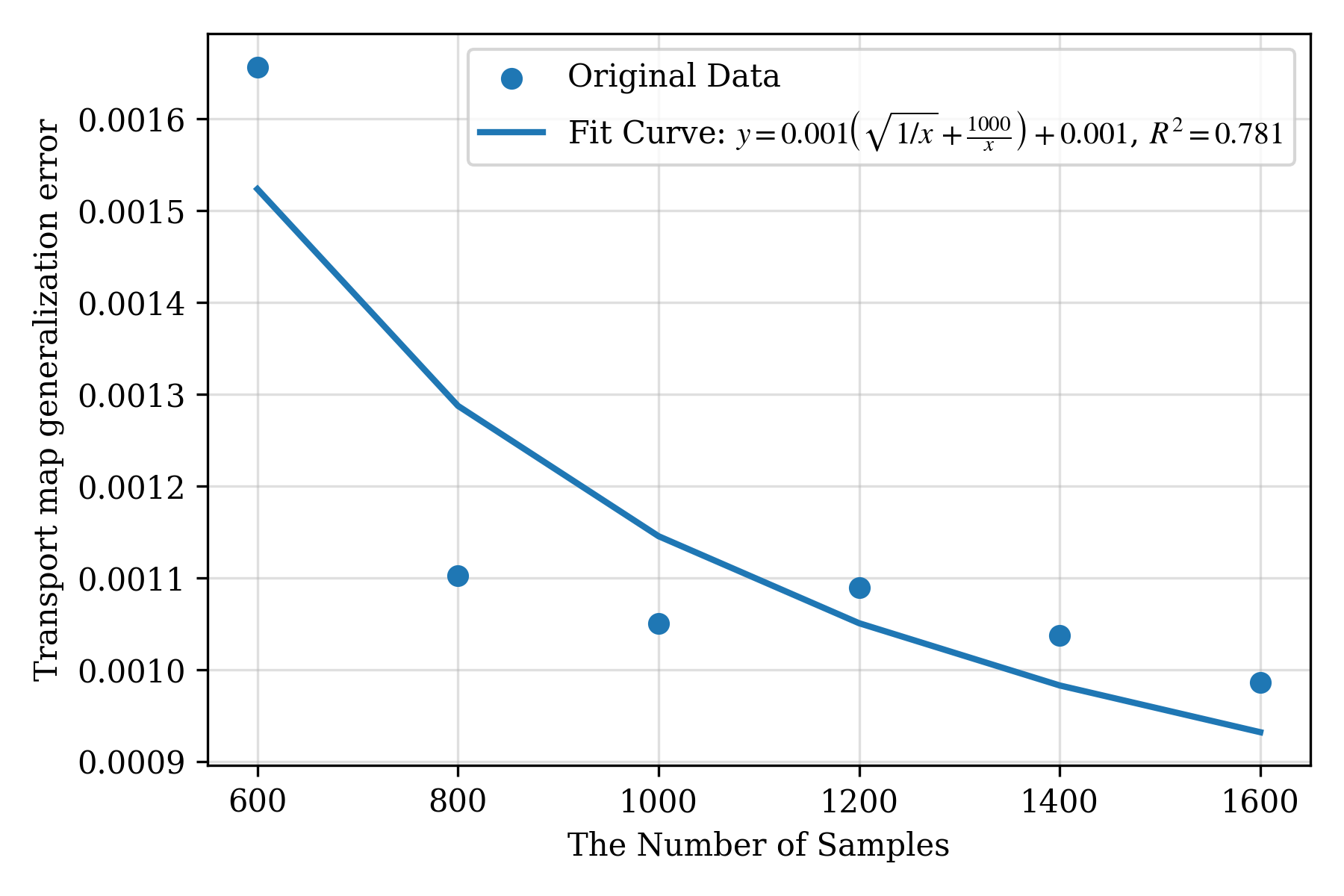}
    \caption{Left: Excess loss versus number of samples with a least-squares fit with the coefficient of determination $R^2 = 0.830$. Right: Transport map generalization error versus number of samples with a least-squares fit with the coefficient of determination $R^2 = 0.781$.}
    \label{fig:quan1}
\end{figure}

\section{Conclusion}\label{sec: conclusion}
We introduced an in-context operator learning framework for estimating optimal transport maps between probability measures using only few-shot samples as prompts and without requiring gradient updates at inference time. This perspective treats OT map estimation as a contextual prediction problem and enables a single learned operator to generalize across a family of transport tasks.

Our analysis covers both nonparametric and parametric regimes. In the nonparametric setting, we proved generalization bounds under the assumption that source--target distribution pairs lie on a low-dimensional task manifold, yielding quantitative rates that reflect prompt size, intrinsic task dimension, and model capacity. In the parametric Gaussian setting, we constructed an explicit architecture that recovers the exact OT map in context and established finite-sample excess-risk bounds.
On the algorithmic side, we developed a transformer-based model capable of processing variable-size prompts and producing set-to-set transport predictions. Numerical experiments on synthetic transports, image-based datasets, and 3D point clouds demonstrate that the proposed method accurately infers conditional mappings and exhibits the scaling behavior predicted by theory.

Overall, these results show that in-context operator learning provides a unified and effective approach for learning families of optimal transport maps. Future directions include extending the framework to other operators on probability spaces, incorporating structural priors on task manifolds, and integrating in-context transport operators into broader generative modeling pipelines.

\section{Acknowledgment}
F. Cole and Y. Lu thank the support from the NSF CAREER Award DMS-2442463. D. Wang and R. Lai's research are partially supported by NSF DMS-2401297.

\bibliographystyle{plainnat}
\bibliography{ref}

\newpage

\appendix

\section{Overview of Notation}\label{app: notation}

\begin{table}\label{tab: notation}
\caption{Overview of notation used in paper} 
\begin{tabularx}{\textwidth}{@{}XX@{}}
\toprule
\textbf{Sets} \\
$\mathcal{F}$ & Space of Lipschitz ICL predictors \\
$\mathcal{M}$ & Task manifold of probability measures \\
\textbf{Functions/operators} \\
  $\mathcal{G}_{OT}$ & Optimal transport solution operator \\
  $\mathcal{G}^{\dagger}$  & Minimizer of $\mathcal{R}$ \\
  $\mathcal{G}^{\ast}$ & Best approximation to $G^{\dagger}$ within $\mathcal{F}$ \\
  $\mathcal{G}^{\ast}_\Lambda$ & Best  approximation within $\mathcal{F}$ given dataset $\Lambda$\\
  $\theta$ & Transformer parameters \\
  $A_{n,\theta}$ & In-context mapping induced by $\theta$ \\
  $\mathcal{R}$ & Population risk functional \\
  $\mathcal{R}_{\Lambda}$ & Empirical risk functional \\
  $\mathcal{R}^1$ & Transport cost term of $\mathcal{R}$ \\
  $\mathcal{R}^2$ & Terminal cost term of $\mathcal{R}$ \\
   $\mathcal{R}^1_{\Lambda}$ & Transport cost term of $\mathcal{R}_{\Lambda}$ \\
  $\mathcal{R}_{\Lambda}^2$ & Terminal cost term of $\mathcal{R}_{\Lambda}$ \\
\textbf{Probability distributions} \\
$\mu$ & Task distribution on $\mathcal{M}$ \\
$(\rho_0,\rho_1)$ & Pair of distributions sampled from $\mu$ \\
$\Sigma$ & Covariance of Gaussian distribution \\
\textbf{Statistical constants} \\
$d$ & Dimension of data distribution \\
$d_{\mathcal{M}}$ & Dimension of task distribution \\
$N$ & Number of pre-training tasks \\
$n$ & Number of samples per task \\
$s$ & Prompt size \\
$\lambda$ & Regularization parameter \\
$M$ & Capacity bound on feedforward neural nets \\
$R$ & Truncation radius bound \\
\bottomrule
\end{tabularx}
\end{table}

\section{Proofs for Section \ref{sec: nonparametricproofs}}\label{appendix:nonparametricproofs}
\paragraph{Transport cost bound:} We prove Lemma \ref{lem:transport} which estimates the statistical error for the transport cost. When $\gX,\gY$ are compact subspace in $\mathbb{R}^d$, Assumption \ref{ass:moment} is no longer needed. A similar proof with minor modification yields a tighter bound.

\begin{proof}[Lemma \ref{lem:transport}]
Let $\Lambda' =\{\nu^{k_2},\nu^{k_2}_1, \hat\nu^{k_2}_0,\hat\nu^{k_2}_1,\tilde{\nu}^{k_2}_0,\tilde{\nu}^{k_2}_1\}_{k_2=1}^N$ be another training set, generated independently of $\Lambda$.
By Jensen's inequality
\begin{align}\label{ineq_tran1}
   \mathbb{E}_\Lambda\sup_{\G\in\mathcal{F}}|\mathcal{R}^1(\G) - \mathcal{R}^1_\Lambda(\G)| &= \mathbb{E}_{\Lambda}\sup_{\G\in\mathcal{F}}\big| \mathbb{E}_{\Lambda'}[\mathcal{R}^1_{\Lambda'}(\G) - \mathcal{R}^1_\Lambda(\G)]\big|\nonumber\\
    &\leq \mathbb{E}_{\Lambda}\mathbb{E}_{\Lambda'}\sup_{\G\in\mathcal{F}}|\mathcal{R}^1_{\Lambda'}(\G) - \mathcal{R}^1_\Lambda(\G)|,
\end{align}
where $\mathcal{R}^1_{\Lambda'}(\G)$ is given by
\begin{align*}
    \mathcal{R}^1_{\Lambda'}(\G)
    =\frac{1}{Nn}\sum^N_{k_2=1}\sum^n_{j_2=1}\|\gG(\hat{\nu}^{k_2}_0,\hat{\nu}^{k_2}_1;x^{k_2}_{j_2})-x^{k_2}_{j_2}\|^2.
\end{align*}
Let $\pi$ defined on $\gP(\gM)\times\gP(\gM)$ be a coupling between the empirical measures $\frac{1}{N}\sum^N_{k_1=1}\delta_{(\rho^{k_1}_0, \rho^{k_1}_1)}$ and $\frac{1}{N}\sum^N_{k_2=1}\delta_{(\nu^{k_2}_0, \nu^{k_2}_1)}$, and $\gamma^{k_1,k_2}$ be a coupling between  $(\Tilde{\rho}^{k_1}_0,\Tilde{\rho}^{k_1}_1)$ and $(\Tilde{\nu}^{k_2}_0,\Tilde{\nu}^{k_2}_1)$. For any $\G\in\gF$, by the uniform boundedness of $\gG$ and Fubini's theorem, 
\begin{align}\label{ineq:1'-1}
    &\big|\mathcal{R}^1_{\Lambda'}(\G) - \mathcal{R}^1_\Lambda(\G)\big|\nonumber\\
    =&\Big|\frac{1}{Nn}\sum^N_{k_1=1}\sum^n_{j_1=1}\|\G(\hat{\rho}^{k_1}_0,\hat{\rho}^{k_1}_1;x^{k_1}_{j_1})-x^{k_1}_{j_1}\|^2
    - \frac{1}{Nn}\sum^N_{k_2=1}\sum^n_{j_2=1}\|\G(\hat{\nu}^{k_2}_0,\hat{\nu}^{k_2}_1;x^{k_2}_{j_2})-x^{k_2}_{j_2}\|^2\Big|\nonumber\\
    \leq& \sum^N_{k_1,k_2=1}\pi_{k_1,k_2}\sum^n_{j_1,j_2=1}\gamma^{k_1,k_2}_{j_1,j_2}\Big|\|\G(\hat{\rho}^{k_1}_0,\hat{\rho}^{k_1}_1;x^{k_1}_{j_1})-x^{k_1}_{j_1}\|^2 - \|\G(\hat{\nu}^{k_2}_0,\hat{\nu}^{k_2}_1;x^{k_2}_{j_2})-x^{k_2}_{j_2}\|^2\Big|.
\end{align}
For all $\G\in\mathcal{F}$, $(\rho_0,\rho_1),~(\rho'_0,\rho'_1)\in\gM$ and $x,x'\in\gX$, Assumption \ref{ass:bdd&lips} implies
    \begin{align}\label{ineq:L2diff}
        &\big|\|\G(\rho'_0,\rho'_1;x')-x'\|^2 - \|\G(\rho_0,\rho_1;x)-x\|^2\big|\nonumber\\
        \leq& \big(\|\G(\rho'_0,\rho'_1;x')-x'\| + \|\G(\rho_0,\rho_1;x)-x\|\big)\cdot\big\|\G(\rho'_0,\rho'_1;x')-x' - \G(\rho_0,\rho_1;x)+x\big\|\nonumber\\
        \leq& (2M_\mathcal{F}+\|x\|+\|x'\|)\big(d(x,x')+L_\mathcal{F}(d(x,x')+W_2(\rho_0, \rho'_0) + W_2(\rho_1,\rho'_1))\big)\nonumber\\
        =& 2M_\mathcal{F}L_\mathcal{F}\big(W_2(\rho_0, \rho'_0) + W_2(\rho_1,\rho'_1)\big)+L_\mathcal{F}(\|x\|+\|x'\|)\big(W_2(\rho_0, \rho'_0) + W_2(\rho_1,\rho'_1)\big)\nonumber\\
        &+2M_\mathcal{F}(L_\mathcal{F}+1)d(x,x')+ (L_\mathcal{F}+1)(\|x\|+\|x'\|)d(x,x').
    \end{align} 
Plug inequality \eqref{ineq:L2diff} into \eqref{ineq:1'-1} with substitution implies that
\begin{align*}
    &|\mathcal{R}^1_{\Lambda'}(\G) - \mathcal{R}^1_\Lambda(\G)|\nonumber\\
    \leq& 2 M_\mathcal{F}L_\mathcal{F}\sum^N_{k_1,k_2=1}\pi_{k_1,k_2}\sum^n_{j_1,j_2=1}\gamma^{k_1,k_2}_{j_1,j_2} \big(W_2(\hat{\rho}^{k_1}_0,\hat{\nu}^{k_2}_0) + W_2(\hat{\rho}^{k_1}_1,\hat{\nu}^{k_2}_1)\big) \\
    &+ L_\mathcal{F}\sum^N_{k_1,k_2=1}\pi_{k_1,k_2}\sum^n_{j_1,j_2=1}\gamma^{k_1,k_2}_{j_1,j_2}(\|x^{k_1}_{j_1}\|+\|x^{k_2}_{j_2}\|)\big(W_2(\hat{\rho}^{k_1}_0,\hat{\nu}^{k_2}_0) + W_2(\hat{\rho}^{k_1}_1,\hat{\nu}^{k_2}_1\big)\\
    &+ 2M_\mathcal{F}(L_\mathcal{F}+1) \sum^N_{k_1,k_2=1}\pi_{k_1,k_2}\sum^n_{j_1,j_2=1}\gamma^{k_1,k_2}_{j_1,j_2}d(x^{k_1}_{j_1},x^{k_2}_{j_2})\\
    &+(L_\mathcal{F}+1)\sum^N_{k_1,k_2=1}\pi_{k_1,k_2}\sum^n_{j_1,j_2=1}\gamma^{k_1,k_2}_{j_1,j_2}(\|x^{k_1}_{j_1}\|+\|x^{k_2}_{j_2}\|)d(x^{k_1}_{j_1},x^{k_2}_{j_2}) =: \mathrm{I+II+III+IV.}
\end{align*}
We estimate each term separately. Since the prompts for each task are generated independently of the samples, and the coupling $\gamma^{k_1,k_2}$ is between the empirical measures of samples. We have 
\begin{align}\label{ineq:I}
    \mathrm{I} = 2 M_\mathcal{F}L_\mathcal{F} \sum^N_{k_1,k_2=1}\pi_{k_1,k_2}\big(W_2(\hat{\rho}^{k_1}_0,\hat{\nu}^{k_2}_0) + W_2(\hat{\rho}^{k_1}_1,\hat{\nu}^{k_2}_1)\big).
\end{align}
By Cauchy-Schwarz inequality and choosing the optimal $\gamma_*^{k_1,k_2}$, $\mathrm{III}$ is bounded by
\begin{align}\label{ineq:III}
    \mathrm{III}&=2M_\mathcal{F}(L_\mathcal{F}+1)\sum^N_{k_1,k_2=1}\pi_{k_1,k_2}\Big[\Big(\sum^n_{j_1,j_2=1}(\gamma^{k_1,k_2}_{j_1,j_2})^{\frac{1}{2}}(\gamma^{k_1,k_2}_{j_1,j_2})^{\frac{1}{2}}d(x^{k_1}_{j_1},x^{k_2}_{j_2})\Big)^2\Big]^{\frac{1}{2}}\nonumber\\
    &\leq 2M_\mathcal{F}(L_\mathcal{F}+1)\sum^N_{k_1,k_2=1}\pi_{k_1,k_2}\Big[\Big(\sum^n_{j_1,j_2=1}\gamma^{k_1,k_2}_{j_1,j_2}\Big)\Big(\sum^n_{j_1,j_2=1}\gamma^{k_1,k_2}_{j_1,j_2}d^2(x^{k_1}_{j_1},x^{k_2}_{j_2})\Big)\Big]^{\frac{1}{2}}\nonumber\\
    &\leq 2M_\mathcal{F}(L_\mathcal{F}+1)\sum^N_{k_1,k_2=1}\pi_{k_1,k_2}\Big[\sum^n_{j_1,j_2=1}\gamma^{k_1,k_2}_{j_1,j_2}\big(d^2(x^{k_1}_{j_1},x^{k_2}_{j_2})+d^2(y^{k_1}_{j_1},y^{k_2}_{j_2})\big)\Big]^{\frac{1}{2}}\nonumber\\
    &\stackrel{\gamma_*}{=} 2M_\mathcal{F}(L_\mathcal{F}+1)\sum^N_{k_1,k_2=1}\pi_{k_1,k_2}\big(W_2(\Tilde{\rho}^{k_1}_0,\Tilde{\nu}^{k_2}_0)+W_2(\Tilde{\rho}^{k_1}_1,\Tilde{\nu}^{k_2}_1)\big).
\end{align}
Again, by Cauchy-Schwarz inequality and Assumption \ref{ass:moment}, and choosing the optimal $\gamma_*^{k_1,k_2}$
\begin{align}\label{ineq:IV}
    \mathrm{IV}&\leq(L_\mathcal{F}+1)\sum^N_{k_1,k_2=1}\pi_{k_1,k_2} \Big[\sum^n_{j_1,j_2=1}\gamma^{k_1,k_2}_{j_1,j_2}(\|x^{k_1}_{j_1}\|+\|x^{k_2}_{j_2}\|)^2\Big]^{\frac{1}{2}}\Big[\sum^n_{j_1,j_2=1}\gamma^{k_1,k_2}_{j_1,j_2}d^2(x^{k_1}_{j_1},x^{k_2}_{j_2})\Big]^{\frac{1}{2}}\nonumber\\
    &\stackrel{\gamma_*}{\leq}(L_\mathcal{F}+1)\sum^N_{k_1,k_2=1}\pi_{k_1,k_2}\Big[\sum^n_{j_1,j_2=1}\gamma^{k_1,k_2}_{j_1,j_2}(2\|x^{k_1}_{j_1}\|^2+2\|x^{k_2}_{j_2}\|^2)\Big]^{\frac{1}{2}}\Big(W_2(\Tilde{\rho}^{k_1}_0,\Tilde{\nu}^{k_2}_0)+W_2(\Tilde{\rho}^{k_1}_1,\Tilde{\nu}^{k_2}_1)\Big)\nonumber\\
    &\leq(L_\mathcal{F}+1)\sum^N_{k_1,k_2=1}\pi_{k_1,k_2}2\widetilde{M}^{\frac{1}{2}}\Big[\sum^n_{j_1,j_2=1}\gamma^{k_1,k_2}_{j_1,j_2}\Big]^{\frac{1}{2}}\big(W_2(\Tilde{\rho}^{k_1}_0,\Tilde{\nu}^{k_2}_0)+W_2(\Tilde{\rho}^{k_1}_1,\Tilde{\nu}^{k_2}_1)\big)\nonumber\\
    &=2(L_\mathcal{F}+1)\widetilde{M}^{\frac{1}{2}}\sum^N_{k_1,k_2=1}\pi_{k_1,k_2}\big(W_2(\Tilde{\rho}^{k_1}_0,\Tilde{\nu}^{k_2}_0)+W_2(\Tilde{\rho}^{k_1}_1,\Tilde{\nu}^{k_2}_1)\big).
\end{align}
II is bounded similarly, it holds for any coupling $\gamma^{k_1,k_2}$ that 
\begin{align}\label{ineq:II}
    \mathrm{II}&\leq L_\mathcal{F}\sum^N_{k_1,k_2=1}\pi_{k_1,k_2}\Big[\sum^n_{j_1,j_2=1}\gamma^{k_1,k_2}_{j_1,j_2}(\|x^{k_1}_{j_1}\|+\|x^{k_2}_{j_2}\|)^2\Big]^{\frac{1}{2}}\Big(W_2(\hat{\rho}^{k_1}_0,\hat{\nu}^{k_2}_0) + W_2(\hat{\rho}^{k_1}_1,\hat{\nu}^{k_2}_1)\Big)\nonumber\\
    &\leq 2 L_\mathcal{F}\widetilde{M}^{\frac{1}{2}}\sum^N_{k_1,k_2=1}\pi_{k_1,k_2}\big(W_2(\hat{\rho}^{k_1}_0,\hat{\nu}^{k_2}_0) + W_2(\hat{\rho}^{k_1}_1,\hat{\nu}^{k_2}_1)\big).
\end{align}
Combining \eqref{ineq:I}, \eqref{ineq:III},\eqref{ineq:IV} and \eqref{ineq:II} with the same choice of $\gamma_*^{k_1,k_2}$, we obtain
\begin{align*}
    |\mathcal{R}^1_{\Lambda'}(\G) - \mathcal{R}^1_\Lambda(\G)|\leq& 2L_\mathcal{F}(M_\mathcal{F}+\widetilde{M}^{\frac{1}{2}})\sum^N_{k_1,k_2=1}\pi_{k_1,k_2}\big(W_2(\hat{\rho}^{k_1}_0,\hat{\nu}^{k_2}_0) + W_2(\hat{\rho}^{k_1}_1,\hat{\nu}^{k_2}_1)\big)\nonumber\\
    &+ 2 (L_\mathcal{F}+1)(M_\mathcal{F}+\widetilde{M}^{\frac{1}{2}})\sum^N_{k_1,k_2=1}\pi_{k_1,k_2}\big(W_2(\Tilde{\rho}^{k_1}_0,\Tilde{\nu}^{k_2}_0)+W_2(\Tilde{\rho}^{k_1}_1,\Tilde{\nu}^{k_2}_1)\big).
\end{align*}
By triangle inequality, we have
\begin{align}\label{ineq:tri}
    W_2(\Tilde{\rho}^{k_1}_0,\Tilde{\nu}^{k_2}_0)- W_2(\rho^{k_1}_0, \nu^{k_2}_0) &\leq W_2(\Tilde{\rho}^{k_1}_0, \rho^{k_1}_0) + W_2(\Tilde{\nu}^{k_2}_0,\nu^{k_2}_0),
    \nonumber\\
    W_2(\hat{\rho}^{k_1}_0,\hat{\nu}^{k_2}_0) - W_2(\rho^{k_1}_0, \nu^{k_2}_0) &\leq W_2(\hat{\rho}^{k_1}_0, \rho^{k_1}_0) + W_2(\hat{\nu}^{k_2}_0,\nu^{k_2}_0)
\end{align}
By Assumption \ref{ass:moment} and Theorem 1 in \cite{Guillin15}, after taking expectation, 
\begin{align}\label{eq:expectW2}
    \mathbb{E}W_2(\Tilde{\rho}^{k_1}_0, \rho^{k_1}_0) \lesssim \widetilde{M}^{\frac{1}{3}}n^{\max\{-\frac{1}{6},-\frac{1}{d}\}},\quad\textrm{and}~~~~\mathbb{E}W_2(\hat{\rho}^{k_1}_0,\hat{\nu}^{k_2}_0)\lesssim \widetilde{M}^{\frac{1}{3}}s^{\max\{-\frac{1}{6},-\frac{1}{d}\}}.
\end{align}
Similar inequalities holds for $\mathbb{E}W_2(\hat{\rho}^{k_1}_1,\hat{\nu}^{k_2}_1),~\mathbb{E}W_2(\Tilde{\rho}^{k_1}_1,\Tilde{\nu}^{k_2}_1)$. According to \eqref{ineq:tri}, 
\begin{align*}
    &|\mathcal{R}^1_{\Lambda'}(\G) - \mathcal{R}^1_\Lambda(\G)|\\
    \leq&
    2L_\mathcal{F}(M_\mathcal{F}+\widetilde{M}^{\frac{1}{2}})\Big\{\sum^N_{k_1,k_2=1}\pi_{k_1,k_2}\big(W_2(\rho^{k_1}_0,\nu^{k_2}_0) + W_2(\rho^{k_1}_1,\nu^{k_2}_1)\big) \\
    &\qquad\qquad\qquad\qquad\qquad+\pi_{k_1,k_2}\big(W_2(\hat{\rho}^{k_1}_0, \rho^{k_1}_0) + W_2(\hat{\nu}^{k_2}_0,\nu^{k_2}_0)
    +W_2(\hat{\rho}^{k_1}_1, \rho^{k_1}_1) + W_2(\hat{\nu}^{k_2}_1,\nu^{k_2}_1)\big)
    \Big\}\\
    &+  2(L_\mathcal{F}+1)(M_\mathcal{F}+\widetilde{M}^{\frac{1}{2}}) \Big\{\sum^N_{k_1,k_2=1}\pi_{k_1,k_2}\big(W_2(\rho^{k_1}_0,\nu^{k_2}_0)+W_2(\rho^{k_1}_1,\nu^{k_2}_1)\big) \\
    &\qquad\qquad\qquad\qquad\qquad\qquad~~~~~+\pi_{k_1,k_2}\big(W_2(\Tilde{\rho}^{k_1}_0, \rho^{k_1}_0) + W_2(\Tilde{\nu}^{k_2}_0,\nu^{k_2}_0)+W_2(\Tilde{\rho}^{k_1}_1, \rho^{k_1}_1) + W_2(\Tilde{\nu}^{k_2}_1,\nu^{k_2}_1)\big)
    \Big\}.
\end{align*}
Proposition 3.2 in \citep{Smale08} indicates that, given $0<\epsilon_*<\frac{\tau}{2}$ and $N=O(-d_{\gM}{\epsilon_*}^{-d_{\gM}}\log{\epsilon_*})$, the uniformly distributed tasks $(\rho^{k_1}_0,\rho^{k_1}_1)^N_{k_1=1}$ form an $\epsilon_*/2$-net on the task manifold $\mathcal{M}$ with probability greater than $1-\eta$. We choose $\pi_{k_1,k_2}$ such that $(\nu^{k_2}_0, \nu^{k_2}_1)$ is within $\epsilon_*/2$ neighborhood of some $(\rho^{k_1}_0,\rho^{k_1}_1)$ for every $k_2$. 

Let $\mathcal{E}$ be the event that $(\rho^{k_1}_0,\rho^{k_1}_1)^N_{k_1=1}$ forms a $\epsilon_*/2$-net on $\mathcal{M}$, and fix $\eta = \epsilon_*$. 
Taking expectations over $\Lambda, \Lambda'$ and plugging in \eqref{eq:expectW2}, we obtain
\begin{align}\label{ineq_tran2}
    &\mathbb{E}_{\Lambda}\mathbb{E}_{\Lambda'}\sup_{\G\in\mathcal{F}}|\mathcal{R}^1_{\Lambda'}(\G) - \mathcal{R}^1_\Lambda(\G)|\nonumber\\
    =&\mathbb{E}_{\Lambda}\mathbb{E}_{\Lambda'}\big[\sup_{\G\in\mathcal{F}}|\mathcal{R}^1_{\Lambda'}(\G) - \mathcal{R}^1_\Lambda(\G)|\big|\mathcal{E}\big] + \mathbb{E}_{\Lambda}\mathbb{E}_{\Lambda'}\big[\sup_{\G\in\mathcal{F}}|\mathcal{R}^1_{\Lambda'}(\G) - \mathcal{R}^1_\Lambda(\G)|\big|\mathcal{E}^c\big]\nonumber\\
    \lesssim&  (4L_\mathcal{F}+2)(M_\mathcal{F}+\widetilde{M}^{\frac{1}{2}})(1+\text{diam}(\mathcal{M}))\epsilon_* + 2L_\mathcal{F}(M_\mathcal{F}+\widetilde{M}^{\frac{1}{2}})\widetilde{M}^{\frac{1}{3}}s^{\max\{-\frac{1}{6},-\frac{1}{d}\}}\nonumber \\%[l^{-1/1+d_{\gM}}(1-\eta) + \text{diam}(\mathcal{M})\eta]
    &+ 2(L_\mathcal{F}+1)(M_\mathcal{F}+\widetilde{M}^{\frac{1}{2}})\widetilde{M}^{\frac{1}{3}}n^{\max\{-\frac{1}{6},-\frac{1}{d}\}} + O(\epsilon^2_*)\nonumber\\
    =& O\big(N^{-\frac{1}{d_{\gM}+1}}+ s^{\max\{-\frac{1}{6},-\frac{1}{d}\}}+n^{\max\{-\frac{1}{6},-\frac{1}{d}\}}\big),
\end{align}
where in the last inequality we use the relation that $\epsilon_*=O\big(N^{-\frac{1}{d_{\gM}+1}}\big)$. Combining with \eqref{ineq_tran1}, we establish the bound in Lemma \ref{lem:transport}.
\end{proof}
\paragraph{Terminal cost bound:}
\begin{proof}[Lemma \ref{lem:terminal}]
We prove Lemma \ref{lem:terminal} mainly using quadratic kernel. The proof holds with kernel functions satisfying Assumption \ref{ass:kernel} with minor modification.
Since Assumption \ref{ass:moment} provides a uniform bound on the third moment, thus  by Markov's inequality, for any $\rho_1\in\gP(\R^d),~y\sim\rho_1$
\begin{align}\label{ineq:markov}
    \mathbb{P}(\|y\|>R) = \mathbb{P}(\|y\|^3>R^3) \leq \frac{\mathbb{E}[\|y\|^3]}{R^3}\leq \frac{\widetilde{M}}{R^3}, ~~~~R>0.
\end{align}
We later choose $R>0$ sufficient large so that $R>\widetilde{M}$.  With probability at least $1-\frac{\widetilde{M}}{R^3}$, the domain of $\rho_1$ is within a ball $B(R)\subset\R^d$.   The polynomial decay rate in \eqref{ineq:markov} already serves the purpose of the proof. The target Gaussian measure as considered in the parametric case enjoys an exponential decay thus satisfies \eqref{ineq:markov}. 
%$\gR^2(\G)$ and $\gR^2_{\Lambda}(\G)$ restricted on $B(R)$ 
Denote the truncated population and empirical risk functionals as $\gR^2_t(\G),\gR^2_{\Lambda,t}(\G)$ with $y\sim \rho_1$ restricted on $B(R)$, i.e.,
\begin{align*}
    &\gR^2_t(\G) = \lambda\mathbb{E}_{(\rho_0,\rho_1)\sim\mu}\mathrm{MMD}^2(\gG(\hat{\rho}_0,\hat{\rho}_1)_{\#}\rho_0,\rho_1\mathds{1}_{y\in B(R)}),\\
    &\gR^2_{\Lambda,t}(\G) = \frac{\lambda}{N}\sum^N_{k=1} \mathrm{MMD}^2_u(\gG(\hat{\rho}^{k}_0,\hat{\rho}^{k}_1)_\# \Tilde{\rho}^k_0, \Tilde{\rho}^k_1\mathds{1}_{y\in B(R)}). 
`\end{align*}
Without loss of generality, we assume that the  training data for prompts and samples are in $B(R)$ so that $\gR^2_{\Lambda}(\G)\equiv \gR^2_{\Lambda,t}(\G)$ for all $\G\in\gF$. Then we have the error decomposition:
\begin{align}\label{mmd_error}
    &\mathbb{E}_\Lambda \sup_{\G\in\f}|\mathcal{R}^2(\G) -\mathcal{R}^2_\Lambda(\G)|\nonumber\\
    =& \mathbb{E}_\Lambda \sup_{\G\in\f}|\mathcal{R}^2(\G) - \mathcal{R}^2_t(\G) + \mathcal{R}^2_t(\G) - \mathcal{R}^2_{\Lambda,t}(\G) +\mathcal{R}^2_{\Lambda,t}(\G)- \mathcal{R}^2_\Lambda(\G)|\nonumber\\
    \leq& \underbrace{\sup_{\G\in\f}|\mathcal{R}^2(\G) - \mathcal{R}^2_t(\G)|}_{\textrm{truncation error}} + \underbrace{ \mathbb{E}_\Lambda \sup_{\G\in\f}|\mathcal{R}^2_t(\G) - \mathcal{R}^2_{\Lambda,t}(\G)| }_{\textrm{statistical error}} + \underbrace{\mathbb{E}_\Lambda \sup_{\G\in\f}|\mathcal{R}^2_{\Lambda,t}(\G)- \mathcal{R}^2_\Lambda(\G)|}_{=0}.
\end{align}
We first bound the statistical error restricted to the bounded domain $B(R)$ via Rademacher complexity.  
Since $\mathrm{MMD}^2_u$ is an unbiased empirical estimate of $\mathrm{MMD}$, the symmetrization method in \eqref{ineq:sym_rad} cannot be applied. 
We thus denote the biased statistics of $\mathcal{R}^2_t(\G)$, which corresponds to a direct Riemann sum discretization of the squared MMD functional, as  
\begin{align*}
    \Bar{\mathcal{R}}^2_{\Lambda,t}(\G)=\frac{\lambda}{N}\sum^N_{k=1} \mathrm{MMD}^2(\gG(\hat{\rho}^{k}_0,\hat{\rho}^{k}_1)_\# \Tilde{\rho}^k_0, \Tilde{\rho}^k_1).
\end{align*}
Accordingly, we further decompose the statistical error into two parts: the discretization error between population risk and the empirical risk using biased MMD estimator and the error due to the difference between biased and unbiased estimators:
\begin{align}\label{barR2}
\mathbb{E}_\Lambda\sup_{\G\in\f}\big|\mathcal{R}^2_t(\G) -\mathcal{R}^2_{\Lambda,t}(\G)\big| \leq \mathbb{E}_\Lambda\sup_{\G\in\f}\big|\mathcal{R}^2_t(\G) -\Bar{\mathcal{R}}^2_{\Lambda,t}(\G)\big| + \mathbb{E}_\Lambda\sup_{\G\in\f}\big|\Bar{\mathcal{R}}^2_{\Lambda,t}(\G) -\mathcal{R}^2_{\Lambda,t}(\G)\big|,
\end{align}
where 
\begin{align*}
|\mathcal{R}^2_t(\G) -\Bar{\mathcal{R}}^2_{\Lambda,t}(\G)\big| = \lambda
    \big|\mathbb{E}_{(\rho_0,\rho_1)}\mathrm{MMD}^2(\gG(\hat{\rho}_0,\hat{\rho}_1)_{\#} \rho_0,\rho_1\mathds{1}_{y\in B(R)}) - \frac{1}{N}\sum^N_{k=1}\mathrm{MMD}^2(\gG(\hat{\rho}^k_0,\hat{\rho}^k_1)_{\#} \Tilde\rho^k_0,\Tilde\rho^k_1)\big|
\end{align*}
and
\begin{align*}
    |\Bar{\mathcal{R}}^2_{\Lambda,t}(\G) -\mathcal{R}^2_{\Lambda,t}(\G)\big| = \lambda \big|\frac{1}{N}\sum^N_{k=1}\mathrm{MMD}_u^2(\gG(\hat{\rho}^k_0,\hat{\rho}^k_1)_{\#} \Tilde\rho^k_0,\Tilde\rho^k_1) -  \frac{1}{N}\sum^N_{k=1} \mathrm{MMD}^2(\gG(\hat{\rho}^{k}_0,\hat{\rho}^{k}_1)_\# \Tilde{\rho}^k_0, \Tilde{\rho}^k_1) \big| .
\end{align*}
The first term in \eqref{barR2} is bounded by the Rademacher complexity of $\mathrm{MMD}^2_u\circ\gF$ with respect to $\Lambda$. By symmetrization, 
\begin{align}\label{ineq:sym_rad}
    \mathbb{E}_\Lambda \sup_{\G\in\mathcal{F}}\big|\mathcal{R}^2_t(\G) -\Bar{\mathcal{R}}^2_{\Lambda,t}(\G)\big|
    &=\mathbb{E}_\Lambda \sup_{\G\in\mathcal{F}} \big| \Bar{\mathcal{R}}^2_{\Lambda,t}(\G) -\mathbb{E}_{\Lambda'}\Bar{\mathcal{R}}^2_{\Lambda',t}(\G) \big|\nonumber\\
    &\leq \mathbb{E}_\Lambda \sup_{\G\in\mathcal{F}}\mathbb{E}_{\Lambda'}\big| \Bar{\mathcal{R}}^2_{\Lambda,t}(\G) - \Bar{\mathcal{R}}^2_{\Lambda',t}(\G) \big|\leq 2\lambda\mathbb{E}_\Lambda [\mathrm{Rad}(\mathrm{MMD^2}\circ\gF\circ\Lambda)],
    %\leq& 2\lambda_2\mathbb{E}_\Lambda[\mathrm{Rad}(\mathrm{MMD}^2\circ\gF\circ\Lambda)]\leq \frac{8\lambda_2L_k}{\sqrt{n}(n-1)}\mathbb{E}_\Lambda[\mathrm{Rad}(\gF\circ\Lambda)].
\end{align}
where the empirical Rademacher complexity is defined as
\begin{align*}
    \mathrm{Rad}(\mathrm{MMD^2}\circ\gF\circ\Lambda):=\frac{1}{N}\mathbb{E}_{\sigma\sim\{\pm 1\}^N}\Big[\sup_{\G\in\gF}\sum^N_{k=1}\sigma_k(\mathrm{MMD}^2(\gG(\hat{\rho}^{k}_0,\hat{\rho}^{k}_1)_\# \Tilde{\rho}^k_0, \Tilde{\rho}^k_1))\Big].
\end{align*}
The quadratic kernel $k(x,y)=\langle x,y\rangle^2$ considered under our parametric setup is Lipschitz with Lipschitz constant $L_k=2M_\f R\max\{M_\f,R\}=O(R^2)$ for sufficiently large $R$. Characteristic kernels are bounded, for which the Lipschitz constant $L_k$ does not depend on $R$. 
We revoke Lemma \ref{lem:mmd_ub}, Talagrand's contraction Lemma, and Dudley's chaining Theorem,
the empirical Rademacher complexity is bounded by 
\begin{align}\label{ineq:dudley}
    \mathrm{Rad}(\mathrm{MMD^2}\circ\gF\circ\Lambda) \leq 2L_k(1+\frac{1}{n})\mathrm{Rad}(\gF\circ\Lambda)\leq 3L_k\Bigg\{\inf_{\tau>0} 2\tau +\frac{12}{\sqrt{N}} \int^{M_\f}_\tau \sqrt{\log \gN(\delta,\f,\|\cdot\|_{L^\infty})} \mathrm{d}\delta\Bigg\}.
\end{align}
\noindent Here $\gN(\delta,\f,\|\cdot\|_{L^\infty})$ is the covering number of the function class $\f$. Under Assumption \ref{ass:bdd&lips}, 
we estimate this covering number restricted to the bounded domain, i.e. $\gF:\widehat{\gM}_s\times B(R) \to B(M_\f)$,
where $\widehat{\gM}_s$ denotes the manifold containing all possible empirical measures of prompts given any pair of probability measures $(\rho_0,\rho_1)$ in the $d_{\gM}$-dimensional task manifold $\gM$, each with $s$ i.i.d. samples, i.e.
     $$\widehat{\gM}_s:=\Big\{\Big(\frac{1}{s}\sum^s_{i=1}\delta_{x^i},\frac{1}{s}\sum^s_{i=1}\delta_{y^i}\Big),~\text{where}~~x_i\stackrel{iid}{\sim}\rho_0, y_i\stackrel{iid}{\sim}\rho_1,\forall(\rho_0,\rho_1)\in\gM\Big\}.$$
Following the idea in \citep{gottlieb2013efficient}, the covering number of $\gF$ is estimated in terms of the covering numbers of the domain.
\begin{lem}[Metric entropy for $\gF$]\label{lem:coverF}
     Let $\epsilon>0$. With probability at least $1-\eta$, the metric entropy of the function class $\f :\widehat{\gM}_s\times B(R) \to B(M_\f)$  satisfies 
     \begin{align}\label{metric_entropy_F}
         \log\gN(\epsilon,\f,\|\cdot\|_\infty)\lesssim d\Big(1+\frac{8L_\gF R}{\epsilon}\Big)^d \Big(\frac{1}{\epsilon-cs^{-\frac{1}{2}}\log \frac{1}{\eta}}\Big)^{d_{\gM}} \log\Big(\frac{M_\f}{\epsilon}\Big),
     \end{align}
\end{lem}
\begin{proof} [Lemma \ref{lem:coverF}]
%\yl{Is this referencing correct? }
To build up a relation between the covering number of $\widehat{\gM}$ and $\gM$, we leverage the Wasserstein distance between the empirical measure and the continuous counterpart. 
Assumption \ref{ass:moment} guarantees sub-Gaussian tails for any measures $\rho\in \gP_2(\R^d)$ restricted to the bounded domain $B(R)$. Let $\frac{1}{s}\sum^s_{i=1}\delta_{x^i}$ be an empirical measure of $\rho$.  We apply the concentration inequality on bounded domain in \citep{Guillin15} Proposition 10 rescaling to $(-R,R]^d$ and $p=2$ 
\begin{align}\label{concen:subG}
    \mathbb{P}\Big(W_2\Big(\frac{1}{s}\sum^s_{i=1}\delta_{x^i},\rho\Big)> t \Big)\leq C\begin{cases}
        \exp{(-cst)}, &d\neq 4\\
        \exp{\big(-cs\frac{t}{\log^2(c/t^{1/2})}\big)}, &d=4.
    \end{cases}
\end{align}
where $C,c$ are generic constants depending on $d$ but independent of $s,R$. Wlog, we assume that $\log^2(c/t^{1/2})\geq 1$. Let $\epsilon>0$. Fix a covering for $\gM$ with covering number $\gN(\epsilon,\gM,W_2)$. Here we abuse the notation that $W_2$ represents $W_2(\cdot,\cdot)+W_2(\cdot,\cdot)$. 
By a tube argument with $\epsilon>cs^{-1}\log\frac{C}{\eta}$, with 
probability at least $1-\eta$, the covering number for $\widehat{\gM}$ is bounded by
\begin{align*}
    \gN(\epsilon,\widehat{\gM},W_2)\leq \gN(\epsilon-cs^{-1}\log\frac{C}{\eta},\gM,W_2).
\end{align*}
A similar argument as in \citep{gottlieb2013efficient} Lemma 5.2 indicates that $|\hat{\f}|\leq \Big(\frac{M_\f}{\epsilon}\Big)^{d\gN_{B(R)}\gN(\epsilon,\widehat{\gM},W_2)}$. Accordingly,
\begin{align*}
    \log \gN(\epsilon,\f,\|\cdot\|_\infty) &\leq d\Big(1+\frac{8L_\gF R}{\epsilon}\Big)^d \gN(\epsilon-cs^{-1}\log\frac{1}{\eta},\gM,W_2)\log\Big(\frac{M_\f}{\epsilon}\Big)\\
    &\lesssim d\Big(1+\frac{8L_\gF R}{\epsilon}\Big)^d \Big(\frac{1}{\epsilon-cs^{-1}\log\frac{1}{\eta}}\Big)^{d_{\gM}} \log\Big(\frac{M_\f}{\epsilon}\Big).
\end{align*}
\end{proof}
Apply Lemma \ref{lem:coverF} to \eqref{ineq:dudley}, with 
probability $1-\eta$, the Rademacher complexity $\mathrm{Rad}(\gF\circ\Lambda)$ is bounded by
\begin{align}\label{rad_bound}
    &\mathrm{Rad}(\gF\circ\Lambda)\\
    \leq&\inf_{\tau>0} 2\tau +\frac{12}{\sqrt{N}} \int^{M_\f}_\tau \sqrt{d\Big(1+\frac{8L_\gF R}{\delta}\Big)^d \Big(\frac{1}{\delta-cs^{-1}\log \frac{1}{\eta}}\Big)^{d_{\gM}} \log\Big(\frac{M_\f}{\delta}\Big)} \mathrm{d}\delta\nonumber\\
     \lesssim& \inf_{\tau>0} 2\tau +\frac{12}{\sqrt{N}} \int^{M_\f}_\tau \sqrt{d\Big(\frac{9L_\gF R}{\delta}\Big)^{d+1}}\sqrt{\Big(\frac{1}{\delta-cs^{-1}\log \frac{1}{\eta}}\Big)^{d_{\gM}}} \mathrm{d}\delta\nonumber\\
     \lesssim& \inf_{\tau>0} 2\tau +\frac{6}{\sqrt{N}} \int^{M_\f}_\tau d\Big(\frac{9L_\gF R}{\delta}\Big)^{d+1} \mathrm{d}\delta + \frac{6}{\sqrt{N}} \int^{M_\f}_\tau \Big(\frac{1}{\delta-cs^{-1}\log \frac{1}{\eta}}\Big)^{d_{\gM}} \mathrm{d}\delta \nonumber\\
     =& \inf_{\tau>0} 2\tau -\frac{6(9L_\gF R)^{d+1}}{\sqrt{N}}\delta^{-d}\Big|^{M_\f}_\tau - \frac{6}{(d_{\gM}-1)\sqrt{N}}\Big(\delta-cs^{-1}\log \frac{1}{\eta}\Big)^{-(d_{\gM}-1)}\Big|^{M_\f}_\tau\nonumber\\
    \leq & O\Big( N^{-\frac{1}{4\max\{d,d_{\gM}\}}} + N^{-\frac{1}{4}}R^{d+1} + N^{-\frac{1}{4}}\Big),
\end{align}
where in the last inequality we take $\tau = N^{-\frac{1}{4\max\{d,d_{\gM}\}}}$,  and $O(\cdot)$ hides the dependency on $d_{\gM},d, L_\f$ and $M_\f$. We choose $\eta=(\frac{s\tau}{2c})^{-1/2} = O(s^{-\frac{1}{2}}N^{\frac{1}{8\max\{d,d_{\gM}\}}})$ such that $\tau-cs^{-1}\log \frac{1}{\eta}>\frac{1}{2}\tau$. Now we bound the first term in \eqref{barR2}. Combining \eqref{ineq:sym_rad} and \eqref{ineq:dudley}, the expectation over the training samples $\Lambda$ is split into the event $\gA$ such that \eqref{rad_bound} holds with probability at least $1-\eta$, and we control $\gA^c$ by noticing that
$|\mathrm{Rad}(\gF\circ\Lambda)|\leq M_\f$. Thus
\begin{align}\label{ineq:mmdbound1}
    \mathbb{E}_\Lambda \sup_{\G\in\mathcal{F}}\big|\mathcal{R}^2_t(\G) -\Bar{\mathcal{R}}^2_{\Lambda,t}(\G)\big|
    \leq & 6\lambda L_k\Big(O\big( N^{-\frac{1}{4\max\{d,d_{\gM}\}}} + N^{-\frac{1}{4}}R^{d+1} + N^{-\frac{1}{4}}\big)\cdot(1-\eta)+M_\f\cdot\eta \Big)\nonumber\\
    \leq& O(N^{-\frac{1}{4\max\{d,d_{\gM}\}}}R^2+N^{-\frac{1}{4}}R^{d+3}+N^{\frac{1}{8\max\{d,d_{\gM}\}}}s^{-\frac{1}{2}}R^2).
\end{align}

The second term in \eqref{barR2} describes the difference between biased and unbiased empirical estimates. By definition \eqref{def_mmdu}  and Lemma \ref{lem:mmd_ub}, 
\begin{align}\label{ineq:mmdbound2}
    &\mathbb{E}_\Lambda \sup_{\G\in\mathcal{F}}\big|\Bar{\mathcal{R}}^2_{\Lambda,t}(\G) -\mathcal{R}^2_{\Lambda,t}(\G)\big|\nonumber\\
    =&\mathbb{E}_\Lambda \sup_{\G\in\mathcal{F}}\Bigg\{ \frac{1}{n^2(n-1)}\sum^n_{i=1}\sum^n_{j\neq i}[k(x_i,x_j)+k(y_i,y_j)] + \frac{1}{n^2}\sum^n_{i=1}[k(x_i,x_i)+k(y_i,y_i)]\Bigg\} = O(n^{-1}R^2).
\end{align}
Combining \eqref{ineq:mmdbound1} and \eqref{ineq:mmdbound2},  \eqref{barR2} is bounded by
\begin{align}\label{ineq:barR2}
    \mathbb{E}_\Lambda\sup_{\G\in\f}\big|\mathcal{R}^2_t(\G) -\mathcal{R}^2_{\Lambda,t}(\G)\big|
    \leq O\Big(N^{-\frac{1}{4\max\{d,d_{\gM}\}}}R^2+N^{-\frac{1}{4}}R^{d+3}+N^{\frac{1}{8\max\{d,d_{\gM}\}}}s^{-\frac{1}{2}}R^2+ n^{-1}R^2 \Big).
\end{align}

Meanwhile, for characteristic kernels, \eqref{ineq:mmdbound1} becomes
\begin{align}\label{ineq:mmdbound1_char}
    \mathbb{E}_\Lambda \sup_{\G\in\mathcal{F}}\big|\mathcal{R}^2_t(\G) -\Bar{\mathcal{R}}^2_{\Lambda,t}(\G)\big| \leq O(N^{-\frac{1}{4\max\{d,d_{\gM}\}}}+N^{-\frac{1}{4}}R^{d+1}+N^{\frac{1}{8\max\{d,d_{\gM}\}}}s^{-\frac{1}{2}}).
\end{align}
And the second term in \eqref{barR2} becomes 
\begin{align}\label{ineq:mmdbound2_char}
    \mathbb{E}_\Lambda \sup_{\G\in\mathcal{F}}\big|\Bar{\mathcal{R}}^2_{\Lambda,t}(\G) -\mathcal{R}^2_{\Lambda,t}(\G)\big| = O(n^{-1}).
\end{align}
Due to Assumption \ref{ass:kernel}, \eqref{ineq:mmdbound2_char} has no dependency on $R$.

We next bound the truncation error in \eqref{mmd_error}. Based the definition \eqref{def_mmd} of $\mathrm{MMD}$ and the bound in Lemma \ref{lem:mmd_ub}, 
\begin{align}\label{ineq:pushforward}
    &\sup_{\G\in\f}\big|\mathcal{R}^2(\G) - \mathcal{R}^2_t(\G)\big| \nonumber\\
    =& \lambda\sup_{\G\in\f} \Big|\mathrm{MMD}^2(\gG(\hat{\rho}_0,\hat{\rho}_1)_{\#}\rho_0, \rho_1) - \mathrm{MMD}^2(\gG(\hat{\rho}_0,\hat{\rho}_1)_{\#}\rho_0,\rho_1\mathds{1}_{y\in B(R)})\Big|\nonumber\\
    \leq& \sup_{\G\in\f} \lambda \Big(\mathrm{MMD}(\gG(\hat{\rho}_0,\hat{\rho}_1)_{\#}\rho_0, \rho_1)+\mathrm{MMD}(\gG(\hat{\rho}_0,\hat{\rho}_1)_{\#}\rho_0,\rho_1\mathds{1}_{y\in B(R)})\Big)\Big|\Big|   \nonumber\\
    =& \sup_{\G\in\f} \lambda \Big| \Big[\sup_{\|f\|_\gH\leq 1}\big(\mathbb{E}_{x\sim\G(\hat{\rho}_0,\hat{\rho}_1)_{\#}\rho_0}[f(x)]-\mathbb{E}_{y\sim\rho_1}[f(y)] \big)\Big]^2\nonumber\\
    &- \Big[\sup_{\|f\|_\gH\leq 1}\big( \mathbb{E}_{x\sim\G(\hat{\rho}_0,\hat{\rho}_1)_{\#}\rho_0}[f(x)]-\mathbb{E}_{y\sim\rho_1}[f(y)\mathds{1}_{y\in B(R)}]  \big)\Big]^2 \Big|\nonumber\\
    \leq& 8\lambda M^2_\f R^2 \big|\sup_{\|f\|_\gH\leq 1} \big(\mathbb{E}_{y\sim\rho_1}[f(y)]-\mathbb{E}_{y\sim\rho_1}[f(y)\mathds{1}_{y\in B(R)}]\big)\big|
    \leq \frac{16\lambda\widetilde{M}M^2_\f R^2}{R^3} = \frac{16\lambda\widetilde{M}M^2_\f}{R}.
\end{align}
Here we used 
the Markov's inequality \eqref{ineq:markov} in the second inequality, and the last inequality is due the fact that $\|f\|_{\gH}\leq 1$. In the last equality, we used the quadratic kernel satisfies $k(x,y)\leq M^2_\f R^2$. 
We choose $R=\widetilde{M}n^{\frac{1}{3}}$ in the estimates \eqref{ineq:barR2} and \eqref{ineq:pushforward}, we conclude that 
\begin{align*}
    \mathbb{E}_\Lambda \sup_{\G\in\f}|\mathcal{R}^2(\G) -\mathcal{R}^2_\Lambda(\G)| \leq O\Big(N^{-\frac{1}{4\max\{d,d_{\gM}\}}}n^{\frac{2}{3}}+N^{-\frac{1}{4}}n^{\frac{d+3}{3}}+N^{\frac{1}{8\max\{d,d_{\gM}\}}}s^{-\frac{1}{2}}n^{\frac{2}{3}} + n^{-\frac{1}{3}}\Big).
\end{align*}
Here $O(\cdot)$ hides the constants depending on $\lambda, d_{\gM},d, L_\f$ and $M_\f$.

For characteristic kernels, \eqref{ineq:pushforward} becomes 
\begin{align}\label{ineq:pushforward_char}
    \sup_{\G\in\f}\big|\mathcal{R}^2(\G) - \mathcal{R}^2_t(\G)\big| \leq O(R^{-3}).
\end{align}
We choose $R=\widetilde{M}n^{\frac{1}{3}}$ and combine \eqref{ineq:mmdbound1_char},\eqref{ineq:mmdbound2_char} and \eqref{ineq:pushforward_char} to obtain the bound for characteristic kernels:
\begin{align*}
    \mathbb{E}_\Lambda \sup_{\G\in\f}|\mathcal{R}^2(\G) -\mathcal{R}^2_\Lambda(\G)| \leq O\Big(N^{-\frac{1}{4\max\{d,d_{\gM}\}}}+N^{-\frac{1}{4}}n^{\frac{d+1}{3}}+N^{\frac{1}{8\max\{d,d_{\gM}\}}}s^{-\frac{1}{2}} + n^{-1}\Big).
\end{align*}
\end{proof}

\section{Proofs for Section \ref{sec: parametric}}\label{app: parametricproofs}
\subsection{Proof of Theorem \ref{thm: lossgap}}\label{sec: excesslosspf}
Let us recall the notation used in to prove Theorem \ref{thm: lossgap}. We define the individual loss function $\ell$ by
\begin{align*}
    \ell(\theta, y_1, \dots, y_{2n},) = \E_{x \sim \mathcal{N}(0,I)}[\|A_{n,\theta}x-x\|^2] + \lambda \|A_{n,\theta}^2 - \Sigma_n\|_F^2,
\end{align*}
so that the population risk can be expressed as
\begin{align*}
    \mathcal{R}(\theta) = \E_{\Sigma \sim \mu, y_1, \dots, y_{2n} \sim \mathcal{N}(0,\Sigma)} [\ell(\theta, y_1, \dots, y_{2n})].
\end{align*}
For $t > 0$ and random variables $y_1, \dots, y_{2n} \in \R^d$, consider the event
\begin{align*}
    \mathcal{A}_t &= \bigcap_{i=1}^{n} \left \{\|y_i\| \leq \sqrt{d} \sigma_{\max} + t, \; \|\Sigma_n\|_{\op} \leq \sigma_{\max}^2 \left(1 + t + \sqrt{\frac{d}{n}} \right) \right\},
\end{align*}
where we recall $\Sigma_n = \frac{1}{n} \sum_{k=n+1}^{2n} y_k y_k^T.$ Define the $t$-truncated individual loss function 
\begin{align*}
    \ell_t(\theta, y_1, \dots, y_{2n}) = \ell(\theta, y_1, \dots, y_{2n}) \cdot \mathbf{1}\left(y_1, \dots, y_{2n} \in \mathcal{A}_t \right).
\end{align*}
Let $\mathcal{R}_{t}$ and $\mathcal{R}_{t,\Lambda}$ denote the population and empirical risk functionals with $\ell$ replaced by $\ell_t$. We restate the error decomposition, proven in Section \ref{sec: pfsketch}:

\begin{lem}[Error decomposition]
    Let $\widehat{\theta} \in \textrm{arg} \min_{\theta \in \Theta} \mathcal{R}_{\Lambda}(\theta).$ Then, for any $\theta^{\ast} \in \Theta$, we have
    \begin{align*}
        \mathcal{R}(\widehat{\theta}) - \mathcal{R}(\mathcal{G}^{\dagger}) &\leq \sup_{\theta \in \Theta} \left(\mathcal{R}(\theta) - \mathcal{R}_{t}(\theta) \right) + 2 \sup_{\theta \in \Theta} \left|\mathcal{R}_{t}(\theta) - \mathcal{R}_{t,\Lambda}(\theta) \right| + \left(\mathcal{R}(\theta^{\ast}) - \mathcal{R}(\mathcal{G}^{\dagger}) \right),
    \end{align*}
    with probability $\geq 1 - \delta_{N,n}$, where $$\delta_{N,n} = 2N\left(n\exp \left(-\frac{t^2}{C \sigma_{\max}^2} \right) + \exp \left(- \frac{nt^2}{2} \right) \right),$$ and $C > 0$ is a universal, dimension-independent constant.
\end{lem}
Our strategy is to bound each source of error individually, beginning with the truncation error.

\paragraph{Truncation error bound:}
Due to the fast tail decay of the Gaussian data, the truncation error can be quickly controlled by an application of the Cauchy-Schwarz inequality.
\begin{lem}\label{lem: truncerrorbound}
    For any $\theta \in \Theta$, we have
    \begin{align*}
        \mathcal{R}(\theta) - \mathcal{R}_{t}(\theta) &\lesssim \left(d(1+\sigma_{\max}^4) + d^2(1+\lambda) \left(1 + (C_{\Theta} L_{\Phi})^{3/2} \left(\sigma_{\max}^{7/4} + d^{3/4} \sigma_{\max} \right) \right) \right) \\ &\cdot  \left( n^{1/2} \exp \left(-\frac{t^2}{2C \sigma_{\max}^2} \right) + \exp \left( -\frac{nt^2}{4} \right) \right).
    \end{align*}
\end{lem}
\begin{proof}
    We have
    \begin{align*}
        \mathcal{R}(\theta) - \mathcal{R}_{t}(\theta) &= \E_{\Sigma \sim \mu, (y_1, \dots, y_{2n}) \sim \mathcal{N}(0,\Sigma)} \left[ \left(\E_{x \sim \mathcal{N}(0,I)} \|A_{n,\theta}x-x\|^2 + \lambda \left\|A_{n,\theta}^2 - \Sigma \right\|_F^2 \right) \cdot \mathbf{1}(\mathcal{A}_t) \right] \\
        &\leq \E_{\Sigma \sim \mu, (y_1, \dots, y_{2n}) \sim \mathcal{N}(0,\Sigma)} \left[ \left(\E_{x \sim \mathcal{N}(0,I)} \|A_{n,\theta}x-x\|^2 + \lambda \left\|A_{n,\theta}^2 - \Sigma \right\|_F^2 \right)^2 \right]^{1/2} \cdot \sqrt{\mathbb{P}(\mathcal{A}_t)}.
    \end{align*}
    To bound the first factor, we use the simple bound
    \begin{align*}
        &\E \left[\left(\E_{x \sim \mathcal{N}(0,I)} \|A_{n,\theta}x-x\|^2 + \lambda \left\|A_{n,\theta}^2 - \Sigma \right\|_F^2 \right)^2  \right] \\
        &= \E \left[ \left(\trace((A_{n,\theta}-I)^2) + \lambda \left\|A_{n,\theta}^2 - \Sigma \right\|_F^2 \right) \right] \\
        &\leq \E \left[\left(\|A_{n,\theta}-I\|_F^2 + \lambda \left\|A_{n,\theta}^2 - \Sigma \right\|_F^2 \right)^2 \right] \\
        &\leq \E \left[\left(2 \|A_{n,\theta}\|_F^2 + 2d + 2\lambda \left(\|A_{n,\theta}\|_F^4 + \|\Sigma\|_F^2 \right) \right)^2 \right] \\
        &\leq 8 \E \left[ \|A_{n,\theta}\|^4 + d^2 + \lambda^2 \left(\|A_{n,\theta}\|_F^8 + \|\Sigma\|_F^4 \right) \right].
    \end{align*}
    Then $\E[\|\Sigma\|_F^4] \leq d^2 \sigma_{\max}^8$ by Assumption \ref{assum: taskdistr}, and the moments of $\|A_{n,\theta}\|^2$ can be bounded by Lemma \ref{lem: tfmomentbd}: in particular, we have
    \begin{align*}
        \E[\|A_{n,\theta}\|_F^4] &\leq d^2 \E[\|A_{n,\theta}\|_{\op}^4] \\
        &\lesssim d^2\left(1 + C_{\Theta} L_{\Phi} \left(\sigma_{\max}^{3/2} + d^{1/2} \sigma_{\max} \right) \right),
    \end{align*}
    and similarly
    \begin{align*}
        \E[\|A_{n,\theta}\|_F^8] \lesssim d^4 \left(1 + (C_{\Theta} L_{\Phi})^3 \left( \sigma_{\max}^{7/2} + d^{3/2} \sigma_{\max}^2 \right) \right).
    \end{align*}
    This proves the bound
    \begin{align*}
        & \E_{\Sigma \sim \mu, (y_1, \dots, y_{2n}) \sim \mathcal{N}(0,\Sigma)} \left[ \left(\E_{x \sim \mathcal{N}(0,I)} \|A_{n,\theta}x-x\|^2 + \lambda \left\|A_{n,\theta}^2 - \Sigma \right\|_F^2 \right)^2 \right]^{1/2} \\
        &\lesssim d(1+\sigma_{\max}^4) + d^2(1+\lambda) \left(1 + (C_{\Theta} L_{\Phi})^{3/2} \left(\sigma_{\max}^{7/4} + d^{3/4} \sigma_{\max} \right) \right).
    \end{align*}
    As discussed in Lemma \ref{lem: errordecomp}, the event $\mathcal{A}_t$ occurs with probability at most $n \exp \left(-\frac{t^2}{C \sigma_{\max}^2} + \exp \left( -\frac{nt^2}{2} \right) \right),$ where $C > 0$ is a universal constant. This gives the final bound
    \begin{align*}
        \mathcal{R}(\theta) - \mathcal{R}_{t}(\theta) &\lesssim d(1+\sigma_{\max}^4) + d^2(1+\lambda) \left(1 + (C_{\Theta} L_{\Phi})^{3/2} \left(\sigma_{\max}^{7/4} + d^{3/4} \sigma_{\max} \right) \right) \\ &\cdot n^{1/2} \exp \left(-\frac{t^2}{2C \sigma_{\max}^2} + \exp \left( -\frac{nt^2}{4} \right) \right),
    \end{align*}
    where $\theta \in \Theta$ is arbitrary.
\end{proof}
Lemma \ref{lem: truncerrorbound} shows that the truncation error is quite mild, since the failure probability of the event $\mathcal{A}_t$ can be made exponentially small by tuning the parameter $t$, and the prefactor is only polynomial in all relevant parameters.

\paragraph{Statistical error bound:} We use standard techniques to control the supremum of the empirical process $ \sup_{\theta \in \Theta} \left|\mathcal{R}_{t}(\theta) - \mathcal{R}_{t,\Lambda}(\theta) \right|$. Specifically, we first show establish a concentration inequality to show that the supremum is close to its mean with high probability; we then use the covering number of the class $\Theta$ to bound the mean of the supremum.

We begin by establishing a concentration inequality of supremum around its mean.

\begin{lem}\label{lem: mcdiarmid}
    For any $t > 0$, $N \in \mathbb{N}$, and $s > 0$, with $$C(t,\lambda,\Theta) := 2\left(C_{\Theta}^4 L_{\Phi}^4 (\sqrt{d}\sigma_{\max}+t)^4 + d + \lambda \left(C_{\Theta}^8 L_{\Phi}^8 (\sqrt{d}\sigma_{\max}+t)^8 + d \left(\sigma_{\max}^2 + t + \sqrt{\frac{d}{n}} \right)^2 \right) \right),$$ it holds that
    \begin{align*}
        \mathbb{P} \left(  \sup_{\theta \in \Theta} \left|\mathcal{R}_{t}(\theta) - \mathcal{R}_{t,\Lambda}(\theta) \right| \geq \E \sup_{\theta \in \Theta} \left|\mathcal{R}_{t}(\theta) - \mathcal{R}_{t,\Lambda}(\theta) \right| + s \right) \leq \exp \left( - \frac{Ns^2}{C^2(t,\lambda,\Theta)} \right),
    \end{align*}
    where the probability and expectation are over the $N$ training tasks.
\end{lem}
\begin{proof}
    This is a standard application of McDiarmid's inequality. The random variable
    $$ \sup_{\theta \in \Theta} \left|\mathcal{R}_{t}(\theta) - \mathcal{R}_{t,\Lambda}(\theta) \right|
    $$
    can be viewed as a function of the $N$ samples $\mathbf{y}^{(1)}, \dots, \mathbf{y}^{(N)}$, where $\mathbf{y}^{(i)} = (y_1^{(i)}, \dots, y_{2n}^{(i)}).$ Denote this function $\mathcal{E}(\mathbf{Y}_N)$, where $\mathbf{Y}_N = (\mathbf{y}^{(1)}, \dots, \mathbf{y}^{(N)}).$ Suppose $\mathbf{y}' = (y_1', \dots, y_{2n}')$ is another collection of $2n$ samples belonging to the event $\mathcal{A}_t$, and, for $j \in [N]$, let $\mathbf{Y}_N^{/j}$ denote the $N$-tuple obtained by replacing $\mathbf{y}^{(j)}$ by $\mathbf{y}'$ in the definition of $\mathbf{Y}_N$. Then it can be easily seen that
    \begin{align*}
        \left|\mathcal{E}(\mathbf{Y}_N) - \mathcal{E}(\mathbf{Y}_N^{/j}) \right| &\leq \frac{2}{N} \sup_{\theta, \mathbf{y}} \ell(\theta, \mathbf{y}),
    \end{align*}
    Lemma \ref{lem: lossbound1} shows that the individual loss satisfies the bound
    $$  \sup_{\theta, \mathbf{y}} \ell_t(\theta, \mathbf{y}) \leq C(t,\lambda,\Theta),
    $$
    where $C(t,\lambda,\Theta)$ is defined in the statement of the lemma. It follows that the random variable $\mathcal{E}(\mathbf{Y}_N)$ satisfies a bounded differences inequality with constant $\frac{2C(t,\lambda,\Theta)}{N}.$ By McDiarmid's inequality, we deduce the result.
\end{proof}
Lemma \ref{lem: mcdiarmid} implies that to control the statistical error, it suffices to control the expected supremum
$$ \E \sup_{\theta \in \Theta} \left|\mathcal{R}_{t}(\theta) - \mathcal{R}_{t,\Lambda}(\theta) \right|.
$$
This term can be bounded using systematic techniques from empirical process theory. Before presenting and proving the bound, let us fix the following notation. Define the function class $L_t(\Theta) = \left\{ \ell_t(\theta, \cdot): \theta \in \Theta \right\}$, and define $D_t(\Theta) = \sup_{\theta \in \Theta} \|\ell_t(\theta,\cdot)\|_{L^{\infty}})$. Finally, given a function class $\mathcal{F}$, a metric $\rho$ on $\mathcal{F}$, and $\tau > 0$, let $N(\tau, \mathcal{F}, \rho)$ denote the covering number of $\mathcal{F}$ with respect to the metric $\rho.$

\begin{lem}\label{lem: finalcovnumberbound}
    Given $t > 0$, define the constants $R_1(t) = \sqrt{d}\sigma_{\max}+t$, $R_2(t) = \sigma_{\max}^2 \left( 1 + t + \sqrt{\frac{d}{n}} \right),$
    $$ K_1(t,\lambda,\Theta) = 4 (C_{\Theta}^4 L_{\Phi}^4 R_1^4(t) + C_{\Theta}^2 L_{\Phi}^2 R_1(t)^2) + 8\lambda \left(C_{\Theta}^8 L_{\Phi}^8 R_1^8(t) + C_{\Theta}^4 L_{\Phi}^4 R_1^4(t) R_2(t) \right),
    $$
    and
    $$ D_t(\Theta) = 2 \left(C_{\Theta}^4 L_{\Phi}^4 R_1^4(t) + d + \lambda \left(C_{\Theta}^8 L_{\Phi}^8 R_1^8(t) + d R_2^8(t) \right) \right).
    $$
    Then we have the bound We have
    \begin{align*}
        &\sup_{\theta \in \Theta} \left|\mathcal{R}_{t}(\theta) - \mathcal{R}_{t,\Lambda}(\theta) \right| \\ 
        &\lesssim \frac{1}{\sqrt{N}} \left( D_t(\Theta) \cdot d \cdot \left( \sqrt{\log \left( 1 + 4 \sqrt{N} K_1(t,\lambda,\Theta)\right)}+ \sqrt{\log \left(1 + 4\sqrt{N} \sqrt{d} M R_1(t)\right) }\right) + M \log \left(\sqrt{N}(D_t(\Theta) \right) \right).
    \end{align*}
\end{lem}

\begin{proof}
    Recall that the individual loss $\ell_t$ is a function of $\theta \in \Theta$ and a tuple $(y_1, \dots, y_{2n})$ of samples, the latter of which we abbreviate to $\mathbf{y}.$ By Dudley's chaining bound, we have
    \begin{align*}
        \E \sup_{\theta \in \Theta} \left|\mathcal{R}_{t}(\theta) - \mathcal{R}_{t,\Lambda}(\theta) \right| &= \E \sup_{\theta \in \Theta} \left| \frac{1}{N} \sum_{i=1}^{n} \left(\E[\ell_t(\theta, \mathbf{y})] - \ell_t(\theta, \mathbf{y}^{(i)}) \right) \right| \\
        &\leq \inf_{\tau > 0} \tau + \int_{\tau}^{D_t(\Theta)} \sqrt{\log N(\epsilon, L_t(\Theta), L^{\infty})} d\epsilon.
    \end{align*}
    Lemma \ref{lem: coveringnumloss} establishes the inequality
    $$  |\ell_t(\theta_1, \cdot) - \ell_t(\theta_2, \cdot)\|_{L^{\infty}} \leq K_1(t,\lambda,\Theta) \left(\|Q_1-Q_2\|_F + \|\phi_1-\phi_2\|_{L^{\infty}(B(R_1(t)))} \right),
    $$
    where 
    $$ K_1(t,\lambda,\Theta) = 4 (C_{\Theta}^4 L_{\Phi}^4 R_1^4(t) + C_{\Theta}^2 L_{\Phi}^2 R_1(t)^2) + 8\lambda \left(C_{\Theta}^8 L_{\Phi}^8 R_1^8(t) + C_{\Theta}^4 L_{\Phi}^4 R_1^4(t) R_2(t) \right).
    $$
    This implies that if $\{Q_i\}_i$ is a $\frac{\epsilon}{2 K_1(t, \lambda, \Theta)}$-cover of the $C_{\Theta}$-ball of $\R^{d \times d}$ in the Frobenius norm, and $\{\phi_j\}_j$ is a $\frac{\epsilon}{2 K_1(t, \lambda, \Theta)}$-cover of $\Phi$ in the $L^{\infty}(B(R_1(t)))$-norm, then the set of functions 
    $$ \left\{\ell_t(\theta, \cdot): \theta = (Q,\phi), Q \in \{Q_i\}_i, \; \phi \in \{\phi_j\}_j \right\}
    $$
    is an $\epsilon$ cover of $L_t(\Theta)$ in the $L^{\infty}$-norm. In other words, we have the inequality of covering numbers 
    \begin{align*}
        N(\epsilon, L_t(\Theta), L^{\infty}) &\leq N\left( \frac{\epsilon}{2 K_1(t, \lambda, \Theta)}, \left \{\|Q\|_F \leq C_{\Theta} \right \}, \| \cdot \|_F\right) \cdot N \left( \frac{\epsilon}{2 K_1(t, \lambda, \Theta)}, \Phi, L^{\infty}(B(R_1(t))) \right).
    \end{align*}
    By Example 5.8 in \citep{wainwright2019high}, the covering number of the $C_{\Theta}$-ball in $\R^{d \times d}$ satisfies the bound
    $$ \log N\left( \frac{\epsilon}{2 M(t, \lambda, \Theta)}, \left \{\|Q\|_F \leq C_{\Theta} \right \}, \| \cdot \|_F\right) \leq d^2 \log \left( 1 + \frac{4 K_1(t,\lambda,\Theta)}{\epsilon}\right),
    $$
    while Lemma \ref{lem: nncoveringnum} demonstrates that the covering number of the neural network class $\Phi = \{x \mapsto \psi(Wx): \|W\|_{\op} \leq C_{\Theta}, \psi \in \nn(M)\}$
    satisfies the bound
    $$  \log N \left(\epsilon, \phi(M), L^{\infty} B(R_1(t)) \right) \lesssim d^2 \log \left(1 + \frac{4 \sqrt{d} M R_1(t)}{\epsilon} \right) + \frac{M^2}{\epsilon^2}.
    $$
    Thus, for any $\tau > 0$, we have
    \begin{align*}
        &\int_{\tau}^{D_t(\Theta)} \sqrt{\log N(\epsilon, L_t(\Theta), L^{\infty})} d\epsilon \\ &\lesssim \int_{\tau}^{D_t(\Theta)}\left( d \left(\sqrt{\log \left( 1 + \frac{4 K_1(t,\lambda,\Theta)}{\epsilon}\right)}+ \sqrt{\log \left(1 + \frac{4 \sqrt{d} M R_1(t)}{\epsilon} \right) }\right) + \frac{M}{\epsilon} \right)d\epsilon \\
        &\leq D_t(\Theta) \cdot d \cdot \left( \sqrt{\log \left( 1 + \frac{4 K_1(t,\lambda,\Theta)}{\tau}\right)}+ \sqrt{\log \left(1 + \frac{4 \sqrt{d} M R_1(t)}{\tau} \right) }\right) + \int_{\tau}^{D_t(\Theta)} \frac{M}{\epsilon}d\epsilon \\
        &= D_t(\Theta) \cdot d \cdot \left( \sqrt{\log \left( 1 + \frac{4 K_1(t,\lambda,\Theta)}{\tau}\right)}+ \sqrt{\log \left(1 + \frac{4 \sqrt{d} M R_1(t)}{\tau} \right) }\right) + M \log \left(\frac{D_t(\Theta)}{\tau} \right).
    \end{align*}
    Finally, Lemma \ref{lem: lossbound1} shows that the $L^{\infty}$-diameter of $L_t(\Theta)$ satisfies the bound
    $$ D_t(\Theta) \leq 2 \left(C_{\Theta}^4 L_{\Phi}^4 R_1^4(t) + d + \lambda \left(C_{\Theta}^8 L_{\Phi}^8 R_1^8(t) + d R_2^8(t) \right) \right).
    $$
    Taking $\tau = \frac{1}{\sqrt{N}}$, we deduce the desired bound.
\end{proof}
Up to logarithmic factors, Lemmas \ref{lem: mcdiarmid} and \ref{lem: finalcovnumberbound} combine to establish a bound on the generalization error of $O \left( \frac{D_t(\Theta) d  + M}{\sqrt{N}}\right).$

\paragraph{Approximation error bound:}
We now seek to bound the approximation error
$$ \left(\mathcal{R}(\theta^{\ast}) - \mathcal{R}(\mathcal{G}^{\dagger}) \right)
$$
for a suitable transformer $\theta^{\ast} \in \Theta$. We first state the following reduction lemma when the penalty term $\lambda$ is large.

\begin{lem}\label{lem: reductionlem}
    When $\lambda > 0$ is sufficiently large, we have
    \begin{align*}
        \mathcal{R}(\theta^{\ast}) - \mathcal{R}(\mathcal{G}^{\dagger}) &\leq \left|\mathcal{R}(\theta^{\ast}) - \E_{\Sigma \sim \mu}\left[W_2^2\left(\mathcal{N}(0,I), \mathcal{N}(0,\Sigma) \right) \right] \right| +  \frac{\lambda}{n} \sup_{\Sigma \in \textrm{sup}(\mu)} \left[2 \|\Sigma\|_F^2 + \sum_{i,j \in [d], \; i \neq j} \sigma_i^2 \sigma_j^2 \right] \\ &+ \frac{d^2 C_{\mu,1}}{\lambda} \left( 2(1+\sigma_{\max}) + \frac{C_{\mu,1}}{\lambda} \right) + \frac{2 d C_{\mu,1}(C_{\mu,1} + \sigma_{\max}^2)}{\lambda},
    \end{align*}
    where $C_{\mu,1}$ is a constant depending on $\textrm{supp}(\mu).$
\end{lem}
\begin{proof}
    Write
    \begin{align*}
          \mathcal{R}(\theta^{\ast}) - \mathcal{R}(\mathcal{G}^{\dagger}) &= \left(\mathcal{R}(\theta^{\ast}) - \E_{\Sigma \sim \mu}\left[W_2\left(\mathcal{N}(0,I), \mathcal{N}(0,\Sigma) \right) \right]\right) +  \left( \E_{\Sigma \sim \mu}\left[W_2\left(\mathcal{N}(0,I), \mathcal{N}(0,\Sigma) \right) \right] - \mathcal{R}(\mathcal{G}^{\dagger}) \right).
    \end{align*}
    Define functions on the set $\R^{d \times d}_{sym}$ of symmetric matrices by $f_{\Sigma,\lambda}(A) = \E_{x \sim \mathcal{N}(0,I)}\|Ax-x\|^2 + \lambda \|A^2-\Sigma\|_F^2$ and $$f_{\Sigma,\lambda,n}(A) = \E_{y_1, \dots, y_n \sim \mathcal{N}(0,\Sigma)} \left[\E_{x \sim \mathcal{N}(0,I)}\|Ax-x\|^2 + \lambda \|A^2-\Sigma_n\|_F^2 \right],$$ where $\Sigma_n = \frac{1}{n} \sum_{i=1}^{n} y_iy_i^T,$ $y_i \sim \mathcal{N}(0,\Sigma)$. Let $A_{\Sigma,\lambda} \in \textrm{argmin}_{A \in \R^{d \times d}_{sym}} f_{\Sigma,\lambda}$ and $A_{\Sigma,\lambda,n} \in \textrm{argmin}_{A \in \R^{d \times d}_{sym}} f_{\Sigma,\lambda,n}.$ Then we have
    \begin{align*}
        &\E_{\Sigma \sim \mu}\left[W_2\left(\mathcal{N}(0,I), \mathcal{N}(0,\Sigma) \right) \right] - \mathcal{R}(\mathcal{G}^{\dagger}) = \left( \E_{\Sigma \sim \mu}\left[W_2\left(\mathcal{N}(0,I), \mathcal{N}(0,\Sigma) \right) \right] - \E_{\Sigma \sim \mu} \left[f_{\Sigma,\lambda}(A_{\Sigma,\lambda}) \right] \right) \\
        &+ \left(\E_{\Sigma \sim \mu} \left[f_{\Sigma,\lambda}(A_{\Sigma,\lambda}) - f_{\Sigma,\lambda}(A_{\Sigma,\lambda,n}) \right] \right) + \left(\E_{\Sigma \sim \mu, (y_1, \dots, y_n) \sim \mathcal{N}(0,\Sigma)}\left[(f_{\Sigma,\lambda}-f_{\Sigma,\lambda,n})(A_{\Sigma,\lambda,n}) \right] \right).
    \end{align*}
    Notice that the second term is strictly nonpositive by the optimality of $A_{\Sigma,\lambda}$, thus
    \begin{align*}
        \E_{\Sigma \sim \mu}\left[W_2\left(\mathcal{N}(0,I), \mathcal{N}(0,\Sigma) \right) \right] &- \mathcal{R}(\mathcal{G}^{\dagger}) \leq \left| \E_{\Sigma \sim \mu}\left[W_2\left(\mathcal{N}(0,I), \mathcal{N}(0,\Sigma) \right) \right] - \E_{\Sigma \sim \mu} \left[f_{\Sigma,\lambda}(A_{\Sigma,\lambda}) \right]  \right| \\ &+ \left|\E_{\Sigma \sim \mu, (y_1, \dots, y_n) \sim \mathcal{N}(0,\Sigma)}\left[(f_{\Sigma,\lambda}-f_{\Sigma,\lambda,n})(A_{\Sigma,\lambda,n}) \right] \right|.
    \end{align*}
    By Proposition \ref{prop: largelambdalimit}, item 2, we have
    \begin{align*}
        \left| \E_{\Sigma \sim \mu}\left[W_2\left(\mathcal{N}(0,I), \mathcal{N}(0,\Sigma) \right) \right] - \E_{\Sigma \sim \mu} \left[f_{\Sigma,\lambda}(A_{\Sigma,\lambda}) \right]  \right| &\leq \frac{d^2 C_{\mu,1}}{\lambda} \left( 2(1+\sigma_{\max}) + \frac{C_{\mu,1}}{\lambda} \right) + \frac{2 d C_{\mu,1}(C_{\mu,1} + \sigma_{\max}^2)}{\lambda},
    \end{align*}
    where the constant $C_{\mu,1}$ is defined in Proposition \ref{prop: largelambdalimit}. For the last term, we have
    \begin{align*}
        &\left|\E_{\Sigma \sim \mu, (y_1, \dots, y_n) \sim \mathcal{N}(0,\Sigma)}\left[(f_{\Sigma,\lambda}-f_{\Sigma,\lambda,n})(A_{\Sigma,\lambda,n}) \right] \right| \\ &= \lambda\left|\E_{\Sigma \sim \mu, \; (y_1, \dots, y_n) \sim \mathcal{N}(0,\Sigma)} \left[\|A_{\Sigma,\lambda,n}^2- \Sigma\|_f^2 - \|A_{\Sigma,\lambda,n}^2 - \Sigma_n\|_F^2 \right] \right| \\
        &= \lambda\left|\E_{\Sigma \sim \mu, \; (y_1, \dots, y_n) \sim \mathcal{N}(0,\Sigma)} \left[ \|\Sigma\|_F^2 - \|\Sigma_n\|_F^2 \right] \right| \\
        &= \frac{\lambda}{n}\E_{\Sigma \sim \mu} \left[2 \|\Sigma\|_F^2 + \sum_{i,j \in [d], \; i \neq j} \sigma_i^2 \sigma_j^2 \right],
    \end{align*}
    where the last line follows from a simple computation involving expectations of empirical covariance matrices of Gaussians. We conclude the proof
\end{proof}
Lemma \ref{lem: reductionlem} demonstrates that when $\lambda$ is large, it suffices to bound the error between $\mathcal{R}(\theta^{\ast})$ and the expected Wasserstein distance between $\mathcal{N}(0,I)$ and $\mathcal{N}(0,\Sigma).$ This essentially amounts to constructing a transformer which in-context approximates the optimal transport map between $\mathcal{N}(0,I)$ and $\mathcal{N}(0,\Sigma).$ Recall that the class $\Phi$ of feature mappings is defined as the set of one-dimensional shallow ReLU networks (extended componentwise to $\R^d$) with an additional inner linear layer:
$$ \Phi = \Phi(M) = \left\{x \mapsto \phi(x) = \psi(Wx): \|W\|_{\op} \leq C_{\Theta}, \; \psi \in \nn(M) \right\}.$$ 
The precise values of $M$ and $C_{\Theta}$ will be chosen later so that the various approximation and statistical errors are balanced. We recall that $L_{\Phi}$ is a uniform bound on the Lipschitz constant of $\phi \in \Phi$. Notice that for a one-dimensional function $\psi \in NN(m,M)$, the Lipschitz constant of $\psi$ is naturally bounded by $M$. It follows that the Lipschitz constant $L_{\Phi}$ of the function class $\Phi$ is bounded by $\sqrt{d} M C_{\Theta}.$ Our first lemma shows that the transformer class $\Theta$ can express the Gaussian-to-Gaussian optimal transport map.

\begin{lem}\label{lem: approx1}
    For any $\epsilon > 0$, there exists $M = O(d^{5/8} \epsilon^{-1/2})$ and $\theta^{\ast} \in \Theta(M)$ such that for any $t > 0$, $\Sigma \in \textrm{supp}(\mu)$, and $y_1, \dots, y_n \sim \mathcal{N}(0,\Sigma),$
    $$ \|A_{n,\theta^{\ast}} - \Sigma^{1/2}\|_{\op} \lesssim \sqrt{1+t} \epsilon + \sigma_{\max}^2 \sqrt{\frac{d+t}{n}}, \; \textrm{w.p. $\geq 1 - 2 \left( nd e^{-\frac{t^2}{2\sigma_{\max}^2}} + e^{-t} \right).$}
    $$
\end{lem}
\begin{proof}
    Throughout this proof, when $A$ is a matrix, $\|A\|$ denotes the operator norm of $A$, even if we do not explicitly label the norm. Recall that $\theta \in \Theta$ is of the form $\theta = (Q,\phi)$, where $Q \in \R^{d \times d}$ satisfies $\|Q\|_F \leq C_{\Theta}$ and $\phi(x) = \psi(Wx)$, where $\|W\|_{\textrm{op}} \leq C_{\Theta}$ and $\psi$ belongs to the one-dimensional neural network function space $\nn(M)$. We choose $\theta^{\ast} = (Q^{\ast},\phi^{\ast})$ as follows. Take $Q^{\ast} = \left( \frac{\pi}{2} \right)^{1/4}U$, where $U$ is the orthogonal matrix defined in Assumption \ref{assum: taskdistr} which diagonalizes the covariance $\Sigma$. To define $\phi^{\ast}$, choose $W^{\ast} = U^T$ for the inner linear layer. For the nonlinearity $\psi^{\ast}$, we define the extended root function $g_{sq}: \R \rightarrow \R$
    $$ g_{sq}(z) = \begin{cases}
        \sqrt{z}, \; z \geq 0 \\
        -\sqrt{-z}, \; z < 0. 
    \end{cases}
    $$
    We then choose $\psi^{\ast}$ to be any element of $\nn(M)$ which satisfies
    $$ \sup_{|z| \leq t} |\psi^{\ast}(z) - g_{sq}(z)| \leq \epsilon.
    $$
    By Lemma \ref{lem: sqrtapproximation}, such a $\psi^{\ast}$ exists. It follows that 
    \begin{align*}
        A^{\ast} := A_{n,\theta^{\ast}} = U \cdot \frac{c}{n} \sum_{i=1}^{n} \psi^{\ast}(U^T y_i) \psi^{\ast} (U^T y_i)^T \cdot U^T,
    \end{align*}
    where $c = \sqrt{\frac{\pi}{2}}.$ We also introduce the auxiliary matrices
    \begin{equation}
        \tilde{A} := U \cdot \frac{c}{n} \sum_{i=1}^{n} g_{sq}(U^T y_i) g_{sq} (U^T y_i)^T \cdot U^T
    \end{equation}
    and
    \begin{equation}\label{eqn: lambdaN}
        \Lambda_n^{1/2} := \frac{c}{n} \sum_{i=1}^{n} g_{sq}(U^Ty_i) g_{sq}(U^T y_i)^{T}.
    \end{equation}
    For each $i \in [n]$ and $j \in [d]$, the random variable $(U^T y_i)_{j}$ is a centered Gaussian whose variance is bounded by $\sigma_{\max}^2$. By a union bound and a standard concentration inequality for Gaussian random variables, we have
    \begin{align*}
        \mathbb{P} \left(\max_{1 \leq i \leq n} \|U^T y_i\|_{\infty} > t \right) &= \mathbb{P} \left(\max_{i \in [n], \; j \in [d]} |(U^T y_i)_j| > t   \right) \\
        &\leq 1-2dn e^{-\frac{t^2}{2\sigma_{\max}^2}}.
    \end{align*}
    Thus, for the rest of the proof, we will assume that $\max_{1 \leq i \leq n} \|U^T y_i\|_{\infty} \leq t.$ We decompose the error as
    \begin{align*}
        \left\|A^{\ast} - \Sigma^{1/2} \right\| \leq \|A^{\ast} - \tilde{A}\| + \|\tilde{A} - \Sigma^{1/2}\|.
    \end{align*}
    To bound the error between $A^{\ast}$ and $\tilde{A}$, we use the approximation properties of $\psi^{\ast}$:
    \begin{align*}
        \left\|A^{\ast} - \tilde{A} \right \| &=  \left\|U \cdot \frac{c}{n} \sum_{i=1}^{n} \left(\psi^{\ast}(U^T y_i) \psi^{\ast}(U^T y_i)^T - g_{sq}(U^T y_i) g_{sq}(U^T y_i)^T \right) \cdot U^T \right\| \\
        &= \left\|\frac{c}{n} \sum_{i=1}^{n} \left(\psi^{\ast}(U^T y_i) \psi^{\ast}(U^T y_i)^T - g_{sq}(U^T y_i) g_{sq}(U^T y_i)^T \right) \right\| \\
        &\leq \frac{c}{n} \sum_{i=1}^{n} \left( \left\|(\psi^{\ast}(U^T y_i)-g_{sq}(U^T y_i)) \psi^{\ast}(U^T y_i)^T \right\| + \left\| g_{sq}(U^T y_i) (\psi^{\ast}(U^T y_i) - g_{sq}(U^T y_i))^T \right\| \right) \\
        &\leq \frac{c}{n} \sum_{i=1}^{n} \left(\left\|\psi^{\ast}(U^T y_i) \right\| + \left\|g_{sq}(U^T y_i) \right\| \right) \left\|\psi^{\ast}(U^T y_i) - g_{sq}(U^T y_i) \right\|. \\
    \end{align*}
    To finish the bound, we note that since $\max_i \|U^T y_i\|_{\infty} \leq t$, the assumption on $\psi^{\ast}$ implies that $\left\|\psi^{\ast}(U^T y_i) - g_{sq}(U^T y_i) \right\| \leq \sqrt{d} \epsilon.$ Similarly, by the definition of $g_{sqrt},$ we have $ \left\|g_{sq}(U^T y_i) \right\| = \|U^T y_i\|_1^{1/2} \leq d^{1/2} R^{1/2}$, and similarly, since $\psi^{\ast}$ approximates $g_{sq}$, $\left\|\psi^{\ast}(U^T y_i) \right\| \leq 2 d^{1/2} t^{1/2}$ when $\epsilon \leq t$. Combining these estimates furnishes the bound
    \begin{align*}
        \|A^{\ast}-\tilde{A}\| \leq \frac{3\pi}{2} d t^{1/2} \epsilon.
    \end{align*}
    To bound the error between $\tilde{A}$ and $\Sigma^{1/2}$, let $\Lambda = U^T \Sigma U$ denote the diagonalization of $\Sigma$. Then
    \begin{align*}
        \left\|\tilde{A} - \Sigma^{1/2} \right\| = \|\Lambda_n^{1/2} - \Lambda^{1/2}\|,
    \end{align*}
    where $\Lambda_n$ has been defined in Equation \eqref{eqn: lambdaN}. Note that $\Lambda_n^{1/2}$ is the empirical covariance matrix of the random variables $z_1, \dots, z_n$, where $z = (\pi/2)^{1/4} g_{sq}(U^T y)$, for $y \sim \mathcal{N}(0,\Sigma)$. Lemma \ref{lem: changeofvar} shows that the covariance matrix of $z$ is indeed $\Lambda^{1/2}$, and that $z$ is sub-Gaussian with a constant depending only on $\sigma_{\max}$. By Exercise 4.17 in \citep{vershynin2018high}, the inequality
    $$ \left\| \Lambda_n^{1/2} - \Lambda^{1/2} \right\| \lesssim \sigma_{\max}^2 \sqrt{\frac{d+t}{n}}
    $$
    holds with probability at least $1-2e^{-t}$. By combining the estimates for $\|A^{\ast} - \tilde{A}\|$ and $\|\tilde{A} - \Sigma^{1/2}\|$, applying a union bound for the probabilities, and bounding $\trace(\Sigma) \leq d \sigma_{\max}^2,$ we conclude that
    $$ \|A^{\ast} - \Sigma^{1/2}\| \lesssim d t^{1/2} \epsilon + \sigma_{\max}^2 \sqrt{\frac{d+t}{n}}
    $$
    with the desired probability. Finally, we rescale $\epsilon$ to $\epsilon \cdot  d^{-5/4}$ and conclude that
    $$ \|A^{\ast} - \Sigma^{1/2}\| \lesssim \sqrt{1+t}\epsilon + \sigma_{\max}^2 \sqrt{\frac{d+t}{n}}.
    $$
    This rescaling means that we now need $M = O(d^{5/8} \epsilon^{-1/2})$ rather than $O(\epsilon^{-1/2}).$
\end{proof}
Lemma \ref{lem: approx1} shows that the transformer class $\Theta(M)$ can approximate the Gaussian-to-Gaussian optimal transport map in-context, which is an interesting result on its own. It also demonstrates that the complexity of the transformer class, quantified by the parameter $M$, must be chosen as a function of the number of samples $n$ in order to obtain estimates that only depend on the sample size. From the conclusion of Lemma \ref{lem: approx1}, the scaling $M = O \left(d^{1/2} n^{1/4} \right)$ appears to be the correct scaling to balance the approximation error of the neural network with the statistical error of the empirical covariance matrix; we will discuss more precisely in the proof of Theorem \ref{thm: lossgap}.

The next lemma shows that a bound on the error between $A_{n,\theta}$ and $\Sigma^{1/2}$ implies a bound on the approximation gap of the population risk.

\begin{lem}\label{lem: mainapproxbd}
    Let $\epsilon > 0$ be given, let $M = O( d^{5/8}\epsilon^{-1/2})$ and let $\theta^{\ast} \in \Theta(M)$ denote the transformer parameterization constructed in Lemma \ref{lem: approx1}. Then
    \begin{align*}
        &\mathcal{R}(\theta^{\ast}) - \E_{\Sigma \sim \mu} W_2^2 \left(\mathcal{N}(0,I), \mathcal{N}(0,\Sigma) \right) \lesssim  d (1+\sigma_{\max}) \left((\sqrt{1+t}\epsilon + \sigma_{\max}^2 \sqrt{\frac{d+t}{n}}  \right) \\ &+ \lambda d \sigma_{\max}^2 \left( (1+t)\epsilon^2 + \sigma_{\max}^4 \frac{d+t}{n} \right) \\
        &+ \left( d(1+\sigma_{\max}^4) + d^2(1+\lambda) \left(1 + (C_{\Theta} L_{\Phi})^{3/2} \left(\sigma_{\max}^{7/4} + d^{3/4} \sigma_{\max} \right) \right) \right) \cdot \left( e^{-\frac{t^2}{4\sigma_{\max}^2}}+e^{-t/2} \right).
    \end{align*}
\end{lem}
\begin{proof}
    To abbreviate notation, we will write $W_2^2(\Sigma_1,\Sigma_2)$ to denote the squared Wasserstein distance between $\mathcal{N}(0,\Sigma_1)$ and $\mathcal{N}(0,\Sigma_2).$ Fix $t > 0$ and let $\mathcal{E}_t$ denote the event on which the bound in Lemma \ref{lem: approx1} holds. As in the proof of Lemma \ref{lem: approx1}, we abbreviate notation to $A^{\ast} = A_{n,\theta^{\ast}}$
    \begin{align*}
        \mathcal{R}(\theta^{\ast}) &= \E_{\Sigma \sim \mu, (y_1, \dots, y_{2n}) \sim \mathcal{N}(0,\Sigma)} \left[ \E_{x \sim \mathcal{N}(0,I)}\|A^{\ast}x-x\|^2 + \lambda \|(A^{\ast})^2 - \Sigma\|_F^2 \right] \\
        &= \E_{\Sigma \sim \mu, (y_1, \dots, y_{2n}) \sim \mathcal{N}(0,\Sigma)} \left[ \left( \E_{x \sim \mathcal{N}(0,I)}\|A^{\ast}x-x\|^2 + \lambda \|(A^{\ast})^2 - \Sigma\|_F^2 \right) \cdot \mathbf{1}(\mathcal{E}_t) \right] \\
        &+ \E_{\Sigma \sim \mu, (y_1, \dots, y_{2n}) \sim \mathcal{N}(0,\Sigma)} \left[ \left( \E_{x \sim \mathcal{N}(0,I)}\|A^{\ast}x-x\|^2 + \lambda \|(A^{\ast})^2 - \Sigma\|_F^2 \right) \cdot \mathbf{1}(\mathcal{E}_t^c) \right].
    \end{align*}
    On the event $\mathcal{E}_t$, we have $\|A^{\ast} - \Sigma^{1/2}\| \lesssim \sqrt{1+t}\epsilon + \sigma_{\max}^2 \sqrt{\frac{d+t}{n}}$ for any $\Sigma \in \textrm{supp}(\mu)$. For any symmetric matrix $A$, we have 
    $$ \E_{x \sim \mathcal{N}(0,I)}[\|Ax-x\|^2] = \trace(A^2) + d - 2 \trace(A).
    $$
    Therefore, fixing $\Sigma \in \textrm{supp}(\mu)$ and writing $A^{\ast} = \Sigma^{1/2} + (A^{\ast} - \Sigma^{1/2})$, we have
    \begin{align*}
        \E_{x \sim \mathcal{N}(0,I)}\|A^{\ast}x-x\|^2 &= \E_{x \sim \mathcal{N}(0,I)}\|\Sigma^{1/2}x-x\|^2 + \trace((A^{\ast}-\Sigma^{1/2})^2) + 2 \trace(\Sigma^{1/2}(A^{\ast}-\Sigma^{1/2})) - 2 \trace(A^{\ast}-\Sigma^{1/2}) \\
        &= W_2^2\left(I, \Sigma) \right) + \trace((A^{\ast}-\Sigma^{1/2})^2) + 2 \trace(\Sigma^{1/2}(A^{\ast}-\Sigma^{1/2})) - 2 \trace(A^{\ast}-\Sigma^{1/2}).
    \end{align*}
    On the event $\mathcal{E}_t$, we have $\|A^{\ast}-\Sigma^{1/2}\| \lesssim \sqrt{1+t}\epsilon + \sigma_{\max}^2 \sqrt{\frac{d+t}{n}}.$ Therefore, using the Cauchy-Schwarz inequality along with the inequality $\|A\|_F \leq \sqrt{d} \|A\|$, we find that
    \begin{align*}
        \trace((A^{\ast}-\Sigma^{1/2})^2) &\leq \|A^{\ast}-\Sigma^{1/2}\|_F^2 \\
        &\lesssim d \left((1+t)\epsilon^2 + \sigma_{\max}^4 \frac{d+t}{n} \right),
    \end{align*}
    and similarly
    \begin{align*}
        2 \trace(\Sigma^{1/2}(A^{\ast}-\Sigma^{1/2})) &\leq 2 \|\Sigma^{1/2}\|_F \|A^{\ast}-\Sigma^{1/2}\|_F \\
        &\lesssim 2d \sigma_{\max} \left(\sqrt{1+t}\epsilon + \sigma_{\max}^2 \sqrt{\frac{d+t}{n}} \right)
    \end{align*}
    and
    \begin{align*}
        \trace(A^{\ast}-\Sigma^{1/2}) \lesssim d \left(\sqrt{1+t}\epsilon + \sigma_{\max}^2 \sqrt{\frac{d+t}{n}} \right).
    \end{align*}
    Therefore, for any $\Sigma \in \textrm{supp}(\mu)$, on the event $\mathcal{E}_t$, it holds that 
    \begin{align*}
        \left| \E_{x \sim \mathcal{N}(0,I)} \|A^{\ast}x-x\|^2 - W_2^2\left(I, \Sigma) \right)\right| &\lesssim d\left( (1+t)\epsilon^2 + \sigma_{\max}^4 \frac{d+t}{n} + (1+\sigma_{\max}) \left(\sqrt{1+t}\epsilon + \sigma_{\max}^2 \sqrt{\frac{d+t}{n}} \right) \right) \\
        &\lesssim d(1+\sigma_{\max}) \left( \sqrt{1+t}\epsilon + \sigma_{\max}^2 \sqrt{\frac{d+t}{n}} \right)
    \end{align*}
    when $n$ is sufficiently large and $\epsilon$ is sufficiently small. Similarly, on the event $\mathcal{E}_t,$ for any $\Sigma \in \textrm{supp}(\mu)$, it holds that
    \begin{align*}
        \left \|(A^{\ast})^2 - \Sigma \right\|_F^2 &\leq 2 \left( \left \|A^{\ast}(A^{\ast}-\Sigma^{1/2}) \right \|_F^2 + \left \|(A^{\ast}-\Sigma^{1/2})\Sigma^{1/2} \right\|_F^2 \right) \\
        &\lesssim 2 \left( \left \|A^{\ast} \right\|_F^2 + \left \|\Sigma^{1/2} \right \|_F^2 \right) \cdot \left((1+t)\epsilon^2 + \sigma_{\max}^4 \frac{d+t}{n} \right).
    \end{align*}
    We can further bound $\|\Sigma^{1/2}\|_F^2 \leq d \sigma_{\max}^2$ and
    \begin{align*}
        \|A^{\ast}\|_F^2 &\leq 2\|A^{\ast}-\Sigma^{1/2}\|_F^2 + 2 \|\Sigma^{1/2}\|_F^2 \\
        &\leq 2 \sqrt{d}\|A^{\ast}-\Sigma^{1/2}\|^2 + 2 d \sigma_{\max}^2 \\
        &\lesssim \sqrt{d} \cdot \left((1+t)\epsilon^2 + \sigma_{\max}^4 \frac{d+t}{n} \right) + 2d \sigma_{\max}^2.
    \end{align*}
    
    This proves the bound 
    \begin{align*}
        &\E_{\Sigma \sim \mu, (y_1, \dots, y_{2n}) \sim \mathcal{N}(0,\Sigma)} \left[ \left( \E_{x \sim \mathcal{N}(0,I)}\|A^{\ast}x-x\|^2 + \lambda \|(A^{\ast})^2 - \Sigma\|_F^2 \right) \cdot \mathbf{1}(\mathcal{E}_t) \right] \\
        &= \E_{\Sigma \sim \mu}[W_2^2\left(I, \Sigma) \right)] + O \left(d (1+\sigma_{\max}) \left((\sqrt{1+t}\epsilon + \sigma_{\max}^2 \sqrt{\frac{d+t}{n}}  \right) + \lambda d \sigma_{\max}^2 \left( (1+t)\epsilon^2 + \sigma_{\max}^4 \frac{d+t}{n} \right) \right).
    \end{align*}
    
    To bound the second term, we use the Cauchy-Schwarz inequality:
    \begin{align*}
        &\E_{\Sigma \sim \mu, (y_1, \dots, y_{2n}) \sim \mathcal{N}(0,\Sigma)} \left[ \left( \E_{x \sim \mathcal{N}(0,I)}\|A^{\ast}x-x\|^2 + \lambda \|(A^{\ast})^2 - \Sigma\|_F^2 \right) \cdot \mathbf{1}(\mathcal{E}_t^c) \right] \\
        &\leq \E_{\Sigma \sim \mu, (y_1, \dots, y_{2n}) \sim \mathcal{N}(0,\Sigma)} \left[ \left( \E_{x \sim \mathcal{N}(0,I)}\|A^{\ast}x-x\|^2 + \lambda \|(A^{\ast})^2 - \Sigma\|_F^2 \right)^2 \right]^{1/2} \cdot \sqrt{\mathbb{P}(\mathcal{E}_t^c)}.
    \end{align*}
    As was shown in Lemma \ref{lem: truncerrorbound}, 
    \begin{align*}
         &\E_{\Sigma \sim \mu, (y_1, \dots, y_{2n}) \sim \mathcal{N}(0,\Sigma)} \left[ \left(\E_{x \sim \mathcal{N}(0,I)} \|A_{n,\theta}x-x\|^2 + \lambda \left\|A_{n,\theta}^2 - \Sigma \right\|_F^2 \right)^2 \right]^{1/2} \\
        &\lesssim d(1+\sigma_{\max}^4) + d^2(1+\lambda) \left(1 + (C_{\Theta} L_{\Phi})^{3/2} \left(\sigma_{\max}^{7/4} + d^{3/4} \sigma_{\max} \right) \right),
    \end{align*}
    and by Lemma \ref{lem: approx1}, we have
    \begin{align*}
        \sqrt{\mathbb{P}(\mathcal{E}_t^c)} \leq \sqrt{2} \left( e^{-\frac{t^2}{4\sigma_{\max}^2}}+e^{-t/2} \right).
    \end{align*}
    We have therefore shown that, for any $t > 0$,
    \begin{align*}
        \mathcal{R}(\theta^{\ast}) - \E[W_2^2\left(I, \Sigma) \right)] &\lesssim d \sigma_{\max}^2 \left(\sqrt{1+t}\epsilon + \sigma_{\max}^2 \sqrt{\frac{d+t}{n}} + \lambda \left( (1+t)\epsilon^2 + \sigma_{\max}^4 \frac{d+t}{n} \right) \right) \\
        &+ \left( d(1+\sigma_{\max}^4) + d^2(1+\lambda) \left(1 + (C_{\Theta} L_{\Phi})^{3/2} \left(\sigma_{\max}^{7/4} + d^{3/4} \sigma_{\max} \right) \right) \right) \cdot \left( e^{-\frac{t^2}{4\sigma_{\max}^2}}+e^{-t/2} \right),
    \end{align*}
    hence the result.
\end{proof}
Combining Lemmas \ref{lem: reductionlem} and \ref{lem: mainapproxbd}, we have the following result.
\begin{cor}\label{cor: mainapproximation}
    We have
    \begin{align*}
        \mathcal{R}(\theta^{\ast}) &- \mathcal{R}(\mathcal{G}^{\dagger}) \lesssim d (1+\sigma_{\max}) \left((\sqrt{1+t}\epsilon + \sigma_{\max}^2 \sqrt{\frac{d+t}{n}}  \right) + \lambda d \sigma_{\max}^2 \left( (1+t)\epsilon^2 + \sigma_{\max}^4 \frac{d+t}{n} \right) \\
        &+ \left( d(1+\sigma_{\max}^4) + d^2(1+\lambda) \left(1 + (C_{\Theta} L_{\Phi})^{3/2} \left(\sigma_{\max}^{7/4} + d^{3/4} \sigma_{\max} \right) \right) \right) \cdot \left( e^{-\frac{t^2}{4\sigma_{\max}^2}}+e^{-t/2} \right) \\
        &+ \frac{\lambda}{n} \sup_{\Sigma \in \textrm{sup}(\mu)} \left[2 \|\Sigma\|_F^2 + \sum_{i,j \in [d], \; i \neq j} \sigma_i^2 \sigma_j^2 \right] + \frac{d^2 C_{\mu,1}}{\lambda} \left( 2(1+\sigma_{\max}) + \frac{C_{\mu,1}}{\lambda} \right) + \frac{2 d C_{\mu,1}(C_{\mu,1} + \sigma_{\max}^2)}{\lambda}.
    \end{align*}
\end{cor}

We are now in a position to prove Theorem \ref{thm: lossgap}. We state a precise version below with constants made explicit.

\begin{thm}\label{thm: lossgapprecise}
    We have
    \begin{align*}
        \mathcal{R}(\widehat{\theta}) - \mathcal{R}(\mathcal{G}^{\dagger}) &\lesssim \frac{(1+\lambda)d^{22} \epsilon^{-4} \sigma_{\max}^8 + d^2 \sigma_{\max}^{16}}{\sqrt{N}} + \left(d(1+\sigma_{\max}) \sigma_{\max}^2 \sqrt{\frac{d}{n}} + \lambda d \sigma_{\max}^6 \frac{d}{n} \right) \\ &+ \left(d(1+\sigma_{\max})\epsilon + \lambda d \sigma_{\max}^2 \epsilon^2 \right) + \frac{d^2 C_{\mu,1}}{\lambda} \left( 2(1+\sigma_{\max}) + \frac{C_{\mu,1}}{\lambda} \right) + \frac{2 d C_{\mu,1}(C_{\mu,1} + \sigma_{\max}^2)}{\lambda} + \frac{1}{n}
    \end{align*}
\end{thm}

\begin{proof}[Proof of Theorem \ref{thm: lossgap}]
    By Lemma \ref{lem: errordecomp}, we have, for any $t > 0$ and $\theta^{\ast} \in \Theta$,
    \begin{align*}
        \mathcal{R}(\widehat{\theta}) - \mathcal{R}(\mathcal{G}^{\dagger}) &\leq \sup_{\theta \in \Theta} \left(\mathcal{R}(\theta) - \mathcal{R}_{t}(\theta) \right) + 2 \sup_{\theta \in \Theta} \left|\mathcal{R}_{t}(\theta) - \mathcal{R}_{t,\Lambda}(\theta) \right| \\ &+ \left(\mathcal{R}(\theta^{\ast}) - \mathcal{R}(\mathcal{G}^{\dagger}) \right),
    \end{align*}
    with probability at least
    $$ 1 - 2N\left(n\exp \left(-\frac{t^2}{C \sigma_{\max}^2} \right) + \exp \left(- \frac{nt^2}{2} \right) \right).
    $$
    Lemma \ref{lem: truncerrorbound} bounds the truncation error by
       \begin{align*}
         \sup_{\theta \in \Theta} \mathcal{R}(\theta) - \mathcal{R}_{t}(\theta) \lesssim &\left( d(1+\sigma_{\max}^4) + d^2(1+\lambda) \left(1 + (C_{\Theta} L_{\Phi})^{3/2} \left(\sigma_{\max}^{7/4} + d^{3/4} \sigma_{\max} \right) \right) \right) \\ &\cdot \left( n^{1/2} \exp \left(-\frac{t^2}{2C \sigma_{\max}^2} \right) + \exp \left( -\frac{nt^2}{4} \right) \right).
    \end{align*}
    Lemmas \ref{lem: mcdiarmid} and \ref{lem: finalcovnumberbound} bound the statistical error by
    \begin{align*}
         &\sup_{\theta \in \Theta} \left|\mathcal{R}_{t}(\theta) - \mathcal{R}_{t,\Lambda}(\theta) \right| \\
         &\lesssim \frac{1}{\sqrt{N}} \left( D_t(\Theta) \cdot d \cdot \left( \sqrt{\log \left( 1 + 4 \sqrt{N} K_1(t,\lambda,\Theta)\right)}+ \sqrt{\log \left(1 + 4\sqrt{N} \sqrt{d} M R_1(t)\right) }\right) + M \log \left(\sqrt{N}(D_t(\Theta) \right) \right) + s,
    \end{align*}
    with probability at least $1 - \exp \left( - \Omega(Ns^2) \right)$. Finally, Lemma \ref{lem: mainapproxbd} bounds the approximation error by
    \begin{align*}
        &\mathcal{R}(\theta^{\ast}) - \mathcal{R}(\mathcal{G}^{\dagger}) \lesssim d \sigma_{\max}^2 \left(\sqrt{1+t}\epsilon + \sigma_{\max}^2 \sqrt{\frac{d+t}{n}} + \lambda \left( (1+t)\epsilon^2 + \sigma_{\max}^4 \frac{d+t}{n} \right) \right) \\
        &+ \left( d(1+\sigma_{\max}^4) + d^2(1+\lambda) \left(1 + (C_{\Theta} L_{\Phi})^{3/2} \left(\sigma_{\max}^{7/4} + d^{3/4} \sigma_{\max} \right) \right) \right) \cdot \left( e^{-\frac{t^2}{4\sigma_{\max}^2}}+e^{-t/2} \right) \\
        &+ \frac{2 d C_{\mu,1}(C_{\mu,1} + \sigma_{\max}^2)}{\lambda} + \frac{1}{n},
    \end{align*}
    whenever the parameter $M$ satisfies $O(d^{5/8} \epsilon^{-1/2})$. To finish the proof, we simply choose the variables $s$, $t$, $C_{\Theta}$, and $L_{\Phi}$ to grow as suitable functions of $N$, $n$, and $\epsilon$. First, we choose $C_{\Theta} = \sqrt{d}$. This is sufficient for approximation because we only need the class attention weight matrices to include all orthogonal matrices. Next, we recall that the Lipschitz constant $L_{\Phi}$ of the feature map class $\Phi$ and the norm bound $M$ on the neural network class $\nn(M)$ are related by $L_{\Phi} \leq \sqrt{d} C_{\Theta} M$. With $M = O(d^{5/8} \epsilon^{-1/2})$ and $C_{\Theta} = \sqrt{d}$, this bounds the Lipschitz constant $L_{\Phi} = O \left(d^{13/8} \epsilon^{-1/2} \right).$ Thus it only remains to choose $t$ and $s$. To this end, choose
    $$ t = \Omega \left(\sigma_{\max} (\log(n) + p \log(N)) \right)
    $$
    for some $p > 1$ and $s = \Omega\left(\frac{\log(N)}{N} \right)$. Then, up to factors which are logarithmic in $N$, $n$, and the other problem parameters, the truncation error is bounded by
    \begin{align*}
        \sup_{\theta \in \Theta} \mathcal{R}(\theta) - \mathcal{R}_{t}(\theta) &= O \left(\Bigg( d(1+\sigma_{\max}^4) + d^2(1+\lambda) \left(1 + (C_{\Theta} L_{\Phi})^{3/2} \left(\sigma_{\max}^{7/4} + d^{3/4} \sigma_{\max} \right) \right) \Bigg) \cdot \frac{1}{N^{p/2}} \right) \\
        &= O \left( \left(d(1+\sigma_{\max}^4 + d^7 \epsilon^{-3/4}(1+\lambda) \left(\sigma_{\max}^{7/4} + d^{3/4} \sigma_{\max} \right) \right) \cdot \frac{1}{N^{p/2}} \right) \\
        %&= O \left( \left( d(1+\sigma_{\max}^4) + d^5(1+\lambda) n^{1/8} (\sigma_{\max}^{1/4}+d^{3/4} \sigma_{\max}^{-1/2}) \right) \cdot \frac{1}{N^{p/2}} \right)
    \end{align*}
    the statistical error is bounded by 
    \begin{align*} \sup_{\theta \in \Theta} \left|\mathcal{R}_{t}(\theta) - \mathcal{R}_{t,\Lambda}(\theta) \right| &= O \left((d \cdot D_t(\Theta) + M) \cdot \frac{1}{\sqrt{N}} \right) \\
    &= O \left( \frac{(1+\lambda)d^{22} \epsilon^{-4} \sigma_{\max}^8 + d^2 \sigma_{\max}^{16}}{\sqrt{N}} \right)
    %&= O \left( (1+\lambda) \cdot \left(d^{24} \sigma_{\max}^{16} \cdot n^2 \right) \cdot \frac{1}{\sqrt{N}} \right),
    \end{align*}
    and the approximation error is bounded by
    \begin{align*}
        &\mathcal{R}(\theta^{\ast}) - \mathcal{R}(\mathcal{G}^{\dagger})  
        = O \Bigg( d (1+\sigma_{\max}) \left((\epsilon + \sigma_{\max}^2 \sqrt{\frac{d}{n}} \right) + \lambda d \sigma_{\max}^2 \left( \epsilon^2 + \sigma_{\max}^4 \frac{d}{n} \right) \\ &+ \left( d(1+\sigma_{\max}^4) + d^2(1+\lambda) \left(1 + d^7 \epsilon^{-3/4} \left(\sigma_{\max}^{7/4} + d^{3/4} \sigma_{\max} \right) \right) \right) \cdot \frac{1}{N^{p/2}} \\ &+ \frac{2 d C_{\mu,1}(C_{\mu,1} + \sigma_{\max}^2)}{\lambda} + \frac{\lambda}{n}  \Bigg).
        %&= O \left(\frac{d^{3/2}\sigma_{\max}^4}{\sqrt{n}} + \frac{\lambda d^2 \sigma_{\max}^6}{n} + \left( d(1+\sigma_{\max}^4) + d^5(1+\lambda) n^{1/8} (\sigma_{\max}^{1/4}+d^{3/4} \sigma_{\max}^{-1/2}) \right) \cdot \frac{1}{N^{p/2}} \right).
    \end{align*}
    Noting that the $O(N^{-p/2})$ terms are lower order than the statistical error and combining everything, the final bound on the excess loss, up to leading order, is
    \begin{align*}
        \mathcal{R}(\widehat{\theta}) &- \mathcal{R}(\mathcal{G}^{\dagger}) \lesssim \frac{(1+\lambda)d^{22} \epsilon^{-4} \sigma_{\max}^8 + d^2 \sigma_{\max}^{16}}{\sqrt{N}} \\ &+ \left(d(1+\sigma_{\max}) \sigma_{\max}^2 \sqrt{\frac{d}{n}} + \lambda d \sigma_{\max}^6 \frac{d}{n} \right) + \left(d(1+\sigma_{\max})\epsilon + \lambda d \sigma_{\max}^2 \epsilon^2 \right) + \frac{2 d C_{\mu,1}(C_{\mu,1} + \sigma_{\max}^2)}{\lambda} + \frac{\lambda}{n},
        %O \left( \frac{(1+\lambda)(d^{24} \sigma_{\max}^{16} n^2)}{\sqrt{N}} +\frac{d^{3/2} \sigma_{\max}^4}{\sqrt{n}} + \frac{\lambda d^2 \sigma_{\max}^6}{n} \right),
    \end{align*}
    with probability at least $1 - O \left(\frac{1}{N^p} \right)$.
\end{proof}
The bounds in Theorems \ref{thm: lossgap} and \ref{thm: lossgapprecise} are equivalent, with the bound in Theorem \ref{thm: lossgap} ignoring constants depending polynomially on problem parameters; additionally, the bound in Theorem \ref{thm: lossgap} expresses the approximation error as a function of the path norm $M$, rather than the tolerance $\epsilon$. The bound in expectation follows from integrating the tail.

\subsection{Transport map estimation bounds}\label{sec: transportmapfs}
Recall that, given $\Sigma \in \textrm{supp}(\mu)$ and $y_1, \dots, y_n \sim \mathcal{N}(0,\Sigma)$, $\widehat{A}_n$ denotes the matrix $A_{n,\widehat{\theta}}$. We now seek to bound the transport map estimation error $\E_{\Sigma \sim \mu}\left\| A_{n,\theta} - \Sigma^{1/2} \right\|_F^2.$

In bounding the excess loss, we assumed $\lambda$ to be a fixed constant with respect to the sample size. In what follows, we will allow $\lambda$ to grow with the sample size, and we therefore use the notation $\mathcal{R}_{\mu,\lambda}(\cdot)$ for the population risk to highlight the dependence on $\lambda.$

For $\lambda > 0$ and $\Sigma \in \textrm{supp}(\mu)$, define the functional $f_{\Sigma,\lambda}: \R_{sym}^{d \times d} \rightarrow \R_+$ by
$$ f_{\Sigma,\lambda}(A) = \E_{x \sim \mathcal{N}(0,I)}\|Ax-x\|^2 + \lambda \|A^2 - \Sigma\|_F^2.
$$
We will make use of two key properties of $f_{\Sigma,\lambda}.$ First, for any $\Sigma \in \textrm{supp}(\mu)$, the minimizers of $f_{\Sigma,\lambda}$ converge to $\Sigma^{1/2}$ as $\lambda \rightarrow \infty$, and the minimum converges to $W_2^2(\mathcal{N}(0,I),\mathcal{N}(0,\Sigma))$. Second, $f_{\Sigma,\lambda}$ is strongly convex around its global minimizer, with a constant of $O(\lambda).$ We state these results precisely below.

\begin{prop}\label{prop: largelambdalimit}
   Let $K_{\mu}, \epsilon_{\mu}$ denote the constants defined in Assumption \ref{assum: continuityofeigenprojections}, and define additional constants
   \begin{align*}
       C_{\mu,1} := \sup_{\Sigma \in \textrm{supp}(\mu), i \in [d]} \frac{(1-\sigma_i)^2}{\sigma_i^4}, \; C_{\mu,2} := \sup_{\Sigma \in \textrm{supp}(\mu)} W_2^2(\mathcal{N}(0,I), \mathcal{N}(0,\Sigma)).
   \end{align*}
   Then there exists a $\lambda_{\mu} > 0$ (depending only on $\textrm{supp}(\mu)$) such that for all $\lambda > \lambda_{\mu}$, the following hold.
    \begin{enumerate}
        \item The minimizers of $f_{\Sigma,\lambda}$ converge to $\Sigma^{1/2}$: for any matrix $A_{\lambda} \in \textrm{arg} \min_{A} f_{\Sigma,\lambda}$, we have $\|A_{\lambda}-\Sigma^{1/2}\|_F^2 \leq \frac{C_{\mu,1} d}{\lambda^2}$.
        \item The optimal value of $f_{\Sigma,\lambda}$ converges to the optimal transport cost: \begin{align*}\left| \min_{A} f_{\Sigma,\lambda}(A) - W_2^2(\mathcal{N}(0,I),\mathcal{N}(0,\Sigma)) \right| \leq \frac{d^2 C_{\mu,1}}{\lambda} \left( 2(1+\sigma_{\max}) + \frac{C_{\mu,1}}{\lambda} \right) + \frac{2 d C_{\mu,1}(C_{\mu,1} + \sigma_{\max}^2)}{\lambda}. \end{align*}
        \item For any matrix $A$ such that $f_{\Sigma,\lambda} - \min f_{\Sigma,\lambda} \leq \lambda \epsilon_{\mu} - C_{\mu,2}$, we have
        \begin{align*}
           \|A - \Sigma^{1/2}\|_F &\leq \sqrt{\frac{1}{1 + 2 \lambda \sigma_{\min}^2} \left(f_{\Sigma,\lambda}(A) - \min f_{\Sigma,\lambda} \right) + \frac{2d C_{\mu,1}}{\lambda^2}} + 2\|A\|_{\op} \cdot \sqrt{\frac{K_{\mu}(1+C_{\mu,2})}{\lambda}}
        \end{align*}
    \end{enumerate}
\end{prop}
The following result relates the excess loss to the function $f_{\Sigma,\lambda}$, which in turn allows us to translate estimates of the excess loss to estimates of the transport map error.

\begin{prop}
    For any $\theta \in \Theta$, we have
    \begin{align*}
         &\E_{\Sigma \sim \mu, (y_1, \dots, y_n) \sim \mathcal{N}(0,\Sigma)} \left[f_{\Sigma,\lambda}(A_{n,\theta}) - \min_{A_{\Sigma}} f_{\Sigma,\lambda}(A_{\Sigma}) \right] \leq \left(\mathcal{R}_{\mu,\lambda}(\theta) - \mathcal{R}(\mathcal{G}^{\dagger})_{\mu,\lambda} \right) \\
         &+ O \left( \frac{\lambda}{n}\sup_{\Sigma \in \textrm{supp}(\mu)} \left[2\|\Sigma\|_F^2 + \frac{1}{n} \sum_{i,j \in [d], i \neq j} \sigma_i^2 \sigma_j^2 \right] + \frac{d^2 C_{\mu,1}}{\lambda} \left( 2(1+\sigma_{\max}) + \frac{C_{\mu,1}}{\lambda} \right) + \frac{2 d C_{\mu,1}(C_{\mu,1} + \sigma_{\max}^2)}{\lambda} \right).
    \end{align*}
\end{prop}
\begin{proof}\label{prop: stabilitylemma}
    The proof is similar to the proof of Proposition \ref{lem: reductionlem} We use the decomposition
    \begin{align*}
        \E_{\Sigma \sim \mu, (y_1, \dots, y_n) \sim \mathcal{N}(0,\Sigma)} &\left[f_{\Sigma,\lambda}(A_{n,\theta}) - \min_A f_{\Sigma,\lambda}(A) \right] =  \left( \E_{\Sigma \sim \mu, (y_1, \dots, y_n) \sim \mathcal{N}(0,\Sigma)} \left[f_{\Sigma,\lambda}(A_{n,\theta}) \right] - \mathcal{R}_{\mu,\lambda}(\theta) \right) \\
        &+ \left( \mathcal{R}_{\mu,\lambda}(\theta) - \mathcal{R}(\mathcal{G}^{\dagger})_{\mu,\lambda} \right) + \left(\mathcal{R}(\mathcal{G}^{\dagger})_{\mu,\lambda} -  \E_{\Sigma \sim \mu, (y_1, \dots, y_n) \sim \mathcal{N}(0,\Sigma)} \left[ \min_{A_{\Sigma}}f_{\Sigma,\lambda}(A_{\Sigma}) \right] \right).
    \end{align*}
    For the first term, we note that the only difference between $\mathcal{R}_{\mu,\lambda}(\theta)$ and $\E [f_{\Sigma,\lambda}(A_{n,\theta})]$ is that the latter replaces $\Sigma_n$ with $\Sigma$ in the regularization term; in particular, with $\Sigma_n = \frac{1}{n} \sum_{i=n+1}^{2n} y_i y_i^T$, we have
   \begin{align*}
        &\left(\E_{\Sigma \sim \mu, (y_1, \dots, y_n) \sim \mathcal{N}(0,\Sigma)}[f_{\Sigma,\lambda}(A_{n,\theta})] - \mathcal{R}_{\Sigma,\lambda}(\theta) \right) = \lambda \cdot \E_{\Sigma \sim \mu, (y_1, \dots, y_{2n}) \sim \mathcal{N}(0,\Sigma)} \left[\left\|A_{n,\theta} - \Sigma \right\|_F^2 - \left\|A_{n,\theta}-\Sigma_n \right\|_F^2 \right] \\
        &= \lambda \cdot \E_{\Sigma \sim \mu, (y_1, \dots, y_{2n}) \sim \mathcal{N}(0,\Sigma)} \left[ \|\Sigma_n\|_F^2 - \|\Sigma\|_F^2 + \left(\trace(A_{n,\theta}(\Sigma_n-\Sigma)) \right)  \right] \\
        &= \lambda \cdot \E_{\Sigma \sim \mu, (y_{1}, \dots, y_{n}) \sim \mathcal{N}(0,\Sigma)} \left[ \|\Sigma_n\|_F^2 - \|\Sigma\|_F^2 \right],
    \end{align*}
    where the last equality uses the fact that $\E[\Sigma_n|\Sigma] = \Sigma$ and that $A_{n,\theta}$ and $\Sigma_n$ are independent. If $\Sigma$ has eigenvalues $\sigma_1^2, \dots, \sigma_d^2$, then a quick calculation shows that
    \begin{align*}
         \E_{(y_{1}, \dots, y_{n}) \sim \mathcal{N}(0,\Sigma)} \left[ \|\Sigma\|_F^2 - \|\Sigma_n\|_F^2 \right] &= \frac{2}{n} \|\Sigma\|_F^2 + \frac{1}{n} \sum_{i,j \in [d], i \neq j} \sigma_i^2 \sigma_j^2.
    \end{align*}
    It follows that
    \begin{align*}
         &\left|\E_{\Sigma \sim \mu, (y_1, \dots, y_n) \sim \mathcal{N}(0,\Sigma)}[f_{\Sigma,\lambda}(A_{n,\theta})] - \mathcal{R}_{\Sigma,\lambda}(\theta) \right| \leq \frac{\lambda}{n}\sup_{\Sigma \in \textrm{supp}(\mu)} \left[2\|\Sigma\|_F^2 + \frac{1}{n} \sum_{i,j \in [d], i \neq j} \sigma_i^2 \sigma_j^2 \right].
    \end{align*}
    For the term $\left(\mathcal{R}(\mathcal{G}^{\dagger})_{\mu,\lambda} -  \E_{\Sigma \sim \mu, (y_1, \dots, y_n) \sim \mathcal{N}(0,\Sigma)} \left[ \min_{A_{\Sigma}} f_{\Sigma,\lambda}(A_{\Sigma}) \right] \right)$, we have
    \begin{align*}
        &\left|\mathcal{R}(\mathcal{G}^{\dagger})_{\mu,\lambda} - \E_{\Sigma \sim \mu, (y_1, \dots, y_n) \sim \mathcal{N}(0,\Sigma)} \left[ \min_{A_{\Sigma}}f_{\Sigma,\lambda}(A_{\Sigma}) \right] \right| \leq \left|\mathcal{R}(\mathcal{G}^{\dagger})_{\mu,\lambda} - \E_{\Sigma \sim \mu} \left[W_2^2(\mathcal{N}(0,I),\mathcal{N}(0,\Sigma)) \right] \right| \\
        &+ \left| \E_{\Sigma \sim \mu} \left[W_2^2(\mathcal{N}(0,I),\mathcal{N}(0,\Sigma)) - \min_{A_{\Sigma}}f_{\Sigma,\lambda}(A_{\Sigma}) \right] \right| \\ 
        &=O \left( \frac{\lambda}{n}\sup_{\Sigma \in \textrm{supp}(\mu)} \left[2\|\Sigma\|_F^2 + \frac{1}{n} \sum_{i,j \in [d], i \neq j} \sigma_i^2 \sigma_j^2 \right] + \frac{d^2 C_{\mu,1}}{\lambda} \left( 2(1+\sigma_{\max}) + \frac{C_{\mu,1}}{\lambda} \right) + \frac{2 d C_{\mu,1}(C_{\mu,1} + \sigma_{\max}^2)}{\lambda} \right),
    \end{align*}
    where the last line used Lemma \ref{lem: reductionlem} and Proposition \ref{prop: largelambdalimit}. We conclude the proof.
\end{proof}

Before proving Theorem \ref{thm: TPGE}, we are left with a technicality to resolve. Namely, Proposition \ref{prop: largelambdalimit} states that, if $\lambda$ is sufficiently large and $A \in \R^{d \times d}_{sym}$ is such that $f_{\Sigma,\lambda}(A) - \min f_{\Sigma,\lambda} \leq \lambda \epsilon_{\mu} - C_{\mu,2}$, the error $\|A-\Sigma^{1/2}\|_F$ can be bounded in terms of $f_{\Sigma,\lambda}(A) - \min f_{\Sigma,\lambda}$. We intend to apply this result to the \textit{random} matrix $A_{n,\theta}$, where $\theta \in \Theta$, and in this setting, we cannot expect the bound ib $f_{\Sigma,\lambda}(A_{n,\theta}) - \min f_{\Sigma,\lambda}$ to hold pointwise. However, if, we can show that the random variable $f_{\Sigma,\lambda}(A_{n,\theta}) - \min f_{\Sigma,\lambda}$, conditioned on the value of $\Sigma$, concentrates around its mean, then we can restrict the expectation to the event of sufficiently high probability on which the desired pointwise estimate holds, and the error we incur for truncating the expectation will be marginal. Lemma \ref{lem: weibullconcentration2} in Appendix \ref{sec: auxlemmas} establishes the relevant concentration inequality, while following result demonstrates how the concentration of the random variable $f_{\Sigma,\lambda}(A_{n,\theta}) - \min f_{\Sigma,\lambda}$ translates into an estimate on the transport map generalization error.

\begin{prop}\label{prop: concentration1}
    For any $\theta \in \Theta$, we have the estimate
    \begin{align*}
        \E &\left[ \|A_{n,\theta} - \Sigma^{1/2}\|_F^2 \right] \lesssim \frac{1}{1+2 \lambda \sigma_{\min}^2} \Bigg( \left| \mathcal{R}_{\mu,\lambda}(\theta) - \mathcal{R}(\mathcal{G}^{\dagger})_{\mu,\lambda} \right|  +  \frac{\lambda}{n} \left(2 + \E_{\Sigma} \sum_{i,j \in [d],  \neq j} \sigma_i^2 \sigma_j^2 \right) \\
        &+ \frac{d^2 C_{\mu,1}}{\lambda} \left( 2(1+\sigma_{\max}) + \frac{C_{\mu,1}}{\lambda} \right) + \frac{2 d C_{\mu,1}(C_{\mu,1} + \sigma_{\max}^2)}{\lambda} \Bigg) + \frac{d C_{\mu,1}}{\lambda^2} + \frac{K_{\mu}(1+C_{\mu,2})}{\lambda} \E[\|A_{n,\theta}\|_F^2] \\ &+ \inf_{q \in (1,\infty)} \E[\|A_{n,\theta}\|_F^{2q}]^{1/q} \cdot \mathbb{P} \left(f_{\Sigma,\lambda}(A_{n,\theta}) - \min f_{\Sigma,\lambda} > \lambda \epsilon_{\mu} - C_{\mu,2} \right)^{\frac{q-1}{q}},
    \end{align*}
    where the expectations and probability are taken over $(\Sigma, y_1, \dots, y_n)$ with $\Sigma \sim \mu$ and $(y_i|\Sigma) \sim \mathcal{N}(0,\Sigma).$
\end{prop}
\begin{proof}
    Given $\theta \in \Theta$, $\Sigma \in \textrm{supp}(\mu)$, and $y_1, \dots, y_n \sim \mathcal{N}(0,\Sigma),$ let $\mathcal{S}_{n,\theta}$ denote the event on which 
    $$ f_{\Sigma,\lambda}(A_{n,\theta}) - \min f_{\Sigma,\lambda} \leq \lambda \epsilon_{\mu} - C_{\mu}.
    $$
    Then we have 
    \begin{align*}
        \E\left[ \left\|A_{n,\theta}-\Sigma^{1/2} \right\|_F^2 \right] &= \E\left[ \left\|A_{n,\theta}-\Sigma^{1/2} \right\|_F^2 \left( \mathbf{1}(\mathcal{S}_{n,\theta}) + \mathbf{1}(\mathcal{S}^c_{n,\theta} \right) \right] \\
        &\leq \E\left[ \left\|A_{n,\theta}-\Sigma^{1/2} \right\|_F^2 \cdot \mathbf{1}(\mathcal{S}_{n,\theta}) \right] + \inf_{q \in (1,\infty)} \E\left[ \left\|A_{n,\theta}-\Sigma^{1/2} \right\|_F^{2q} \right]^{1/q} \cdot \mathbb{P}(\mathcal{S}^c_{n,\theta})^{\frac{q-1}{q}}.
    \end{align*}
    By Proposition \ref{prop: largelambdalimit}, the first term is bounded by 
    \begin{align*}
        \E\left[ \left\|A_{n,\theta}-\Sigma^{1/2} \right\|_F^2 \cdot \mathbf{1}(\mathcal{S}_{n,\theta}) \right] &\lesssim \E \left[ \frac{1}{1+2\lambda \sigma_{\min}^2} \left(f_{\Sigma,\lambda}(A_{n,\theta}) - \min f_{\Sigma,\lambda} \right) + \frac{d C_{\mu,1}}{\lambda^2} + \frac{K_{\mu}(1+C_{\mu,2})}{\lambda} \|A_{n,\theta}\|^2  \right],
    \end{align*}
    and by Proposition \ref{prop: stabilitylemma}, we have
    \begin{align*}
        &\E \left[f_{\Sigma,\lambda}(A_{n,\theta}) - \min f_{\Sigma,\lambda} \right] \leq \left| \mathcal{R}_{\mu,\lambda}(\theta) - \mathcal{R}(\mathcal{G}^{\dagger})_{\mu,\lambda} \right|   \\
        &+ O \left( \frac{\lambda}{n} \left(2\|\Sigma\|_F^2 + \E_{\Sigma \sim \mu} \sum_{i,j \in [d],  \neq j} \sigma_i^2 \sigma_j^2 \right) +  \frac{d^2 C_{\mu,1}}{\lambda} \left( 2(1+\sigma_{\max}) + \frac{C_{\mu,1}}{\lambda} \right) + \frac{2 d C_{\mu,1}(C_{\mu,1} + \sigma_{\max}^2)}{\lambda} \right).
    \end{align*}
    We conclude the proof.
\end{proof}
Proposition \ref{prop: concentration1} can be interpreted as follows: excluding terms which are small with respect to $\lambda$ and $n$, the transport map generalization is decomposed into a sum of two terms. The first term represents the excess loss, and by Theorem \ref{thm: lossgap}, it can be made small when $\theta$ is chosen to be the empirical risk minimizer. The second term is determined by the probability that $f_{\Sigma,\lambda}(A_{n,\theta}) - \min f_{\Sigma,\lambda} > 1.$ On the one hand, when $\theta$ is chosen to be the empirical risk minimizer, Proposition \ref{prop: stabilitylemma} guarantees that the mean of $f_{\Sigma,\lambda}(A_{n,\theta}) - \min f_{\Sigma,\lambda}$ tends to zero as $N$, $\frac{n}{\lambda}$, and $\lambda$ tend to $\infty$. On the other hand, since $A_{n,\theta}$ is an empirical covariance matrix of $n$ iid random vectors (this is true no matter the parameterization $\theta$), we should expect control of the variance of $f_{\Sigma,\lambda}(A_{n,\theta}) - \min f_{\Sigma,\lambda}$ tends to zero as $n \rightarrow \infty$; this is formalized by Lemma \ref{lem: weibullconcentration2} in Appendix \ref{sec: auxlemmas}. We are now in position to prove Theorem \ref{thm: TPGE}.

\begin{proof}[Proof of Theorem \ref{thm: TPGE}]
    Let $\widehat{\theta} \in \textrm{arg} \min_{\theta \in \Theta} \mathcal{R}_{\mu,\lambda}^N(\theta).$ In what follows, all expectations are conditioned on the training samples, and all estimates hold with high probability over the training data. Denote $\widehat{A}_n = A_{n,\widehat{\theta}}$. By Proposition \ref{prop: concentration1}, we have (excluding higher-order terms in $1/\lambda$ and setting $q = 2$ in Holder's inequality) 
    \begin{align}\label{eqn: tpgedecomp}
        \E \left[ \left\|\widehat{A}_n-\Sigma^{1/2} \right\|_F^2 \right] &\lesssim \frac{1}{1+2 \lambda \sigma_{\min}^2} \Bigg( \left| \mathcal{R}_{\mu,\lambda}(\widehat{\theta}) - \mathcal{R}(\mathcal{G}^{\dagger})_{\mu,\lambda} \right| + \frac{\lambda}{n} \left(2 + \E_{\Sigma \sim \mu} \sum_{i,j \in [d], \neq j} \sigma_i^2 \sigma_j^2 \right) \Bigg) \\
        &+ \frac{K_{\mu}(1+C_{\mu,2})}{\lambda} \E[\|\widehat{A}_n\|_F^2] +  \E[\|\widehat{A}_n\|_F^{4}]^{1/2} \cdot \mathbb{P} \left(f_{\Sigma,\lambda}(\widehat{A}_n) - \min f_{\Sigma,\lambda} > \lambda \epsilon_{\mu} - C_{\mu,2} \right)^{\frac{1}{2}}.
    \end{align}
    By Theorem \ref{thm: lossgap}, the excess loss is bounded by
    \begin{align*}
        &\left| \mathcal{R}_{\mu,\lambda}(\widehat{\theta}) - \mathcal{R}(\mathcal{G}^{\dagger})_{\mu,\lambda} \right| \\
        &= O \left( \frac{(1+\lambda)d^{22} \epsilon^{-4} \sigma_{\max}^8 + d^2 \sigma_{\max}^{16}}{\sqrt{N}} + \left(d(1+\sigma_{\max})\sigma_{\max}^2 \sqrt{\frac{d}{n}} + \lambda d \sigma_{\max}^6 \frac{d}{n} \right) + \left( d(1+\sigma_{\max})\epsilon + \lambda d \sigma_{\max}^2 \epsilon^2 \right)  \right),
    \end{align*}
    with high probability. Moreover, the above bound implies that $\E[\|\widehat{A}_n\|_F^2] = O(1)$ when $N$, $n$, and $1/\epsilon$ are large. To bound the tail probability term, note that Theorem \ref{thm: lossgap}, combined with Proposition \ref{prop: stabilitylemma}, imply that 
    $$  \E_{\Sigma \sim \mu, (y_1, \dots, y_n) \sim \mathcal{N}(0,\Sigma)} \left[f_{\Sigma,\lambda}(A_{n,\theta}) - \min_A f_{\Sigma,\lambda}(A) \right] \rightarrow 0,
    $$
    as $N,n,\lambda \rightarrow 0$. In particular, when $\lambda,$ $N$, $n$, and $1/\epsilon$ are large, it follows that
    \begin{align}
        \mathbb{P} \left( f_{\Sigma,\lambda}(\widehat{A}_n) - \min f_{\Sigma,\lambda} > \lambda \epsilon_{\mu} - C_{\mu,2} \right) \leq \mathbb{P}\left(f_{\Sigma,\lambda}(\widehat{A}_n) - \min f_{\Sigma,\lambda} > \E \left[f_{\Sigma,\lambda}(\widehat{A}_n) - \min f_{\Sigma,\lambda} \right] + c\lambda, \right),
    \end{align}
    where $c > 0$ is a constant depending only on $\epsilon_{\mu}$ and $C_{\mu,2}.$ Lemma \ref{lem: weibullconcentration2} shows that
    \begin{align*}
        &\mathbb{P}\left(f_{\Sigma,\lambda}(\widehat{A}_n) - \min f_{\Sigma,\lambda} > \E \left[f_{\Sigma,\lambda}(\widehat{A}_n) - \min f_{\Sigma,\lambda} \right] + c\lambda \Bigg| \Sigma \right) \\ &\leq \exp \left(- \left(\frac{\Omega(n \lambda)}{O \left(C_{\Theta}^2 L_{\Phi}^2 \right)} \right)^{1/2} \right) + \exp \left(- \left( \frac{\Omega(n)}{O \left(C_{\Theta}^2 L_{\Theta}^2 \right)} \right)^{1/4} \right) .
    \end{align*}
     The same estimate therefore holds when averaging over $\Sigma$. Here, $C_{\Theta}$ and $L_{\Phi}$ are the constants defining the parameter space $\Theta.$ In particular, in the proof of Theorem \ref{thm: lossgap}, we saw that it sufficed to take $C_{\Theta} \leq \sqrt{d}$ and $L_{\Phi} = O \left(d^{13/8} \epsilon^{-1/2} \right).$ For these choices of $C_{\Theta}$ and $L_{\Phi}$, the leading order of the tail probability bound becomes
     \begin{align*}
          \mathbb{P}\left(f_{\Sigma,\lambda}(\widehat{A}_n) - \min f_{\Sigma,\lambda} > \epsilon_{\mu} \lambda - C_{\mu,2} \right) = O \left( - \left( \Omega(n \epsilon) \right)^{1/4} \right).
     \end{align*}
     This shows that the final term of Inequality \eqref{eqn: tpgedecomp} is bounded by
     \begin{align*}
         &\E[\|\widehat{A}_n\|_F^{4}]^{1/2} \cdot \mathbb{P} \left(f_{\Sigma,\lambda}(\widehat{A}_n) - \min f_{\Sigma,\lambda} > \lambda \epsilon_{\mu} - C_{\mu,2} \right)^{\frac{1}{2}} \\
         &\leq \E[\|\widehat{A}_n\|_F^{4}]^{1/2} \cdot \exp \left(-\frac{1}{2} \left(\Omega \left(n \epsilon \right) \right)^{1/4} \right).
     \end{align*}
     Finally, Lemma \ref{lem: tfmomentbd} shows that the moment $\E[\|\widehat{A}_n \|^4]^{1/2}$ grows only polynomially in $(1/\epsilon)$. Omitting factors which are logarithmic in sample size and constant in problem parameters, this gives the final bound
     \begin{align*}
          \E \left[ \left\|\widehat{A}_n-\Sigma^{1/2} \right\|_F^2 \right] &= O \left( \frac{1}{1+2\lambda \sigma_{\min}^2} \left(\frac{(1+\lambda)\epsilon^{-4}}{\sqrt{N}} + \sqrt{\frac{1}{n}} + \frac{\lambda}{n} + \epsilon + \lambda \epsilon^2 \right) + \frac{1}{\lambda} + \textrm{poly}(\epsilon^{-1}) e^{-\frac{1}{2} \left(\Omega \left(n \epsilon \right) \right)^{1/4}} \right) %\exp \left(-\frac{1}{2} \left(\Omega \left(\frac{n}{\epsilon} \right) \right)^{1/4} \right) \right).
     \end{align*}
    In particular, setting $\lambda = n^{1/2}$ and $\epsilon = n^{-1/2} \cdot \textrm{polylog}(n)$, we get, up to log factors,
    \begin{align*}
         \E \left[ \left\|\widehat{A}_n-\Sigma^{1/2} \right\|_F^2 \right] &= O \left(\frac{n^2}{\sqrt{N}} + \frac{1}{\sqrt{n}} \right).
    \end{align*}
    The bound can also be stated in terms of the feedforward capacity bound $M$ by using the relation $M = \Theta(\epsilon^{-1/2})$
\end{proof}

It remains to prove Proposition \ref{prop: largelambdalimit}, which we divide into a series of lemmas. First, we show that for any $\Sigma \in \textrm{supp}(\mu)$, the minimizer is simultaneously diagonalizable with $\Sigma$, and the eigenvalues of the minimizer satisfy a related optimization problem.

\begin{lem}\label{lem: descriptionofminimizers}
    Fix $\Sigma \in \textrm{supp}(\mu)$, let $\sigma_1^2, \dots, \sigma_d^2$ denote the ordered eigenvalues of $\Sigma$, and suppose $A_{\lambda} \in \textrm{arg} \min_{A \in \R^{d \times d}_{sym}} f_{\Sigma,\lambda}(A).$ Then $A_{\lambda}$ is simultaneously diagonalizable with $\Sigma$. In addition, if we define $h_i: \R \rightarrow \R$ by $h_{\Sigma,\lambda,i}(x) = x^2 -2x + \lambda(x^2-\sigma_i^2)^2,$ then the ordered eigenvalues $a_{1,\lambda}, \dots, a_{d, \lambda}$ of $A_{\lambda}$ solve the optimization problem
    $$ a_{\lambda,i} \in \textrm{arg} \min_{a \in \R} h_{\Sigma,\lambda,i}(a).
    $$
\end{lem}
\begin{proof}
    Parameterize $A \in \R^{d \times d}_{sym}$ by $A = WDW^T$, with $W$ orthogonal and $D$ diagonal. The minimization problem then becomes
    \begin{align*}
        \min_{A \in \R^{d \times d}_{sym}} f_{\Sigma,\lambda}(A) &= \min_{W,D} \E_{x \sim \mathcal{N}(0,I)} \|WDW^T x - x\|^2 + \lambda \|W D^2 W^T - \Sigma\|_F^2 \\
        &= \min_{D} \min_{W} \trace(D^2-2D) + \lambda \left(\trace(D^4) + \trace(\Sigma^2) - 2 \trace(W D^2 W \Sigma) \right).
    \end{align*}
    The only term that depends on the $W$ is the quantity $\trace(WD^2W \Sigma)$, and by Von Neumann's trace inequality \citep{mirsky1975trace}, it is maximized when $W$ is chosen so that $W D^2 W$ and $\Sigma$ commute. This proves that the minimizer $A_{\lambda}$ is simultaneously diagonalizable with $\Sigma$. It is easy to see that if $A \in \R^{d \times d}_{sym}$ is simultaneously diagonalizable with $\Sigma$, then
    $$ f_{\Sigma,\lambda}(A) = \sum_{i=1}^{d} h_{\Sigma,\lambda,i}(a_i),
    $$
    and from this the claim follows.
\end{proof}
Next, we show that for each $i \in [d],$ the minimizers of $h_{\Sigma,\lambda,i}$ concentrate around $\sigma_i^2$.

\begin{lem}\label{lem: convergenceofminimizers}
    Fix $\Sigma \in \textrm{supp}(\mu).$ There exists $\lambda^{\ast}$, depending only on $\sigma_{\max}^2$ and $\sigma_{\min}^2$, such that for each $i \in [d]$, $$h_{\Sigma,\lambda,i}(\sigma_i) \leq h_{\Sigma,\lambda,i}\left(\sigma_i + \frac{c}{\lambda} \right),$$
    whenever $|c| > \max_i \frac{\left|1-\sigma_i \right|}{\sigma_i^2}.$ Consequently, when $\lambda > \lambda^{\ast}$, the minimizers of $h_{\Sigma,\lambda,i}$ are located in an interval of radius $O \left( \frac{1}{\lambda} \right)$ centered at $\sigma_i.$
\end{lem}
\begin{proof}
     Since the index $i \in [d]$ is fixed and arbitrary throughout this proof, we omit the dependence on $i$ and write $h_{\Sigma,\lambda}$ instead of $h_{\Sigma,\lambda,i}$ and $\sigma$ instead of $\sigma_i.$ For each $\lambda > 0$, define the function $g_{\lambda}(x) = \lambda \left( f(\sigma + \frac{x}{\lambda}) - f(\sigma) \right).$ Then it suffices to show for appropriately chosen $c^{\ast}$ and $\lambda^{\ast}$ that if $|x| > c^{\ast}$ and $\lambda > \lambda^{\ast},$ then $g_{\lambda}(x) \geq 0.$ To prove this, we compute directly
    \begin{align*}
        g_{\lambda}(x) &= 2x(\sigma-1) + \frac{x^2}{\lambda} + \lambda^2 \left(\frac{x^2}{\lambda^2}+\frac{2x \sigma}{\lambda} \right)^2 \\
        &= 4\sigma^2 x^2 + 2x(\sigma-1) + \frac{x^2+4\sigma x^3}{\lambda} + \frac{x^4}{\lambda^2}.
    \end{align*}
    Thus, $g_{\lambda}$ is a polynomial with $\lim_{|x| \rightarrow \infty} g_{\lambda}(x) = \infty$ whose coefficients converge to the coefficients of $g_{\infty}(x) := 2x(\sigma - 1) + 4\sigma^2 x^2$ as $\lambda \rightarrow \infty$. Note that the roots of $g_{\infty}$ are $x_1 = 0$ and $x_2 = \frac{(1-\sigma)}{2\sigma^2}.$ Since the roots of a polynomial are continuous functions of its coefficients, there must exist a $\lambda^{\ast} > 0$ such that the roots of $g_{\lambda}$ are contained in the interval $\left[- \frac{(1-\sigma)}{\sigma^2},  \frac{(1-\sigma)}{\sigma^2} \right]$ whenever $\lambda \geq \lambda^{\ast}.$ It follows that whenever $\lambda \geq \lambda^{\ast}$ and $|x| >  \frac{|1-\sigma|}{\sigma^2}$, we have $g_{\lambda}(x) \geq 0.$
\end{proof}

Lemmas \ref{lem: descriptionofminimizers} and \ref{lem: convergenceofminimizers} allow us to prove quantitative bounds between the minimizer of $h_{\Sigma,\lambda}$ and $\Sigma^{1/2}$. With a little bit more work, we can also prove a bound the error $\|A-\Sigma^{1/2}\|_F$ in terms of the risk gap $h_{\Sigma,\lambda}(A) - \min h_{\Sigma,\lambda}$ for any matrix $A \in \R^{d \times d}_{sym}$ which is simultaneously diagonalizable with $\Sigma$.

\begin{lem}\label{lem: stability1}
    Fix $\Sigma \in \textrm{supp}(\mu).$ Then there exists $\lambda^{\ast} > 0$ such that when $\lambda > \lambda^{\ast}$, the following statements hold.
    \begin{enumerate}
        \item If $A_{\lambda} \in \textrm{arg} \min_{A \in \R^{d \times d}_{sym}} h_{\Sigma,\lambda}(A)$, then $\|A-\Sigma^{1/2}\|_F^2 \leq \frac{C_{\mu,1}d}{\lambda^2}$, where $C_{\mu,1} := \max_{\Sigma \in \textrm{supp}(\mu), \; i \in [d]} \frac{(1-\sigma_i)^2}{\sigma_i^4}.$
        \item If $A \in \R^{d \times d}_{sym}$ is simultaneously diagonalizable with $\Sigma,$ then $$\|A-\Sigma^{1/2}\|_F^2 \leq \frac{1}{1 + 2 \lambda \sigma_{\min}^2} \left(f_{\Sigma,\lambda}(A) - \min f_{\Sigma,\lambda} \right) + \frac{2d C_{\mu,1}}{\lambda^2}.$$
    \end{enumerate}
\end{lem}
\begin{proof}
    Take $\lambda^{\ast}$ large enough that the conclusion of Lemma \ref{lem: convergenceofminimizers} holds and that
    $$ \lambda^{\ast} > \max_{\Sigma \in \textrm{supp}(\mu), \; i \in [d]} \sqrt{\frac{3}{2}} \frac{|1-\sigma_i|}{\sigma_i^3}.
    $$
    Lemmas \ref{lem: descriptionofminimizers} and \ref{lem: convergenceofminimizers} imply that the ordered eigenvalues $a_{\lambda,i}$ of $A$ satisfy $|a_{\lambda,i} - \sigma_i| \leq \frac{|1-\sigma_i|}{\sigma_i^2 \lambda}.$ The bound on $\|A_{\lambda}-\Sigma^{1/2}\|_F^2$ then follows from the fact that $A_{\lambda}$ and $\Sigma^{1/2}$ are simultaneously diagonalizable. To prove the second claim, we compute that
    $$ h_{\Sigma,\lambda,i}''(x) = 2 + 4\lambda (3x^2-\sigma^2).
    $$
    Since $|a_{\lambda,i} - \sigma_i| \leq \frac{|1-\sigma_i|}{\sigma_i^2}$, the choice of $\lambda^{\ast}$ implies that $h_{\Sigma,\lambda,i}''(a_{\lambda,i}) \geq 2 + 4 \lambda \sigma_i^2$, which in turn implies that
    $$ h_{\Sigma,\lambda,i}(x) - h_{\Sigma,\lambda,i}(a_{\lambda,i}) \geq (2+4 \lambda \sigma^2)(x-a_{\lambda,i})^2.
    $$
    To conclude, we note that when $A$ and $\Sigma$ are simultaneously diagonalizable, we have (with $a_1, \dots, a_d$ denoting the ordered eigenvalues of $A$)
    \begin{align*}
    \|A-\Sigma^{1/2}\|_F^2 &= \sum_{i=1}^{d} (a_i-\sigma_i)^2 \\
    &\leq 2 \sum_{i=1}^{d} \left( (a_i - a_{\lambda,i})^2 + (a_{\lambda,i} - \sigma_i)^2 \right) \\
    &\leq \frac{1}{1+2 \lambda \sigma_{\min}^2} \sum_{i=1}^{d} \left( h_{\Sigma,\lambda,i}(a_i) - \min h_{\Sigma,\lambda,i} \right) + \frac{2d C_{\mu,1}}{\lambda^2} \\
    &= \frac{1}{1 + 2 \lambda \sigma_{\min}^2} \left(f_{\Sigma,\lambda}(A) - \min f_{\Sigma,\lambda} \right) + \frac{2d C_{\mu,1}}{\lambda^2}.
    \end{align*}
\end{proof}

When $\lambda$ is large, Lemma \ref{lem: stability1} provides a stability bound to control the error between $\|A-\Sigma^{1/2}\|$ by the loss of $A$ with respect to $h_{\Sigma,\lambda}$ \textit{when $A$ is simultaneously diagonalizable}. When $A$ is not necessarily simultaneously diagonalizable with $\Sigma$, we would like to say that, provided $\lambda$ is large and $h_{\Sigma,\lambda} - \min h_{\Sigma,\lambda}$ is small, $A$ is almost simultaneously diagonalizable with $\Sigma$. This property is ensured by Assumption \ref{assum: continuityofeigenprojections}, and we are now in position to prove Proposition \ref{prop: largelambdalimit}.

\begin{proof}[Proof of Proposition \ref{prop: largelambdalimit}]
    We take $\lambda_{\mu}$ to be the smallest value such that the conclusion of Lemma \ref{lem: stability1} holds. For item 1), the bound $\|A_{\lambda} - \Sigma^{1/2}\|_F^2 \leq \frac{C_{\mu,1}d}{\lambda^2}$ has already been proven in Lemma \ref{lem: stability1}. To prove item 2), we use the estimate
    \begin{align*}
        \E_{x \sim N(0,x)}[\|A_{\lambda}x-x\|^2] &= \trace(A_{\lambda}^2) + d - 2 \trace(A_{\lambda}) \\
        &= \trace((\Sigma^{1/2} + (A_{\lambda}-\Sigma^{1/2}))^2) + d - 2 \trace(\Sigma^{1/2} + (A_{\lambda}-\Sigma^{1/2})) \\
        &= W_2^2(\mathcal{N}(0,I),\mathcal{N}(0,\Sigma)) + \trace((A_{\lambda}-\Sigma^{1/2})^2) + 2 \trace(\Sigma^{1/2}(A_{\lambda}-\Sigma^{1/2})) - 2 \trace(A_{\lambda}-\Sigma^{1/2}).
    \end{align*}
    It follows that
    \begin{align*}
        \left| \E_{x \sim N(0,x)}[\|A_{\lambda}x-x\|^2] - W_2^2(\mathcal{N}(0,I),\mathcal{N}(0,\Sigma)) \right| &\leq \left| \trace((A_{\lambda}-\Sigma^{1/2})^2) + 2 \trace(\Sigma^{1/2}(A_{\lambda}-\Sigma^{1/2})) - 2 \trace(A_{\lambda}-\Sigma^{1/2}) \right| \\
        &\leq \|A_{\lambda}-\Sigma^{1/2}\|_F^2 + 2 \|\Sigma^{1/2}\|_F \|A_{\lambda}-\Sigma^{1/2}\|_F + 2d \|A_{\lambda}-\Sigma^{1/2}\|_F \\
        &\leq \frac{d^2 C_{\mu,1}}{\lambda} \left(2(1+\sigma_{\max}) + \frac{C_{\mu,1}}{\lambda} \right),
    \end{align*}
    where we used the Cauchy-Schwarz inequality in the second line and the bound from item 1) in the third line. In addition, Proposition \ref{lem: convergenceofminimizers} implies that
    \begin{align*}
        \|A_{\lambda}^2 - \Sigma\|_F^2 &= \sum_{i=1}^{d} (a_{\lambda,i}^2 - \sigma^2)^2 \\
        &= \sum_{i=1}^{d} (a_{\lambda,i}+\sigma_i)^2(a_{\lambda,i}-\sigma_i)^2 \\
        &\leq \frac{C_{\mu,1}}{\lambda^2} \sum_{i=1}^{d} (a_{\lambda,i}+\sigma_i)^2 \\
        &\leq \frac{2C_{\mu,1}}{\lambda^2} \sum_{i=1}^{d} \left( \sigma_i^2 + \frac{C_{\mu,1}}{\lambda^2} \right) \\
        &\leq \frac{2d C_{\mu,1}(C_{\mu,1}+\sigma_{\max}^2)}{\lambda^2},
    \end{align*}
    where, in the second to last line, we used the inequality $(a_{\lambda,i}+\sigma_i)^2 \leq 2 \left( (a_{\lambda,i}-\sigma_i)^2 + \sigma_i^2 \right) \leq 2 \left(\frac{C_{\mu,1}}{\lambda^2} + \sigma_i^2 \right).$ It follows that
    \begin{align*}
        \left|\min_{A} f_{\Sigma,\lambda}(A) -  W_2^2(\mathcal{N}(0,I),\mathcal{N}(0,\Sigma))\right| &= \left|\E_{x \sim N(0,x)}[\|A_{\lambda}x-x\|^2] + \lambda \|A_{\lambda}^2 - \Sigma\|_F^2 - W_2^2(\mathcal{N}(0,I),\mathcal{N}(0,\Sigma)) \right| \\
        &\leq \frac{d^2 C_{\mu,1}}{\lambda} \left(2(1+\sigma_{\max}) + \frac{C_{\mu,1}}{\lambda} \right) + \frac{2d C_{\mu,1}(C_{\mu,1}+\sigma_{\max}^2)}{\lambda}.
    \end{align*}
    To prove item 3), notice that the inequality $f_{\Sigma,\lambda}(A) - \min f_{\Sigma,\lambda} \leq \eta$ implies that
    \begin{align*}
        \|A^2 - \Sigma\|_F^2 &\leq \frac{f_{\Sigma,\lambda}(A)}{\lambda} \\
        &\leq \frac{\eta + \min f_{\Sigma,\lambda}}{\lambda} \\
        &= \frac{\eta + W_2^2(\mathcal{N}(0,I), \mathcal{N}(0,\Sigma))}{\lambda} \\
        &\leq \frac{\eta + C_{\mu,2}}{\lambda}.
    \end{align*}
    Recall the constant $\epsilon_{\mu} > 0$ defined in Assumption \ref{assum: continuityofeigenprojections}. Then the assumption $\eta \leq \lambda \epsilon_{\mu} + C_{\mu,2}$, along with the estimate above, together imply that $\|A^2 - \Sigma\|_F^2 \leq \epsilon_{\mu}$. Hence, the bound in Assumption \ref{assum: continuityofeigenprojections} applies, and the matrices $A$ and $\Sigma^{1/2}$ can be almost simultaneously diagonalized: there exist orthogonal matrices $U, U_{A}$ such that $U \Sigma^{1/2} U^T$ and $U_{A} A U_{A}^T$ are diagonal and
    \begin{align*}
        \|U-U_A\|^2 \leq K_{\mu} \cdot \sup_{\Sigma \in \textrm{supp}(\mu)} \frac{1 + W_2^2(\mathcal{N}(0,I), \mathcal{N}(0,\Sigma))}{\lambda}.
    \end{align*}
    We let $D_A$ denote the diagonal matrix $D_A = U_A^T A U_A$ and we also define $A_U = U D_A U^T.$ Then we have
    $$ \|A-\Sigma^{1/2}\|_F \leq \|A_U - \Sigma^{1/2}\|_F +  \|A-A_U\|_F.
    $$
    For the first term, we can directly apply Lemma \ref{lem: stability1} to bound
    \begin{align*}
        \|A_U - \Sigma^{1/2}\|_F &\leq \sqrt{\frac{1}{1 + 2 \lambda \sigma_{\min}^2} \left(f_{\Sigma,\lambda}(A_U) - \min f_{\Sigma,\lambda} \right) + \frac{2d C_{\mu,1}}{\lambda^2}} \\
        &\leq \sqrt{\frac{1}{1 + 2 \lambda \sigma_{\min}^2} \left(f_{\Sigma,\lambda}(A) - \min f_{\Sigma,\lambda} \right) + \frac{2d C_{\mu,1}}{\lambda^2}},
    \end{align*}
    where the inequality $f_{\Sigma,\lambda}(A_U) \leq f_{\Sigma,\lambda}(A)$ follows again from Von-Neumann's trace inequality. For the second term, we have by the triangle inequality that
    \begin{align*}
        \|A-A_U\|_F &\leq 2 \|D_{A}\|_{\op} \|U-U_A\|_F \\
        &\leq 2 \|A\|_{\op} \cdot \sqrt{\frac{K_{\mu}(1+C_{\mu,2})}{\lambda}}.
    \end{align*}
\end{proof}
In turn, this proves the result

\section{Auxiliary lemmas}\label{sec: auxlemmas}

\subsection{MMD properties}
The proof of Lemma \ref{lem:terminal} requires the following result for $\mathrm{MMD}^2$ and its empirical estimator $\mathrm{MMD}^2_u$, which holds for general bounded kernel functions.  
\begin{lem}\label{lem:mmd_ub}
Let $K>0$. Suppose the kernel function satisfies $|k(\cdot,\cdot)|\leq K$. Then, for any $p,q\in\gP_2(\R^d)$, and $X=\{x_1,\dots,x_m\}\stackrel{iid}{\sim}p, Y=\{y_1,\dots,y_m\}\stackrel{iid}{\sim}q$, we have
\begin{align*}
   0 \leq\mathrm{MMD}^2(p,q)\leq 4K~~~~~\text{and}~~~~~-2K\leq\mathrm{MMD}^2_u(X,Y)\leq 2K.
\end{align*}    
Furthermore, if there exists $L_k>0$ such that $|k(x,y)-k(x',y')|\leq L_k(\|x-x'\|+\|y-y'\|)$, then
$\mathrm{MMD}^2$ is Lipschitz in the first component in terms of Wasserstein-2 distance with Lipschitz constant $2L_k(1+\frac{1}{m})$.

\end{lem}
\begin{proof}
By Riesz representation theorem, there exists a point-evaluation mapping $\varphi:\R^d\to \R$ such that $\langle f,\varphi(x)\rangle_\gH = f(x)$. Taking the canonical form, $\varphi(x) = k(x,\cdot)$ and $\langle \varphi(x),\varphi(y)\rangle_\gH = k(x,y)$. By Cauchy-Schwarz inequality, 
\begin{align}\label{Tpf}
    |\mathbb{E}_{x\sim p}[f(x)]|\leq \mathbb{E}_{x\sim p} |f(x)| = \mathbb{E}_{x\sim p} |\langle f,\varphi(x)\rangle_\gH|\leq  \mathbb{E}_{x\sim p}\big[\sqrt{k(x,x)}\|f\|_\gH\big]\leq K^{\frac{1}{2}}\|f\|_\gH.
\end{align}
Recall the definition of $\mathrm{MMD}^2$ in \eqref{def_mmd}, \eqref{Tpf} implies
\begin{align*}
    \mathrm{MMD}^2(p,q) &= \Big[\sup_{\|f\|_{\gH}\leq 1} \Big(\mathbb{E}_{x\sim p}[f(x)] - \mathbb{E}_{y\sim q}[f(y)]\Big)\Big]^2\\
    &\leq 2\sup_{\|f\|_{\gH}\leq 1} \mathbb{E}^2_{x\sim p}[f(x)] + 2\sup_{\|f\|_{\gH}\leq 1} \mathbb{E}^2_{y\sim q}[f(y)]\leq 4K.
\end{align*}
Bounds for $\mathrm{MMD}^2_u(X,Y)$ can be readily checked from the definition \eqref{def_mmdu}. Notice that $\mathrm{MMD}^2_u$ as an unbiased estimator can be negative. To show the Lipschitzness of $\mathrm{MMD}^2$ , let $X'=\{x'_1,...,x'_m\}\stackrel{iid}{\sim}p$ independent from $X$, and $\sigma$ be the optimal permutation so that 
\begin{align*}
    W^2_2\Big(\frac{1}{m}\sum^m_{i=1}\delta_{x_i}, \frac{1}{m}\sum^m_{i=1}\delta_{x'_i}\Big) = \frac{1}{m}\sum^m_{i=1}\|x_i-x'_{\sigma(i)}\|^2.
\end{align*}
We have
\begin{align*}
    &\big|\mathrm{MMD}^2(X,Y)-\mathrm{MMD}^2(X',Y)\big|\\
    = & \frac{1}{m^2}\Big|\sum^m_{i,j}k(x_i,x_j)-2k(x_i,y_j)- k(x'_{\sigma(i)},x'_{\sigma(j)})+2k(x'_{\sigma(i)},y_j)\Big| \\
    \leq& \frac{L_k}{m^2}\sum^m_{i,j}\|x_i-x'_{\sigma(i)}\|+ \|x_j-x'_{\sigma(j)}\| + \frac{2L_k}{m^2}\sum^m_i \|x_i-x'_{\sigma(i)}\| \\
    =& \frac{2L_k}{m}\sum^m_{i}\|x_i-x'_{\sigma(i)}\|+\frac{2L_k}{m^2}\sum^m_i \|x_i-x'_{\sigma(i)}\|
    \leq 2L_k\Big(1+\frac{1}{m}\Big)W_2\Big(\frac{1}{m}\sum^m_{i=1}\delta_{x_i}, \frac{1}{m}\sum^m_{i=1}\delta_{x'_i}\Big).
\end{align*}

\end{proof}

\subsection{Neural network results}\label{app: neuralnet}
We begin with a precise definition of the neural network function class $\Phi(M)$. Given $M > 0$, let $\nn(M)$ denote the class of one-dimensional shallow ReLU networks with weights bounded by $M$:
\begin{equation}\label{eqn: neuralnet}
    \nn(M) = \left\{ z \mapsto \sum_{k=1}^{m} c_k (w_k z+b_k)^+ : c_k, w_k, b_k \in \R, \; \sum_{k=1}^{K} |c_k| \left(|w_k| + |b_k| \right) \leq M \right \},
\end{equation}
where $(z)^+ = \max(0,z)$ denotes the ReLU activation function. When $x \in \R^d$ and $\psi \in \nn(M)$, we define $\psi(x)$ by componentwise application. We then define the class of admissible feedforward layers $\Phi = \Phi(M)$ to be the class of functions in $\nn(M)$, extended to $\R^d$, with an additional inner linear layer:
\begin{equation}
    \Phi = \Phi(M) = \left\{x \mapsto \phi(x) = \psi(Wx): \|W\|_{\op} \leq C_{\Theta}, \; \psi \in \nn(M) \right\}.
\end{equation}
The values $C_{\Theta}$ and $M$ defining the complexity of the model class will be defined precisely in the statement of our results. We also define $L_{\Phi}$ to be a uniform bound on the Lipschitz constant of elements in $\Phi$. A quick computation shows that $L_{\Phi} \leq \sqrt{d} C_{\Theta} M.$

A key element of our proof relies on the ability of neural networks to estimate the extended square root function $g_{sq}: \R \rightarrow \R$
\begin{align*}
    g_{sq}(x) = \begin{cases}
        \sqrt{x}, \; x \geq 0 \\
        -\sqrt{-x}, \; x < 0.
    \end{cases}
\end{align*}

\begin{lem}\label{lem: sqrtapproximation}
    For any $\epsilon, R > 0$, there exists $\psi \in \nn(M)$ with $m = O(R^2 \epsilon^{-2})$ and $M = O(\epsilon^{-1/2})$ such that
    $$ \sup_{|x| \leq R} |g_{sq}(x) - \psi(x)| \leq \epsilon.
    $$
\end{lem}
\begin{proof}
    Given $\epsilon > 0$, define the function $g_{\epsilon}: \R \rightarrow \R$ by
    $$ g_{\epsilon}(x) = \begin{cases}
        g_{sq}(x), \; |x| \geq \epsilon \\
        \frac{x}{\sqrt{\epsilon}}, \; |x| < \epsilon.
    \end{cases}
    $$
    Then, for any $x \in \R$, we have $|g_{sq}(x) - g_{\epsilon}(x)| \leq \frac{\sqrt{\epsilon}}{4}.$ By Example 4.1 in \citep{wojtowytsch2022representation}, the function $g_{\epsilon}$ belongs to the Barron space of functions on $[-R,R]$ with Barron norm $M = O(\epsilon^{-1/2}).$ Therefore, according to Equation 1.4 in \citep{weinan2022some}, there exists a neural network $\psi \in \nn(M)$ such that
    $$ \sup_{|x| \leq R} |g_{\epsilon}(x) - \psi(x)| \leq \frac{MR}{\sqrt{m}} = O \left(\frac{R}{\sqrt{m \epsilon}} \right).
    $$
    Taking $m = O(R^2/\epsilon^2),$ we obtain that $\psi$ approximates $g_{\epsilon}$ uniformly on $[-R,R]$ to accuracy $3\epsilon/4$. In turn, $\psi$ approximates $g_{sq}$ uniformly on $[-R,R]$ to accuracy $\epsilon$.
\end{proof}

In order to control the statistical error, we will also need to control the covering number of our neural network class.

\begin{lem}\label{lem: nncoveringnum}
    Recall the definition of the function class $\Phi = \phi(M)$:
    \begin{align*}
        \phi(M) \left\{x \mapsto \phi(x) = \psi(Wx): \|W\|_{\op} \leq C_{\Theta}, \; \psi \in \nn(M) \right\},
    \end{align*}
    where $\nn(M)$ is the one-dimensional neural network class defined in Equation \ref{eqn: neuralnet}. Then, for any $R > 0$, the covering number of $\phi(M)$ satisfies the bound
    \begin{align*}
        \log N \left(\epsilon, \phi(M), L^{\infty}(B_R) \right) \lesssim d^2 \log \left(1 + \frac{4 \sqrt{d} M R}{\epsilon} \right) + \frac{M^2}{\epsilon^2}
    \end{align*}
\end{lem}
\begin{proof}
    For any $\psi_1, \psi_2 \in \nn(M)$ and $W_1, W_2$ with $\|W_1\|_{\op}, \|W_2\|_{\op} \leq C_{\Theta}$, and any $x \in B_R$, we have
    \begin{align*}
        \|\psi_1(W_1 x) - \psi_2(W_2 x)\| &\leq \|(\psi_1 - \psi_2)(W_1 x)\| + \|\psi_2(W_1x) - \psi_2(W_2 x)\| \\
        &\leq \|\psi_1-\psi_2\|_{L^{\infty}(B_{C_{\Theta} R})} + \|\psi_2\|_{\textrm{Lip}} \|(W_1-W_2)x\| \\
        &\leq \|\psi_1-\psi_2\|_{L^{\infty}(B_{C_{\Theta} R})} + \sqrt{d} M R \|W_1-W_2\|_{\op},
    \end{align*}
    where we used that any $\psi \in \nn(M)$ is $\sqrt{d} M$-Lipschitz when viewed as a mapping on $\R^d$. It follows that
    \begin{align*}
        \log N \left(\epsilon, \phi(M), L^{\infty}(B_R) \right) &\leq \log N \left(\frac{\epsilon}{2}, \nn(M), L^{\infty}(B_{C_{\Theta}R)} \right) + \log N \left(\frac{\epsilon}{2 \sqrt{d} M R}, B_{C_{\Theta}}(\R^{d \times d}), \| \cdot \|_{\op} \right),
    \end{align*}
    where $B_{C_{\Theta}}(\R^{d \times d})$ denotes the ball of radius $C_{\Theta}$ in $\R^{d \times d}$ in the operator norm. By Example 5.8 in \citep{wainwright2019high}, we have the bound
    \begin{align*}
        N \left(\frac{\epsilon}{2 \sqrt{d} M R}, B_{C_{\Theta}}(\R^{d \times d}), \| \cdot \|_{\op} \right) \leq d^2 \log \left(1 + \frac{4\sqrt{d} M R}{\epsilon} \right).
    \end{align*}
    Next, by Equation 1.32 in \citep{siegel2024sharp}, we have
    \begin{align*}
        \log N \left(\frac{\epsilon}{2}, \nn(M), L^{\infty}(B_{C_{\Theta}R}) \right) &\lesssim \frac{M^2}{\epsilon^2}.
    \end{align*}
    Hence the result.
\end{proof}

\subsection{Complexity bounds for transformers}
In this section, we prove pointwise bounds on the norm of the matrix $A_{n,\theta}$ for $\theta \in \Theta$, as well as bounds on its moments. We begin with the following pointwise bound assuming deterministic and bounded input data $y_1, \dots, y_n.$ In turn, it implies a high probability bound on the norm of $A_{n,\theta}$ over the Gaussian inputs.
\begin{lem}\label{lem: normbound1}
    For any $\theta \in \Theta$ and $y_1, \dots, y_n \in B_R$, the matrix $A_{n,\theta}$ satisfies $\|A_{n,\theta}\|_F \leq d C_{\Theta}^2 L_{\Phi}^2 R^2.$ In particular, if $y_1, \dots, y_n$ are iid random vectors drawn from $\mathcal{N}(0,\Sigma)$, then the bound holds with probability at least $1 - 2n \exp \left(-\frac{(R-\sqrt{\trace(\Sigma)})^2}{C \|\Sigma\|_{\op}} \right)$ for some universal constant $C > 0$.
\end{lem}
\begin{proof}
    We will frequently use the inequality $\|AB\|_F \leq \|A\|_{\textrm{op}} \|B\|_{F}$ and the Cauchy-Schwarz inequality for the Frobenius norm. We have
    \begin{align*}
        \|A_{n,\theta}\|_F &= \left\| Q \cdot \frac{1}{n} \sum_{i=1}^{n} \phi(y_i) \phi(y_i)^T \cdot Q^T \right\|_F \\
        &\leq \|Q\|_F \left\| \frac{1}{n} \sum_{i=1}^{n} \phi(y_i) \phi(y_i)^T \cdot Q^T \right\|_F \\
        &\leq \|Q\|_F^2 \left\| \frac{1}{n} \sum_{i=1}^{n} \phi(y_i) \phi(y_i)^T \right\|_{\op} \\
        &\leq \|Q\|_F^2 \cdot \frac{1}{n} \sum_{i=1}^{n} \|\phi(Wy_i)\|^2 \\
        &\leq \|\phi\|_{\textrm{Lip}} \|Q\|_F^2 \cdot \frac{1}{n} \sum_{i=1}^{n} \|y_i\|^2 \\
        &\leq \|\phi\|_{\textrm{Lip}}^2 \|Q\|_F^2 R^2 \\
        &\leq C_{\Theta}^2 L_{\Phi}^2 R^2.
    \end{align*}
    This concludes the bound when $y_1, \dots, y_n \in B_R$ are deterministic. The high-probability bound over Gaussian covariates $y_1, \dots, y_n \sim \mathcal{N}(0,\Sigma)$ follows from the Gaussian concentration inequality in Lemma \ref{lem: gaussianconc}.
\end{proof}
It will also be convenient to establish pointwise bounds for the individual loss $\ell$.
\begin{lem}\label{lem: lossbound1}
    For any $\theta \in \Theta$, any $R_1, R_2 > 0$, and any $\theta \in \Theta$ and $y_1, \dots, y_{2n}$ such that $y_1, \dots y_n \in B_{R_1}$ and $\left\| \frac{1}{n} \sum_{i=1}^{n} y_{n+i} y_{n+i}^T \right\|_{\op} \leq R_2$, it holds that
    \begin{align*}
        \ell(\theta, y_1, \dots, y_{2n}) &\leq 2 \left(C_{\Theta}^4 L_{\Phi}^4 R_1^4 + d + \lambda \left(C_{\Theta}^8 L_{\Phi}^8 R_1^8 + d R_2^2 \right) \right).
    \end{align*}
\end{lem}
\begin{proof}
    Recall the notation $\Sigma_n =  \frac{1}{n} \sum_{i=1}^{n} y_{n+i} y_{n+i}^T.$ We have
    \begin{align*}
        \ell(\theta, y_1, \dots, y_{2n}) &= \E_{x \sim \mathcal{N}(0,I)} \|A_{n,\theta}x-x\|^2 + \lambda \left\|A_{n,\theta}^2 - \Sigma_n \right\|_F^2 \\
        &\leq 2 \left(\|A_{n,\theta}\|_F^2 + d + \lambda \left(\|A_{n,\theta}\|_F^4 + \|\Sigma_n\|_F^2 \right) \right) \\
        &\leq 2 \left(C_{\Theta}^4 L_{\Phi}^4 R_1^4 + d + \lambda \left(C_{\Theta}^8 L_{\Phi}^8 R_1^8 + d R_2^2 \right) \right),
    \end{align*}
    where we used Lemma \ref{lem: normbound1} to bound the norm of $A_{n,\theta}.$
\end{proof}

Next, we translate the high-probability bounds to bounds on the even moments. 
\begin{lem}\label{lem: tfmomentbd}
    Let $\theta \in \Theta$. Then, for any $k \in \mathbb{N}$,
        $$ \E[\|A_{n,\theta}\|_F^{2k}] \lesssim 1 + (C_{\Theta} L_{\Phi})^{\frac{2(k-1)}{2}} \left( \sigma_{\max}^{k-\frac{1}{2}} + d^{\frac{k-1}{2}} \sigma_{\max}^{\frac{k}{2}} \cdot  \exp \left(-\frac{(C_{\Theta}^2 L_{\Phi}^2 - \sqrt{\trace(\Sigma)})^2}{C \|\Sigma\|_{\op}^2} \right) \right),
        $$
        where the implicit constants depend only on $k$.
\end{lem}
\begin{proof}
    Throughout this proof, we let $C > 0$ denote the universal constant from the Gaussian concentration bound of Lemma \ref{lem: gaussianconc}. We use the layer-cake trick
    \begin{align*}
        &\E\left[\|A_{n,\theta}\|_F^{2k} \right] = \int_{0}^{\infty} \mathbb{P} \left(\|A_{n,\theta}\|_F^{2k} > r  \right) dr \\ 
        &\leq 1 + \int_{1}^{\infty} \mathbb{P} \left(\|A_{n,\theta}\|_F^{2k} > r  \right) dr \\
        &=  1+ 2k \int_{1}^{\infty} r^{2k-1} \mathbb{P} \left(\|A_{n,\theta}\|_F > r  \right) dr \\
        &= 1 + k \left(C_{\Theta} L_{\Phi} \right)^{\frac{2(k-1)}{2}} \int_{C_{\Theta}^2 L_{\Phi}^2} r^{\frac{2(k-1)}{2}} \mathbb{P} \left(\|A_{n,\theta}\|_F > C_{\Theta}^2 L_{\Phi}^2 r^2  \right) dr \\
        &\leq 1 + k \left(C_{\Theta} L_{\Phi} \right)^{\frac{2(k-1)}{2}} \int_{C_{\Theta}^2 L_{\Phi}^2}^{\infty} r^{\frac{2(k-1)}{2}} \exp \left(- \frac{(r-\sqrt{\trace(\Sigma)})^2}{C \|\Sigma\|_{\op}} \right) dr, \; (\textrm{Lemma \ref{lem: normbound1}}) \\
        &= 1 + k\left(C_{\Theta} L_{\Phi} \right)^{\frac{2(k-1)}{2}} \int_{C_{\Theta}^2 L_{\Phi}^2 - \sqrt{\trace(\Sigma)}}^{\infty} \left( r + \sqrt{\trace(\Sigma)} \right)^{\frac{2(k-1)}{2}} \exp \left(-\frac{r^2}{C \|\Sigma\|_{\op}} \right) dr \\
        &\leq 1 + k\left(C_{\Theta} L_{\Phi} \right)^{\frac{2(k-1)}{2}} \cdot \sqrt{2\pi C \|\Sigma\|_{\op}} \left(\E_{z \sim N(0, C \|\Sigma\|_{\op}/2)}[(z+\sqrt{\trace(\Sigma)})^{2(k-1)}]^{1/2} \cdot \exp \left(-\frac{(C_{\Theta}^2 L_{\Phi}^2 - \sqrt{\trace(\Sigma)})^2}{C \|\Sigma\|_{\op}^2} \right) \right) \\
        &\lesssim  1 + \left(C_{\Theta} L_{\Phi} \right)^{\frac{2(k-1)}{2}} \cdot  \sqrt{\|\Sigma\|_{\op}} \left( \left(\|\Sigma\|_{\op}^{(k-1)}) + \trace(\Sigma)^{(k-1)/2} \right) \cdot \exp \left(-\frac{(C_{\Theta}^2 L_{\Phi}^2 - \sqrt{\trace(\Sigma)})^2}{C \|\Sigma\|_{\op}^2} \right) \right),
    \end{align*}
    where we have hidden constants depending only on $k$ and $C.$ The final bound arises from the inequalities $\|\Sigma\|_{\op} \leq \sigma_{\max}^2$ and $\trace(\Sigma) \leq d \sigma_{\max}^2$.
\end{proof}

The following result implies bounds on the covering number of $\Theta$ (in terms of the covering number of $\Phi$) when the data is bounded.
\begin{lem}\label{lem: coveringnumattn}
    Let $\theta_1 = (Q_1, \phi_1) \in \Theta$ and $\theta_2 = (Q_2, \phi_2) \in \Theta$ be two transformer parameterizations. Let $y_1, \dots, y_n \in B_R$ and let $A_{n,11} = A_{n,\theta_1}$ and $A_{n,2} = A_{n,\theta_2}$ be defined by Equation \eqref{eqn: Atheta}. Then
    \begin{enumerate}
        \item $|\trace(A_{n,1} - A_{n,2})| \leq 2 C_{\Theta}^2 L_{\Phi}^2 R^2 \left(\|Q_1-Q_2\|_F + \|\phi_1-\phi_2\|_{L^{\infty}(B_R)} \right)$
        \item $|\trace(A_{n,1}^2 - A_{n,2}^2)| \leq 4 C_{\Theta}^4 L_{\Phi}^4 R^4 \left(\|Q_1-Q_2\|_F + \|\phi_1-\phi_2\|_{L^{\infty}(B_R)} \right)$
        \item $|\trace(\Sigma(A_{n,1}^2 - A_{n,2}^2)| \leq 4 C_{\Theta}^4 L_{\Phi}^4 R^4 \|\Sigma\|_{\op} \|A_{n,1}-A_{n,2}\|_F $ for any $\Sigma$.
        \item $|\trace(A_{n,1}^4 - A_{n,2}^4) \leq 8 C_{\Theta}^8 L_{\Phi}^8 R^8 \|A_{n,1}-A_{n,2}\|_F.$
    \end{enumerate}
\end{lem}

\begin{proof}
    For 1), we have
    \begin{align*}
        &|\trace(A_{n,1}-A_{n,2})| = \left| \trace \left(Q_1 \cdot \frac{1}{n} \sum_{i=1}^{n} \phi_1(y_i)\phi_1(y_i)^T \cdot Q_1 - Q_2 \cdot \frac{1}{n} \sum_{i=1}^{n} \phi_2(y_i) \phi_2(y_i)^T \cdot Q_2^T \right) \right| \\
        &\leq \left| \trace \left((Q_1-Q_2) \cdot \frac{1}{n} \sum_{i=1}^{n} \phi_1(y_i) \phi_1(y_i)^T \cdot Q_1^T \right) \right| + \left| \trace \left(Q_2 \cdot \frac{1}{n} \sum_{i=1}^{n} \phi_1(y_i) \phi_1(y_i)^T \cdot (Q_1-Q_2)^T \right) \right| \\
        &+ \left| \trace \left( Q_2 \cdot \frac{1}{n} \sum_{i=1}^{n} (\phi_1(y_i) - \phi_2(y_i)) \phi_1(y_i)^T \cdot Q_2 \right) \right| + \left| \trace \left( Q_2 \cdot \frac{1}{n} \sum_{i=1}^{n} \phi_2(y_i) (\phi_1(y_i)-\phi_2(y_i))^T \cdot Q_2 \right) \right|.
    \end{align*}
    We bound the first two terms using the Cauchy-Schwarz inequality:
    \begin{align*}
        \left| \trace \left((Q_1-Q_2) \cdot \frac{1}{n} \sum_{i=1}^{n} \phi_1(y_i) \phi_1(y_i)^T \cdot Q_1^T \right) \right| &\leq \|Q_1-Q_2\|_F \left\|\frac{1}{n} \sum_{i=1}^{n} \phi_1(y_i) \phi_1(y_i)^T \cdot Q_1^T \right\|_F \\
        &\leq C_{\Theta} L_{\Phi}^2 R^2 \|Q_1-Q_2\|_F,
    \end{align*}
    where we used a bound parallel to the one in Lemma \ref{lem: normbound1}. Analogously, 
    \begin{align*}
        \left| \trace \left(Q_2 \cdot \frac{1}{n} \sum_{i=1}^{n} \phi_1(y_i) \phi_1(y_i)^T \cdot (Q_1-Q_2)^T \right) \right| &\leq C_{\Theta} L_{\Phi}^2 R^2 \|Q_1-Q_2\|_F.
    \end{align*}
    To bound the remaining two terms, we use the fact that $\phi \in \Phi$ is $L_{\Phi}$-Lipschitz and satisfies $\phi(0) = 0$, so that $\|\phi(x)\| \leq L_{\Phi}\|x\|$ for $x \in \R^d$. In particular,
    \begin{align*}
        \left| \trace \left( Q_2 \cdot \frac{1}{n} \sum_{i=1}^{n} (\phi_1(y_i) - \phi_2(y_i)) \phi_1(y_i)^T \cdot Q_2 \right) \right| &\leq \|Q_2\|_F^2 \left\| \frac{1}{n} \sum_{i=1}^{n} (\phi_1(y_i) - \phi_2(y_i)) \phi_1(y_i)^T \right\|_F \\
        &\leq \frac{\|Q\|_F^2}{n} \sum_{i=1}^{d} \|\phi_1(y_i)\| \|\phi_1(y_i) - \phi_2(y_i)\| \\
        &\leq C_{\Theta}^2 L_{\Phi}R \cdot \|\phi_1 - \phi_2\|_{L^{\infty}(B_R)}.
    \end{align*}
    Similarly for the fourth term. This proves that
    \begin{align*}
        |\trace(A_{n,1}-B_{n,1})| &\leq 2 \left( C_{\Theta} L_{\Phi}^2 R^2 \|Q_1-Q_2\|_F + C_{\Theta}^2 L_{\Phi}R \cdot \|\phi_1 - \phi_2\|_{L^{\infty}(B_R)}\right) \\
        &\leq 2 C_{\Theta}^2 L_{\Phi}^2 R^2 \left(\|Q_1-Q_2\|_F + \|\phi_1-\phi_2\|_{L^{\infty}(B_R)} \right),
    \end{align*}
    hence 1) is established. To prove 2), we use the fact that
    \begin{align*}
        |\trace(A_{n,1}^2-A_{n,2}^2)| &= |\trace((A_{n,1}-A_{n,2})(A_{n,1} + A_{n,2}))| \\
        &\leq \|A_{n,1}-A_{n,2}\|_F \|A_{n,1}+A_{n,2}\|_F. 
    \end{align*}
    Note that the bound in 1) on $|\trace(A_{n,1}-A_{n,2})|$ also holds for $\|A_n-B_n\|_F$ (the same argument works, using the triangle inequality and sub-multiplicativity of $\| \cdot \|_F$ in place of the linearity of the trace and Cauchy-Schwarz inequality, respectively). For the remaining factor, we have by Lemma \ref{lem: normbound1},
    \begin{align*}
        \|A_{n,1}+A_{n,2}\|_F \leq \|A_{n,1}\|_F + \|A_{n,2}\|_F \leq 2 C_{\Theta}^2 L_{\Phi}^2 R^2.
    \end{align*}
    This proves the bound
    \begin{align*}
        |\trace(A_{n,1}^2-A_{n,2}^2)| &\leq 4 C_{\Theta}^4 L_{\Phi}^4 R^4 \left(\|Q_1-Q_2\|_F + \|\phi_1-\phi_2\|_{L^{\infty}(B_R)} \right),
    \end{align*}
    hence establishing 2). To prove 3), we have
    
    \begin{align*}
        |\trace(\Sigma(A_{n,1}^2 - A_{n,2}^2)| &= |\trace(\Sigma A_{n,1}(A_{n,1}-A_{n,2}) + \Sigma(A_{n,1}-A_{n,2})A_{n,2})| \\
        &\leq |\trace(\Sigma A_{n,1}(A_{n,1}-A_{n,2}))| + |\trace(\Sigma A_{n,2}(A_{n,1}-A_{n,2}))| \\
        &\leq \|A_{n,1}-A_{n,2}\|_F \cdot \left(\|\Sigma A_{n,1}\|_{F} + \|\Sigma A_{n,2}\|_F \right) \\
        &\leq 2 C_{\Theta}^2 L_{\Phi}^2 R^2 \cdot \|\Sigma\|_{\op} \cdot \|A_{n,1}-A_{n,2}\|_F \\
        &\leq 2 C_{\Theta}^2 L_{\Phi}^2 R^2 \sigma_{\max}^2 \|A_{n,1}-A_{n,2}\|_F.
    \end{align*}
    Finally, to prove 4),
    \begin{align*}
        |\trace(A_{n,1}^4-A_{n,2}^4)| &= |\trace((A_{n,1}-A_{n,2})(A_{n,1}+A_{n,2})(A_{n,1}^2 + A_{n,2}^2))| \\
        &\leq \|A_{n-1}-A_{n,2}\|_F \cdot \|(A_{n,1}+A_{n,2})(A_{n,1}^2+A_{n,2}^2)\|_F \\
        &\leq \|A_{n,1}-A_{n,2}\|_F \cdot 4 \max \left(\|A_{n,1}\|_F^3, \|A_{n,2}\|_F^3 \right) \\
        &\leq 8 C_{\Theta}^8 L_{\Phi}^8 R^8 \|A_{n,1}-A_{n,2}\|_F,
    \end{align*}
    where we used Lemma \ref{lem: normbound1} and 1). We concludes the proof.
\end{proof}
We immediately put Lemma \ref{lem: coveringnumattn} to work to deduce covering number bounds for the hypothesis class $L_t(\Theta) = \{\ell_t(\theta, \cdot): \theta \in \Theta \}.$

\begin{lem}\label{lem: coveringnumloss}
    Fix $t > 0$ and define $R_1(t) = \sqrt{d}\sigma_{\max}+t$ and $R_2(t) = \sigma_{\max}^2 \left( 1+t+\sqrt{\frac{d}{n}} \right).$ For any $\theta_1 = (Q_1,\phi_1)$ and $\theta_2 = (Q_2, \phi_2)$ in $\Theta$, we have
    \begin{align*}
        |\ell_t(\theta_1, \cdot) - \ell_t(\theta_2, \cdot)\|_{L^{\infty}} \leq M(t,\lambda,\Theta) \left(\|Q_1-Q_2\|_F + \|\phi_1-\phi_2\|_{L^{\infty}(B(R_1(t)))} \right),
    \end{align*}
    where
    $$ M(t,\lambda,\Theta) = 4 (C_{\Theta}^4 L_{\Phi}^4 R_1^4(t) + C_{\Theta}^2 L_{\Phi}^2 R_1(t)^2) + 8\lambda \left(C_{\Theta}^8 L_{\Phi}^8 R_1^8(t) + C_{\Theta}^4 L_{\Phi}^4 R_1^4(t) R_2(t) \right).
    $$
    As a consequence, if $\{Q_i\}_i$ denotes an $\frac{\epsilon}{2 M(t,\lambda,\Theta)}$-cover of the $C_{\Theta}$-ball of $\R^{d \times d}$ with respect to the Frobenius norm, and $\{\phi_j\}_j$ denotes a $\frac{\epsilon}{2 M(t,\lambda,\Theta)}$-cover of $\Phi$ with respect to the $L^{\infty}(B_{R_1(t)})$-norm, then
\end{lem}
\begin{proof}
    Given $\mathbf{y} = (y_1, \dots, y_{2n})$ we write $A_{n,1} = A_{n,\theta_1}$ and $A_{n,2}$ as in Lemma \ref{lem: coveringnumattn}. Also, recall $\Sigma_n = \frac{1}{n} \sum_{i=1}^{n} y_{n+i} y_{n+i}^T.$ By the definition of $\ell_t$, it suffices to assume that $\|y_i\| \leq R_1(t)$ for $1 \leq i \leq n$ and $\|\Sigma_n\|_{\op} \leq R_2(t).$ Then
    \begin{align*}
       &|\ell_t(\theta_1, \cdot) - \ell_t(\theta_2, \cdot)| = \left|\E_{x \sim \mathcal{N}(0,I)} \left[\|A_{n,1}x-x\|^2 - \|A_{n,2}x - x\|^2 \right] + \lambda \left(\left\| A_{n,1}^2 - \Sigma \right\|_F^2 - \left\| A_{n,2}^2 - \Sigma \right\|_F^2 \right) \right| \\
       &= \left|\trace((A_{n,1}-I)^2) - \trace((A_{n,2}-I)^2) +  \lambda \left(\left\| A_{n,1}^2 - \Sigma \right\|_F^2 - \left\| A_{n,2}^2 - \Sigma \right\|_F^2 \right) \right| \\
       &= \left|\trace(A_{n,1}^2-A_{n,2}^2) + 2\trace(A_{n,2}-A_{n,1}) + \lambda \left( \trace(A_{n,1}^4-A_{n,2}^4) + 2\trace(\Sigma (A_{n,2}^2-A_{n,1}^2) \right) \right| \\
       &\leq \left(4 (C_{\Theta}^4 L_{\Phi}^4 R_1^4(t) + C_{\Theta}^2 L_{\Phi}^2 R_1(t)^2) + 8\lambda \left(C_{\Theta}^8 L_{\Phi}^8 R_1^8(t) + C_{\Theta}^4 L_{\Phi}^4 R_1^4(t) R_2(t) \right) \right)\\
       & \qquad \times \left(\|Q_1-Q_2\|_F + \|\phi_1-\phi_2\|_{L^{\infty}(B(R_1(t)))} \right) \\
       &=M(t,\lambda,\Theta) \left(\|Q_1-Q_2\|_F + \|\phi_1-\phi_2\|_{L^{\infty}(B(R_1(t)))} \right).
    \end{align*}
    Hence the result holds.
\end{proof}
    
\subsection{Probability and concentration results}

\begin{lem}\label{lem: changeofvar}
   Let $\Sigma = U \Lambda U^T$, where $\Lambda$ is a positive-definite diagonal matrix and $U$ is an orthogonal matrix. Let $y \sim \mathcal{N}(0,\Sigma)$, and define $g: \R \rightarrow \R$ by
    $$ g_{sq}(x) = \begin{cases}
        \sqrt{x}, \; x \geq 0, \\
        -\sqrt{-x}, \; x < 0.
    \end{cases}
    $$
    Then $z =\left(\frac{\pi}{2} \right)^{1/4} g_{sq}(U^T y)$ satisfies $\textrm{cov}(z) = \Lambda^{1/2}$. In addition, $z$ is sub-Gaussian with parameters depending only on the diagonal matrix $\Lambda$.
\end{lem}
\begin{proof}
    Since $U^Ty$ is a symmetric random and $g$ is an odd function, $z$ is a symmetric random variable, which means that $\textrm{cov}(z) = \E[z z^T].$ Since $z$ is symmetric (hence $\E[z] = 0$) and the entries of $z$ are independent, we have $\E[z_k z_j] = 0$ whenever $1 \leq j \neq k \leq d.$ For each $1 \leq k \leq d$, $\sqrt{\frac{2}{\pi}}z_k^2$ is distributed according to the folded normal distribution with variance $\sigma_k^2$. The mean of the folded normal distribution is known to be $\sqrt{2}{\pi} \sigma_k$. Thus, $\E[z_k^2] = \lambda_k$. This proves that $\textrm{cov}(z) = \Lambda^{1/2}.$ The sub-Gaussianity of $z$ is clear from the density of $z$, which is a simple application of the change of variables formula.
\end{proof}

We quote a technical lemma on concentration which we use in the proof of our approximation result, due to \citep{rudelson2013hanson}.

\begin{lem}\label{lem: gaussianconc}[Gaussian concentration bound]
    Let $y \sim \mathcal{N}(0,\Sigma)$. Then
    $$ \mathbb{P}\left\{\|y\| \geq \sqrt{\textrm{tr}(\Sigma)} + t \right\} \leq 2 \exp \Big(-\frac{t^2}{C \|\Sigma\|_{\textrm{op}}} \Big),
    $$
    where $C > 0$ is a constant independent of $\Sigma$ and $d.$
\end{lem}
Finally, we derive the concentration inequalities of the random variable
$$ f_{\Sigma,\lambda}(A_{n,\theta}) - \E[f_{\Sigma,\lambda}(A_{n,\theta})]
$$
used in the proof of Theorem \ref{thm: TPGE}. Here, $f_{\Sigma,\lambda}(A) = \|A-I\|_F^2 + \lambda \|A^2-\Sigma\|_F^2$, $\Sigma \in \textrm{supp}(\mu)$, and we recall 
$$A_{n,\theta} = Q \cdot \frac{1}{n} \sum_{i=1}^{n} \phi(y_i) \phi(y_i)^T \cdot Q^T, \; y_1, \dots, y_n \sim \mathcal{N}(0,\Sigma), \; (Q,\phi) \in \Theta.$$ Since $A_{n,\theta}$ is a covariance matrix of i.i.d. sub-Gaussian random vectors, we expect $f_{\Sigma,\lambda}(A_{n,\theta})$ to enjoy sub-exponential concentration about its mean. However, the statistic $f_{\Sigma,\lambda}$ is defined in terms of powers of $A_{n,\theta}$, which are no longer sub-exponential. Luckily, powers of sub-exponential random variables fit into a broad family known as \textit{sub-Weibull random variables}. The language of sub-Weibull random variables helps us derive concentration inequalities for the following statistics of $A_{n,\theta}.$ The following definition was introduced in \citep{vladimirova2020sub}.

\begin{defn}\label{def: subweibull}
    We say that a centered random variable $X \in \R$ is $\alpha$-sub-Weibull if any of the following equivalent properties:
    \begin{enumerate}
        \item There exists $K_1 > 0 $ such that $\mathbb{P} \left( |X| \geq t \right) \leq 2 \exp \left( -(x/K_1)^{1/\alpha} \right)$ for all $t \geq 0.$
        \item The moments of $X$ satisfy $\E[|X|^k] \leq K_2^k k^{\alpha k}$ for some $K_2 > 0.$
        \item The MGF of $|X|^{1/\alpha}$ satisfies $\E\left [\exp \left((\lambda |X|)^{1/\alpha} \right) \right] \leq \exp \left((\lambda K_3)^{1/\alpha} \right)$ for some $K_3 > 0$.
        \item The MGF of $|X|^{1/\alpha}$ is bounded at some point.
    \end{enumerate}
    We refer to the constant $K_1$ as the \textit{sub-Weibull parameter} of $X$.
\end{defn}

A useful property of sub-Weibull random variables is their closure under algebraic operations. The following Lemma is due to \citep{bakhshizadeh2023algebra}.

\begin{lem}\label{lem: weibullalg}[Algebra properties of sub-Weibull random variables]
    Sub-Weibull random variables are closed under addition and multiplication. In particular, if $X$ and $Y$ are $\alpha$-sub-Weibull and $\beta$-sub-Weibull random variables respectively, then $XY$ is $\alpha + \beta$-sub-Weibull and $X+Y$ is $\max(\alpha,\beta)$-sub-Weibull.
\end{lem}

The language of sub-Weibull random variables allows us to easily state the relevant concentration inequalities for $A_{n,\theta}.$

\begin{lem}\label{lem: weibullconcentration}
    For any $\theta = (Q,\phi) \in \Theta$ and $\Sigma \in \textrm{supp}(\mu)$, with the matrix $A_{n,\theta}$ as defined above, the following hold for any $t > 0$
    \begin{enumerate}
        \item $\mathbb{P} \left( \|A_{n,\theta}-I\|_F^2 - \E\|A_{n,\theta}-I\|_F^2 > t \right) \leq \exp \left( -\left(\frac{nt}{O \left( C_{\Theta}^2 L_{\Theta}^2 \sigma_{\max}^2 \right)} \right)^{1/2} \right).$
        \item $\mathbb{P} \left(\|A_{n,\theta}^2-\Sigma\|_F^2 - \E\|A_{n,\theta}^2-\Sigma\|_F^2 > t \right) \leq \exp \left( - \left( \frac{nt}{O \left( C_{\Theta}^2 L_{\Phi}^2 \sigma_{\max}^2 \right)} \right)^{1/4} \right).$
    \end{enumerate}
\end{lem}
\begin{proof}
    We can write $A_{n,\theta} = \frac{1}{n} \sum_{i=1}^{n} z_{n,\theta} z_{n,\theta}^T$, where $z_{n,\theta} = Q \phi(y_i)$ and $y_i \sim \mathcal{N}(0,\Sigma).$ The mapping $y \mapsto Q \phi(y)$ is $C_{\Theta} L_{\Phi}$-Lipschitz, so by Theorem 2.26 in \citep{wainwright2019high}, the random vectors $z_{1,\theta}, \dots, z_{n,\theta}$ are sub-Gaussian with sub-Gaussian norm $O \left(C_{\Theta} L_{\Phi} \sigma_{\max} \right).$ Chapter 4 of \citep{vershynin2018high} establishes the exponential concentration of matrices of the form $\frac{1}{n} \sum_{i=1}^{n} z_i z_i^T$, where $z_1, \dots, z_n$ are iid sub-Gaussian random vectors. Therefore, the algebra properties of sub-Weibull random variables imply that $\|A_{n,\theta} -I\|_F^2$ and $\|A_{n,\theta}^2 - \Sigma\|_F^2$ are $\alpha$-sub-Weibull with $\alpha = 2$ and $\alpha = 4$ respectively, and the fact that their sub-Weibull parameters are of the same order as the sub-Gaussian parameter of $z_{1,\theta}, \dots, z_{n,\theta}.$ 
\end{proof}

\begin{lem}\label{lem: weibullconcentration2}
    Given $\theta = (Q, \phi) \in \Theta$, $\Sigma \in \textrm{supp}(\mu)$, and $y_1, \dots, y_n \sim \mathcal{N}(0,\Sigma)$, recall the matrix
    \begin{align*}
        A_{n,\theta} = Q \cdot \frac{1}{n} \sum_{i=1}^{n} \phi(y_i) \phi(y_i)^T \cdot Q^T.
    \end{align*}
    Recall also the function $f_{\Sigma,\lambda}$ defined by
    $$ f_{\Sigma,\lambda}(A) = \E_{x \sim \mathcal{N}(0,I)}[\|Ax-x\|^2] + \lambda \left\|A^2 - \Sigma \right\|_F^2, \; A \in \R^{d \times d}_{sym}.
    $$
    Then, when $\lambda$, $n$, and $N$ are sufficiently large, we have for any fixed constant $c > 0$,
    \begin{align*} &\mathbb{P} \left(f_{\Sigma,\lambda}(A_{n,\theta}) - \min f_{\Sigma, \lambda} > \E \left[f_{\Sigma,\lambda}(A_{n,\theta}) - \min f_{\Sigma,\lambda} \right] + c\lambda \right) \\ 
    &\leq \exp \left(- \left(\frac{\Omega(n \lambda)}{O \left(C_{\Theta}^2 L_{\Phi}^2 \right)} \right)^{1/2} \right)+ \exp \left(- \left( \frac{\Omega(n)}{O \left(C_{\Theta}^2 L_{\Theta}^2 \right)} \right)^{1/4} \right).
    \end{align*}
\end{lem}
\begin{proof}
    Clearly we can ignore the constant $\min f_{\Sigma,\lambda}$ and instead prove
    \begin{align*}
        \mathbb{P} \left(f_{\Sigma,\lambda}(A_{n,\theta}) > \E[f_{\Sigma,\lambda}(A_{n,\theta})] + c \lambda \right) \leq \exp \left(- \left(\frac{\Omega(n \lambda)}{O \left(C_{\Theta}^2 L_{\Phi}^2 \right)} \right)^{1/2} \right) + \exp \left(- \left( \frac{\Omega(n)}{O \left(C_{\Theta}^2 L_{\Theta}^2 \right)} \right)^{1/4} \right).
    \end{align*}
    To prove this claim, we use the definition of $f_{\Sigma,\lambda}$ and write
    \begin{align*}
         &\mathbb{P} \left( f_{\Sigma,\lambda}(A_{n,\theta}) > \E[f_{\Sigma,\lambda}(A_{n,\theta})] + c \lambda \right) \\ &= \mathbb{P} \left( \|A_{n,\theta}-I\|_F^2 + \lambda \left\|A_{n,\theta}^2 - \Sigma \right\|_F^2 > \E \left[ \|A_{n,\theta}-I\|_F^2 + \lambda \left\|A_{n,\theta}^2 - \Sigma \right\|_F^2 \right] + c\lambda \right) \\
         &\leq \mathbb{P} \left(\|A_{n,\theta}- I\|_F^2 > \E \left\|A_{n,\theta}-I \right\|_F^2 + \frac{c\lambda}{2} \right) + \mathbb{P} \left( \|A_{n,\theta}^2-\Sigma\|_F^2 > \E \|A_{n,\theta}^2-\Sigma\|_F^2 + \frac{c}{2} \right).
    \end{align*}
    By Lemma \ref{lem: weibullconcentration}, we have the concentration inequalities
    \begin{align*}
        \mathbb{P} \left(\|A_{n,\theta}- I\|_F^2 > \E \left\|A_{n,\theta}-I \right\|_F^2 + \frac{c\lambda}{2} \right) \leq \exp \left(- \left(\frac{\Omega(n \lambda)}{O \left(C_{\Theta}^2 L_{\Phi}^2 \right)} \right)^{1/2} \right)
    \end{align*}
    and
    \begin{align*}
        \mathbb{P} \left( \|A_{n,\theta}^2-\Sigma\|_F^2 > \E \|A_{n,\theta}^2-\Sigma\|_F^2 + \frac{c}{2} \right) \leq \exp \left(- \left( \frac{\Omega(n)}{O \left(C_{\Theta}^2 L_{\Theta}^2 \right)} \right)^{1/4} \right).
    \end{align*}
    We conclude the proof.
\end{proof}
\vskip 0.2in

\section{Experiment Details}\label{Appendix_num}
\paragraph{ICL training protocol}
During training, the ICL model is trained using mini-batches, where each batch corresponds to a single task. Across an epoch, the data loader cycles through all tasks so that each target class (task) is repeatedly visited. We use Adam optimizer \citep{kingma2014adam} with CosineAnnealingLR \citep{loshchilov2016sgdr} schedule and base learning rate 3e-5. Synthetic experiments are trained for 1000 epochs, while others are trained for 3000 epochs.

All real-world experiments are conducted in latent space. Both target prompts and samples are encoded into latent features, with the source Gaussian dimension matched to the latent space. Details for each high-dimensional dataset are:
\paragraph{MNIST}
We scale the pixels of the images to $[0,1]$ and flatten to length $784$. A autoencoder is trained on these flattened inputs and used to provide a $d$-dimensional latent space with $d{=}4$. The encoder uses three sequential convolutional layers with channels $64/128/256$, each followed by BatchNorm and ReLU, with a residual block at each resolution. The decoder is a fully connected expansion to $256{\times}7{\times}7$, two transposed-convolution layers back to $1{\times}28{\times}28$, and a final Sigmoid so that the outputs remain in $[0,1]$. The AE is trained end-to-end with per-pixel mean-squared error (MSE) between prediction and target using Adam (learning rate $10^{-3}$) for 100 epochs.

\paragraph{Fashion-MNIST}
Like MNIST, we scale the pixels to $[0,1]$. 
We train a residual convolutional AE on Fashion-MNIST, which provide a $d$-dimensional latent space (we use $d{=}15$ in this experiment). The encoder uses three convolutional stages with channels $64/128/256$, each followed by BatchNorm and ReLU, with a residual block at each resolution. The resulting $256{\times}7{\times}7$ tensor is flattened and mapped to the latent vector by a linear layer. The decoder mirrors this design by first expanding linearly to $256{\times}7{\times}7$, then two transposed-convolution stages back to $1{\times}28{\times}28$, and a final Sigmoid so outputs remain in $[0,1]$. The AE is optimized with a pixel-wise MSE plus a Structural Similarity Index Measure (SSIM) $\text{MSE} + \mu\,(1-\text{SSIM})$ with $\mu{=}0.25$. The SSIM is used to encourage perceptual fidelity beyond pixel-wise accuracy. For optimizer, we use Adam with learning rate $10^{-3}$ for 300 epochs.

\paragraph{ModelNet10}
The original ModelNet10 dataset is the mesh document. We convert it into point cloud by uniform surface sampling using \texttt{trimesh.sample\_surface}. We fix the target sample number $n_{\text{pts}}$ as 512, and the resulting point set is centered and normalized to the unit ball by dividing the maximum radius.
Using the similar latent-to-latent generative setting as in 2D image case, we train an AE for the point cloud on the flattened $(3n_{\text{pts}})$ inputs and set the latent dimension as 10. The structure of the encoder uses three sequential MLP blocks (\texttt{Conv1d}) with output channels 64/128/256, each followed by BatchNorm and ReLU, a symmetric max-pooling over points to obtain a permutation invariant 256-dimensional global descriptor, and two fully-connected layers ($256\!\to 256$ with ReLU, then $256\!\to z$) to produce the output of the encoder of the latent dimension ($d=10$). The decoder is a lightweight MLP ($z\!\to\!512\!\to\!1024\!\to\!3n_{\text{pts}}$ with Tanh at the output), which produces a flattened $(3n_{\text{pts}})$ reconstruction. The AE is trained end-to-end with the permutation invariant Chamfer-$\ell_2$ loss between the predicted and ground-truth point sets (each reshaped to $(n_{\text{pts}}\!\times\!3)$). We use optimizer Adam with learning rate 1e-3, and train for 300 epochs.

\section{More Numerical Results}\label{Appendix_image}
We present additional examples of the images generated by our model in the setting of Section \ref{sec: gausstohighdim}. Figures \ref{mnistvis}, \ref{famnistvis}, and \ref{app: modelnet} display the model's generated images when trained on the MNIST, FashionMNIST, and ModelNet datasets, respectively.

\begin{figure}[ht]
  \centering
  % 
  % 

  %---  ---
  \foreach \row in {0,...,11}{% 
    \foreach \col in {0,...,9}{% 
      \includegraphics[width=0.09\linewidth]{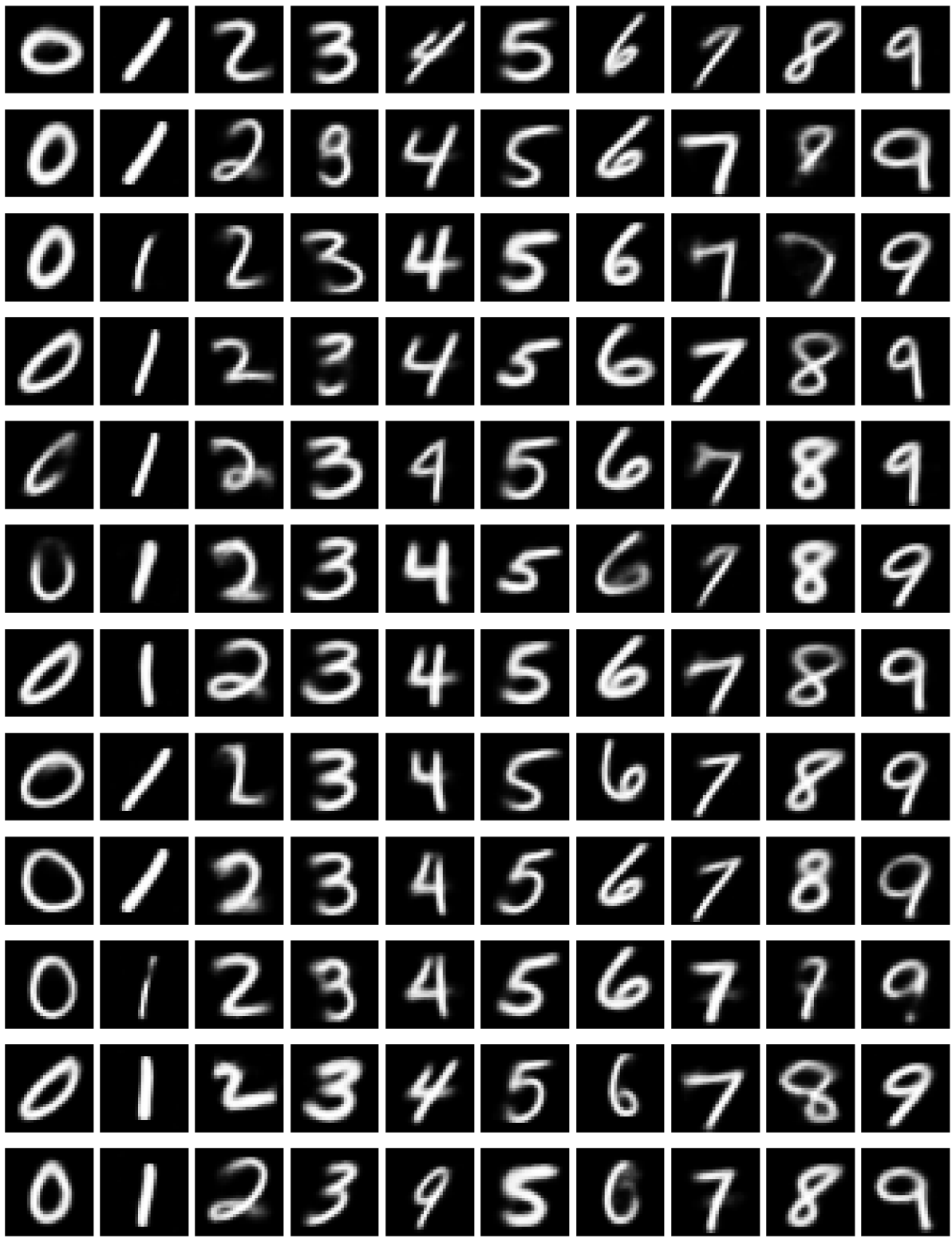}%
      \ifnum\col<9\hspace{\imgsep}\fi
    }%
    \par\medskip 
  }%

\caption{Visualization results of the Gaussian to MNIST task.}
\label{mnistvis}
\end{figure}

\begin{figure}[ht]
  \centering
  % 
  % 

  %---  ---
  \foreach \row in {0,...,11}{% 
    \foreach \col in {0,...,9}{% 
      \includegraphics[width=0.09\linewidth]{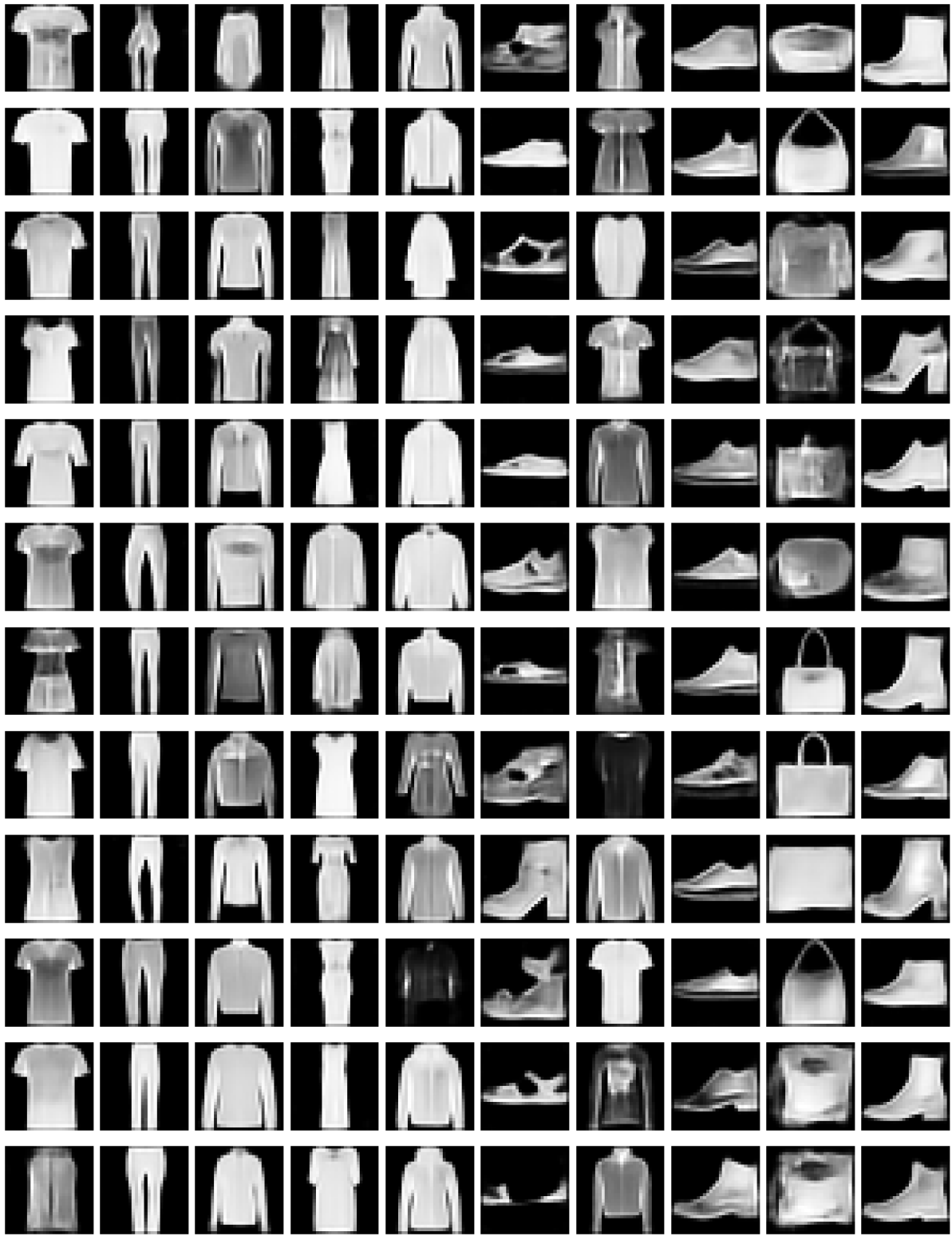}%
      \ifnum\col<9\hspace{\imgsep}\fi
    }%
    \par\medskip 
  }%

\caption{Visualization results of the Gaussian to Fashion-MNIST task.}
\label{famnistvis}
\end{figure}

\begin{figure}[ht]
  \centering
  \foreach \y in {3,...,14}{%
    \foreach \x in {0,...,9}{%
      \includegraphics[width=0.09\linewidth]{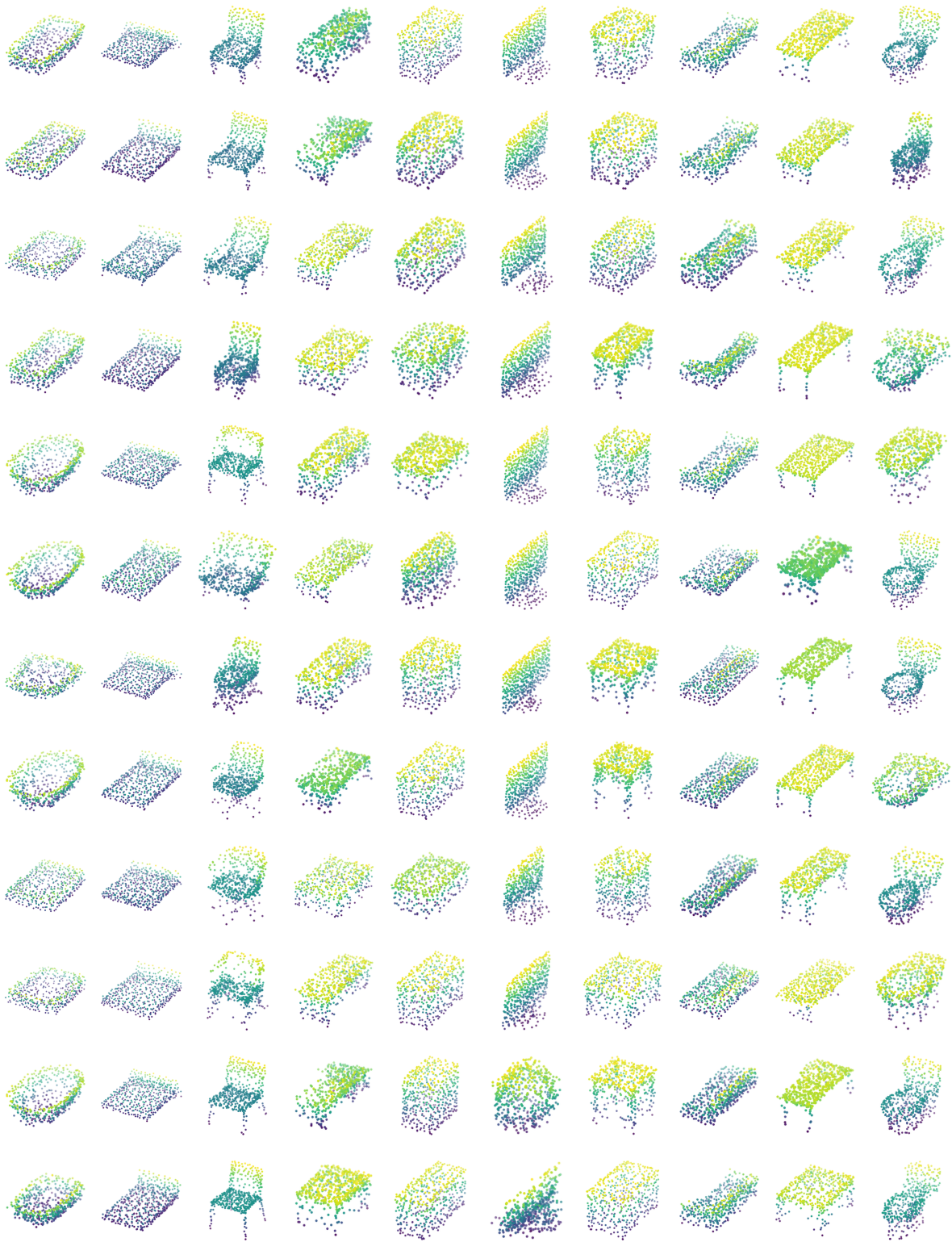}%
      \ifnum\x<9\hspace{\imgsep}\fi
    }%
    \par\medskip
  }%
  \caption{Visualization results of the Gaussian to ModelNet10 Task.}
  \label{app: modelnet}
\end{figure}

\end{document}